\documentclass[10pt]{article}
\usepackage[accepted]{tmlr} 
\usepackage[english]{babel}
\usepackage[utf8]{inputenc}
\usepackage[T1]{fontenc}
\usepackage{amsmath}
\usepackage{amsfonts}
\usepackage{amssymb}
\usepackage{amsthm}
\usepackage{dsfont}
\usepackage{enumerate}
\usepackage{subcaption}
\usepackage{multirow}
\usepackage{makecell}
\usepackage{placeins}
\usepackage{xcolor}
\usepackage{adjustbox}
\usepackage{booktabs}
\usepackage{url}
\definecolor{C0}{HTML}{3182ce}
\usepackage[breaklinks=true,colorlinks,bookmarks=false,citecolor=C0]{hyperref}
\usepackage{tcolorbox}
\usepackage{pdflscape}

\usepackage{tikz-cd}
\usetikzlibrary{babel}

\setcitestyle{citesep={,}}

\newtheorem{constraint}{Constraint}
\newtheorem{remark}{Remark}
\newtheorem{definition}{Definition}

\newtheorem{corollary}{Corollary}
\newtheorem{proposition}{Proposition}
\newtheorem{lemma}{Lemma}
\newtheorem{example}{Example}

\newtheorem{innercustomprop}{Proposition}
\newenvironment{customprop}[1]{\renewcommand\theinnercustomprop{#1}\innercustomprop}
{\endinnercustomprop}

\newcommand{\relu}{\mathrm{ReLU}}
\newcommand{\acto}{\mathrm{o}}
\newcommand{\actp}{\mathrm{p}}
\newcommand{\actop}{\mathrm{o}/\mathrm{p}}

\defcitealias{poole2016exponential}{Poole}
\defcitealias{schoenholz2016deep}{Schoenholz}

\title{Gaussian Pre-Activations in Neural Networks: \\ Myth or Reality?}
\author{\name Pierre Wolinski \email pierre.wolinski@dauphine.psl.eu \\
	\addr LAMSADE, Paris-Dauphine University, PSL University, CNRS, 75016 Paris, France
	\AND
	\name Julyan Arbel \email julyan.arbel@inria.fr \\
	\addr Univ. Grenoble Alpes, Inria, CNRS, Grenoble INP, LJK, 38000 Grenoble, France}

\def\month{04}  
\def\year{2025} 
\def\openreview{\url{https://openreview.net/forum?id=goe6fv6iSh}} 

\begin{document}\sloppy
	
\maketitle

\begin{abstract}
	The study of feature propagation at initialization in neural networks lies at the root of numerous initialization 
	designs. A very common assumption is that the pre-activations are Gaussian. 
	Although this convenient \emph{Gaussian hypothesis} can be justified when the number of neurons per layer tends to 
	infinity, it is challenged by both theoretical and experimental work for finite-width 
	neural networks.
	Our main contribution is to construct a family of pairs of activation functions and initialization distributions that ensure
	that the pre-activations remain Gaussian throughout the network depth, even in narrow neural networks, 
	under the assumption that the pre-activations are independent. 
	In the process, we discover a set of constraints that a neural network should satisfy to ensure Gaussian 
	pre-activations.
	In addition, we provide a critical review of the claims of the Edge of Chaos line of work and construct a 
	non-asymptotic Edge of Chaos analysis.
	We also propose a unified view on the propagation of pre-activations, encompassing the framework of several well-known 
	initialization procedures. 
	More generally, our work provides a principled framework for addressing the much-debated question: is it desirable to initialize the training of 
	a neural network whose pre-activations are guaranteed to be Gaussian? 
	Our code is available on GitHub: \url{https://github.com/p-wol/gaussian-preact/}.
\end{abstract}



\section*{Notations and vocabulary}

Bold letters such as $\mathbf{Z}$, $\mathbf{W}$, $\mathbf{B}$, or $\mathbf{x}$, represent tensors of order larger or equal to $1$. 
For a tensor $\mathbf{W} \in \mathbb{R}^{n \times p}$, we denote by $W_{ij} \in \mathbb{R}$ its 
component at the 
intersection of the $i$-th row and $j$-th column, $\mathbf{W}_{i\cdot} \in \mathbb{R}^{1 \times p}$ its $i$-th 
row and $\mathbf{W}_{\cdot j} \in \mathbb{R}^{n}$ its $j$-th column. 
Upper-case letters such as $\mathbf{W}$, $X$, $Y$, $Z$, or $G$, represent random variables. 
For a random variable $Z$, the function $f_Z$ represents its density, $F_Z$ its cumulative distribution function (CDF), $S_Z$ its survival function, and $\psi_Z$ its characteristic function. 
The \emph{depth} of a neural network is its number of layers. The \emph{width} of one layer is its number
of neurons or convolutional units. The \emph{infinite-width limit} of a neural network is the limiting case where
the width of each layer tends to infinity.

\section{Introduction}

Let us take a neural network with $L$ layers, in which every layer $l \in [1, L]$ performs the following 
operation:
\begin{align}
\mathbf{X}^{l + 1} &:= \phi(\mathbf{Z}^{l + 1}) , \label{eqn1:act} \\
\text{with: } \quad \mathbf{Z}^{l + 1} &:= \frac{1}{\sqrt{n_l}} \mathbf{W}^{l} \mathbf{X}^{l} + 
\mathbf{B}^{l} , \label{eqn1:preact}
\end{align}
where $\mathbf{X}^{l + 1} \in \mathbb{R}^{n_{l + 1}}$ is its \emph{activation}, $\mathbf{Z}^{l + 1} \in 
\mathbb{R}^{n_{l + 
1}}$ its \emph{pre-activation}, $\phi$ is the coordinate-wise \emph{activation function}, 
$\mathbf{W}^{l} \in \mathbb{R}^{n_{l + 1} \times n_l}$ is the \emph{weight matrix} of the layer, 
$\mathbf{B}^{l} \in \mathbb{R}^{n_{l + 1}}$ its \emph{vector of biases}, and $\mathbf{X}^{l} \in \mathbb{R}^{n_l}$ its 
\emph{input} (also the preceding layer activation). This paper focuses on the distribution of the pre-activations 
$\mathbf{Z}^{l}$ as $l$ grows, for a fixed input 
$\mathbf{x}$, and weights $\mathbf{W}^{l}$ and biases $\mathbf{B}^{l}$ randomly sampled from known distributions.

Recurring questions arise in both Bayesian deep learning and in parameter initialization procedures: How to choose the 
distribution of the parameters $\mathbf{W}^l, \mathbf{B}^l$ and according to which criteria, and how should the 
distribution 
of the pre-activations $\mathbf{Z}^l$ look like? Answering these questions is fundamental to finding efficient ways of 
initializing neural networks, that is, appropriate 
distributions for $\mathbf{W}^l$ and $\mathbf{B}^l$ at initialization. In Bayesian deep learning, this question 
is related to the search for a suitable prior, which is still a topic of intense research 
\citep{wenzel2020good,fortuin2022bayesian,fortuin2022priors,arbel2024primer,papamarkou2024position}.

\paragraph{Initialization strategies.}
A whole line of work in the field of initialization strategies for neural networks is based on  
preserving the statistical characteristics of the pre-activations as they propagate into a network. 
In short, the input of the neural network is assumed to be fixed, while all the parameters are considered to be 
randomly drawn, according to a candidate initialization distribution. Then, by using heuristics, some 
statistical characteristics of the pre-activations are deemed desirable. Finally, these statistical 
characteristics are propagated to the initialization distribution, which indicates how to choose it.

For instance, one of the first results of this kind, proposed by \cite{glorot2010understanding}, is based on the preservation of the variance of both the pre-activations and the backpropagated gradients across the layers of the neural network. The resulting constraint on the initialization distribution of the weights $\mathbf{W}^l$ is about its variance: $\mathrm{Var}(W^l_{ij}/\sqrt{n_l}) = 2/(n_l + n_{l + 1})$.%
\footnote{In the original paper, the considered weight matrix is $\tilde{\mathbf{W}}^l := \mathbf{W}^l / \sqrt{n_l}$, 
so $\mathrm{Var}(\tilde{W}^l_{ij}) = 2/(n_l + n_{l + 1})$.}
Then, \cite{he2015delving} have refined this idea by taking into account the nonlinear deformation 
of the pre-activations by the activation function.
They also showed that the inverted arithmetical average $2/(n_l + n_{l + 1})$, 
resulting from a compromise between the preservation of variance during both propagation and backpropagation, can 
be changed into $1/n_l$ or $1/n_{l + 1}$ with negligible impact on the neural network.%
\footnote{More generally, any choice of the form $1/(n_l^{\alpha} \, n_{l + 1}^{1 - \alpha})$ with $\alpha \in 
[0, 1]$ is valid, as long as the same $\alpha$ is used for all layers.}
Notably, with $\phi(x) = \mathrm{ReLU}(x) := \max(0, x)$, they obtain $\mathrm{Var}(W_{ij}^l) = 2/n_l$, where 
the factor $2$ is meant to compensate for the loss of information due to the fact that $\mathrm{ReLU}$ is zero on 
$\mathbb{R}^-$.
We summarize and compare several initialization strategies 
in Table \ref{tbl:commut} (Section \ref{sec:prop:commut})

\paragraph{Edge of Chaos (EOC).} After the studies revolving around initialization, \cite{poole2016exponential} and \cite{schoenholz2016deep} focused on
the correlation between the pre-activations of two data points $\mathbf{x}_a$ and $\mathbf{x}_b$.
By preserving this correlation when propagating the inputs into the neural network, information about the 
geometry of the input space is meant to be preserved. Thus, training the weights is meaningful at initialization,
regardless of their positions in the network.
This heuristic is finer than the previous ones, since attention is paid to the correlation between 
pre-activations and their variance (a joint criterion is used instead of a marginal one).
The range of valid initialization distributions is modified accordingly, with a relation between $\sigma_w^2 = 
\mathrm{Var}(W^l_{ij})$ and $\sigma_b^2 = \mathrm{Var}(B^l_i)$ that should be ensured.
This specific relationship is referred to as the \emph{Edge of Chaos} (EOC).

Finally, it is worth mentioning the work of \cite{hayou2019impact}, in which the usual claims about the EOC 
initialization are tested with several choices of activation functions.
Notably, the authors have run a large series of experiments in order to check whether the intuition
behind the EOC initialization leads to better performance after training.
Also, at the opposite of finding an initialization distribution for the parameters, 
\cite{klambauer2017self} focused on tuning the parameters of the activation function (leading to the $\mathrm{SELU}$ 
activation function). As in the preceding techniques, they aim for variance-preserving layers.

In the following, the term ``Edge of Chaos''  is used in two different manners: ``EOC framework'', ``EOC formalism'', or ``EOC theory''  refer to a setup where input data points are deterministic and weights and biases are 
random, whereas in the context of initialization of weights and biases, ``EOC'' alone refers to a specific set 
of pairs $(\sigma_w^2, \sigma_b^2)$ matching a given theoretical condition (Point 2, see Section~\ref{sec:prop:prop}).

\paragraph{Bayesian prior and initialization distribution.} 
There exists a close relationship between the initialization distribution in deterministic neural networks and the 
prior distribution in Bayesian neural networks.
For instance, let us use variational inference to approximate the Bayesian posterior of the parameters of a neural
network \citep{graves2011practical}. In this case, the Bayesian posterior is approximated sequentially by
performing a gradient descent over the so-called ``variational parameters'' \citep{hoffman2013stochastic}.
This technique requires to backpropagate the gradient of the loss through the network,
as when training deterministic networks.
Therefore, as with the initialization distribution, the prior distribution must be constructed in such a way that
the input and the gradient of the loss propagate and backpropagate correctly
\citetext{\citealp[Sec.\ 2.2,][]{ollivier2018online}; \citealp[Chap.\ 2,][]{neal1996bayesian}}.

\paragraph{Gaussian hypothesis for the pre-activations.}
In the context of random weights and biases, we call the \emph{Gaussian hypothesis} the assumption that all the 
pre-activations $Z_i^l$ are Gaussian random 
variables, at any layer $l$ and for any neuron $i$.
This hypothesis is common in the theoretical analysis of the properties of neural networks at initialization. 
Specifically, this is a fundamental assumption when studying the ``Neural Tangent 
Kernels'' (NTK) \citep{jacot2018neural} or Edge of Chaos \citep{poole2016exponential}.
In a nutshell, the NTK is an operator describing the optimization trajectory of an \emph{infinitely wide} 
neural network (NN), 
which is believed to help understand the optimization of ordinary NNs.
On one side, this Gaussian hypothesis can be justified in the case of ``infinitely wide'' 
NNs (i.e., when the widths $n_l$ of the layers tend to infinity), by application of the Central Limit 
Theorem \citep{matthews2018gaussian}.
On the other side, this Gaussian hypothesis is apparently necessary to get the results of the EOC and NTK lines of work. 
However, it remains debated for both theoretical and practical reasons.

First, from a strictly theoretical point of view, it has been shown that, for finite-width
NNs (finite $n_l$), the distribution of $Z^{l}_i$ has heavier tails as $l$
increases, that is, as information flows from the input to the output 
\citep{vladimirova2019understanding,vladimirova2021bayesian}.
Second, a series of experiments tend to show that pushing the distribution of the
pre-activations towards a Gaussian (e.g., through a specific Bayesian prior) leads to worse performances than pushing 
it towards distributions with heavier tails, e.g., Laplace distribution \citep{fortuin2022bayesian}.

Besides, the condition under which the Gaussian hypothesis remains valid is an important source of confusion. As an 
example, \cite{sitzmann2020implicit} state that: ``for a uniform input in $[-1, 1]$, the [pre-]activations throughout a 
SIREN%
\footnote{Sinusoidal representation network.}
are standard normal distributed [...], irrespective of the depth of the 
network, if the weights are distributed 
uniformly in the interval $[-c, c]$ with $c = \sqrt{6}/n_{l}$ in each layer [$l$].'' (Theorem 1.8, Appendix 1.3). 
Though this formal statement seems to hold for all layers and whatever their widths, it is only an asymptotic 
result, since it uses the Central Limit Theorem in its proof. 
Consequently, this theorem is not usable in practical SIRENs, since it does not provide any 
speed of convergence of the distribution of the pre-activations to a Gaussian, as each $n_l$ tends to 
infinity.%
\footnote{According to \citet[Th.\ 4,][]{matthews2018gaussian}, the pre-activations tend to become
	Gaussian irrespective of the growth rates of each $n_l$, so Theorem 1.8 of \cite{sitzmann2020implicit} is 
	\emph{asymptotically true for all layers and all growth rates of each $n_l$}.
	But this result still does not provide any convergence speed.}

\paragraph{Deviation from the Gaussian distribution in the finite-width case.} 
More recently, several attempts have been made to characterize the distribution
of the pre-activations in the finite-width case. 
In particular, \cite{yaida2020non,balasubramanian2024gaussian,trevisan2023wide,basteri2024quantitative,favaro2023quantitative}
provide bounds on the discrepancy between the actual distribution of the pre-activations
in the finite-width case and the Gaussian limit in the infinite-width case. 
One may also refer to \citet{Roberts_Yaida_Hanin_2022} for a complete study. 
Although these works provide results for finite-widths networks, 
they only apply when the widths are large. 
Additionally, \citet{vladimirova2019understanding,vladimirova2021bayesian} show that for finite-width
NNs, the pre-activation distribution gets heavier-tailed when getting deeper into the neural network.
From another point of view, \cite{noci2021precise} provides 
an analytical expression is provided for the distribution of the pre-activations,
but only in terms of Meijer-G functions \citep{erdelyi1953higher}, which 
are difficult to manipulate.

\paragraph{Addressing the limitations of the EOC analysis.}
Preserving the shape of the distribution of the pre-activations is
a key assumption of the EOC works. This assumption is usually justified
if we assume that the considered network is in the NTK regime, which means that
the layer widths are large enough to ensure that the Neural Tangent Kernel is constant 
throughout training.
However, this assumption is generally not true and has been discussed in
several papers. For instance, \cite{Hanin2020Finite}
showed that if the ratio depth/width does not tend to zero, the NTK varies
significantly during training. This illustrates that the NTK is unlikely to be constant 
during training for narrow and deep neural networks. 
This phenomenon has also been observed empirically by \cite{seleznova2022analyzing}.

In parallel, another limitation of the EOC has been spotted: 
with the EOC initialization scheme, the correlation between
the pre-activations given by two different outputs tends to
a unique real number, independent of the input data.
Therefore, the EOC initialization inevitably leads to degenerate
correlations in deeper layers, which means that the geometry of the input space is progressively 
forgotten during propagation. This phenomenon has been studied 
extensively by \cite{martens2021rapid},
and \cite{martens2021rapid,zhang2022deep} solve
this problem with an innovative solution: 
transform the activation function in order to ``shape'' the kernel.

This idea of ``kernel shaping'' through a transformation 
of the activation function has then been used by \cite{li2022neural}
to solve both issues presented above: the NTK regime does not hold in practice,
and the EOC initialization leads to degenerate correlations.
So, \cite{li2022neural} have proposed to take into account the $O(1/\sqrt{n})$ deviation 
from the infinite-width limit $n \rightarrow \infty$. 
This results in a set of modifications of the activation functions 
that are adapted to $n$ and, therefore, are also dependent on $n$.

\paragraph{Motivation.}

In the works cited above, the general problem is to control the distribution of pre-activations, 
mostly in order to design good initialization strategies. 
To this end, two approaches have been used: 
1.~describing analytically this distribution with given activation functions
\citep[e.g.,][]{noci2021precise,favaro2023quantitative}; 
2.~tuning the activation function to make the pre-activation match a given property
\citep[e.g.,][]{zhang2022deep,li2022neural}.
In the current work, we take approach 2.~by constructing activation functions and initialization
distributions to provide Gaussian pre-activation at initialization,
\emph{without relying on the infinite-width limit hypothesis or
any asymptotic expansion around that limit}.

The starting point of our study is the original EOC works \citep{schoenholz2016deep,poole2016exponential},
which rely only on the preservation of the shape of the pre-activations.
This behavior is usually justified by the central limit theorem in the infinite-width limit, 
which is one of the main assumptions in the works cited above.
The direction we explore in this paper is different: we study the shape of the distributions
of the pre-activations themselves.

So, we first measure the Gaussianity of the pre-activations.
Since the EOC relies on the assumption that the pre-activations are Gaussian,
it is necessary to check it experimentally and directly, regardless
of how close our neural network is to the NTK regime (i.e., infinite-width limit). 
Second, as \cite{martens2021rapid,zhang2022deep,li2022neural}, 
we propose to construct new activation functions. tions is preserved.
But, unlike these works, we focus on the Gaussianity of the pre-activations, 
so we do not rely on the infinite-width limit hypothesis or
any asymptotic expansion around that limit. 
To do so, we study individual weights, and, unlike \cite{li2022neural}, the resulting activation functions
do not depend on the width of the neural network. 
This approach explicitly allows the study of very narrow neural networks, regardless
of the depth. 

More generally, the goal of the series of articles presented above, of which this one is a part, is to propose a principled framework for
studying the neural networks at initialization, while proposing initialization
schemes and activation functions to solve problems we may encounter at initialization.
To study the distribution of the pre-activations in detail, several strategies can be used.
As explained in Section \ref{sec:common-assumptions}, such a study is very difficult to achieve
for arbitrary activation functions. Moreover, these studies are based on deviations from the infinite-width limit
\citep{yaida2020non,balasubramanian2024gaussian,favaro2023quantitative}.
The strategy we adopt in this paper is the opposite: we construct activation functions that, by design,
should lead to pre-activations whose distribution is known.

\paragraph{Contributions.}
Therefore, we present in this paper a new framework for studying neural networks at initialization
and propose new families of initialization distributions and activation functions.
In particular, to obtain results that are independent of how close we are to the NTK regime, 
we focus on the distribution of the pre-activations during propagation. 
First, we experimentally test their Gaussianity, i.e., whether the \emph{Gaussian hypothesis}
holds or not. We show that it does not hold in many cases.
Second, we construct families of activation functions and initialization distributions
to make this Gaussian hypothesis hold, and we show their construction process.
Since this process is flexible and can be used to construct other activation functions,
it is a contribution in itself.

To be more specific:
\begin{itemize}
	\item we experimentally demonstrate that the Gaussian hypothesis is mostly invalid in multilayer perceptrons 
	with finite width (Section~\ref{sec:prop:prop} and Section~\ref{sec:prop:realistic});
	\item contrary to a claim of \cite{poole2016exponential} and usual practical results in the 
	EOC framework, 
	we show in Proposition~\ref{prop:counter} that the variance of the pre-activations does not always have at most one nonzero attraction point; we provide 
	an example of an activation function for which the number of such attraction points is infinite (Section~\ref{sec:prop:counter});
    \item we propose a unified view on pre-activations propagation, encompassing the framework of several well-known 
	initialization procedures (Section~\ref{sec:prop:commut});
	\item we deduce a set of constraints that the activation function and the initialization distribution of weights 
	and biases must fulfill to guarantee Gaussian pre-activations at initialization under a pre-activations' independence assumption -- including with 
	finite-width layers (Section~\ref{sec:solve:decomposing} and Section~\ref{sec:solve:why_weibull});
	\item we propose new families of activation functions, denoted by $(\phi^{\acto}_{\theta})_{\theta}$
	and $(\phi^{\actp}_{\theta})_{\theta}$, and initialization distributions designed to achieve this 
	goal of Gaussian pre-activations at initialization (Section~\ref{sec:solve:product} and Section~\ref{sec:solve:function});
	\item we empirically demonstrate that, with our activation functions $(\phi^{\actp}_{\theta})_{\theta}$,
	the distribution of the pre-activations remains Gaussian during propagation, while it
	drifts away from the standard Gaussian when using $\mathrm{tanh}$ and 
	$\mathrm{ReLU}$ (Sections \ref{sec:expe:ks} and \ref{sec:expe:multilayer});
	\item we show that, overall, when training very narrow and deep networks, the usual setups with 
	$\tanh$ or $\mathrm{ReLU}$ activation functions and EOC initialization lead to 
	worse performance than with our activation functions $(\phi^{\acto}_{\theta})_{\theta}$ 
	(Section \ref{sec:expe:training}).
\end{itemize}
Additionally, we train, evaluate and compare neural networks built according to our family of activation 
functions and initialization distributions, and usual ones ($\mathrm{tanh}$ or $\mathrm{ReLU}$, Gaussian EOC 
initialization).

\paragraph{Summary of the paper.}
We begin by reviewing the widely used assumptions about pre-activations in Section \ref{sec:common-assumptions}.
Then, we make in Section~\ref{sec:propagating} a critical review of several results about pre-activations 
propagation in a neural network: the discussion, additional experiments, and the criticism we are 
proposing, particularly in the EOC line of works, are the foundations of our contributions. 
In Section~\ref{sec:imposing_th}, we propose a new family of activation functions, along with a 
family of initialization distributions. They are defined so as to ensure that the 
pre-activations distribution propagates without deformation across the layers, including with networks that are far 
from the ``infinite-width limit''. 
More specifically, we ensure that the pre-activations remain Gaussian at any layer, and we provide a set of constraints
that the activation function and the initialization distribution of the parameters should match to attain this goal, under a pre-activations' independence assumption.
Finally, we propose in Section~\ref{sec:experiments} a series of simulations in order to check whether 
our propositions meet the requirement of maintaining Gaussian pre-activations across neural networks. 
We also show the performance of trained neural networks in different setups, including standard ones and the 
one we are proposing.

\section{Common assumptions about pre-activations}
\label{sec:common-assumptions}

Several attempts have been made to provide an accurate description of
the pre-activations at initialization.

\paragraph{Neural networks as Gaussian processes.}
A line of works, initiated by \cite{neal1996bayesian}, has focused on the study of the pre-activations
of neural networks on their infinite-width limit, that is, the limit of infinite number of neurons per layer.
Notably, it has been proven \citep{matthews2018gaussian,lee2018deep} that the outputs and the pre-activations of
a randomly initialized neural network tend to a Gaussian process
with known mean and covariance. The proof can be achieved by using the fact
that the pre-activations of the same layer are exchangeable random variables, 
making it possible to use a variation of the central limit theorem at each layer.
These first results have led to the development of the Neural Tangent Kernel
framework \citep{jacot2018neural,lee2019wide,arora2019exact}, which 
allows us to study the full optimization trajectory of a neural network trained by gradient descent
\citep{du2019gradient,allen2019convergence,arora2019fine}.

Despite the significance of these convergence results, the discussion about 
the practicality of their hypotheses is still active,
notably about the infinite number of neurons per layer
and the linear training dynamic.
Besides, two essential by-products of these hypotheses are subject to discussion:
the fact that the pre-activations are \emph{Gaussian} and \emph{independent} in the infinite-width limit.

\paragraph{Assumption of independent pre-activations.}
The assumption that pre-activations of the same layer are independent is very common in the
literature on neural networks' initialization. For instance, 
this assumption is explicitly made in \cite{he2015delving}.
In the Edge of Chaos literature, the pre-activations of the same layer are implicitly assumed to be independent
\citep{poole2016exponential,schoenholz2016deep}:
at a given layer, the vector of pre-activations is assumed to be a Gaussian vector
with a diagonal covariance, which implies independence.
More broadly, as soon as the infinite-width limit is taken, the pre-activations tend
to Gaussian processes with zero covariance between pre-activations of the same layer
\citep{matthews2018gaussian,jacot2018neural},
that is independence.

\section{Propagating pre-activations} \label{sec:propagating}

In this section, we propose a critical review of several aspects of the Edge of Chaos framework. We recall the 
fundamental ideas of the EOC in Section~\ref{sec:prop:prop}. In Section~\ref{sec:prop:realistic}, we perform some 
experiments at the initialization of a multilayer perceptron, in which we propagate data points sampled from CIFAR-10. 
These results illustrate a limitation of the EOC framework when using neural networks with a small number of neurons 
per layer. Then, we build in Section~\ref{sec:prop:counter} an activation function such that the variance of the 
propagated pre-activations admits an infinite number of stable fixed points, which is a counterexample to a claim of 
\cite{poole2016exponential}. Finally, we propose in Section~\ref{sec:prop:commut} a unified representation of several 
initialization procedures.

\subsection{Propagation of the correlation between data points} \label{sec:prop:prop}

\paragraph{Edge of Chaos (EOC) framework.}
In the EOC line of work, the inputs of the neural network are supposed to be fixed, while the weights and biases are random. 
In order to study the propagation of the distribution of the pre-activations $\mathbf{Z}^l$, \cite{poole2016exponential} and \cite{schoenholz2016deep} propose to study two quantities:
\begin{align}
	v_{a}^l &:= \mathbb{E} [(Z_{j;a}^l)^2] , \label{eqn:def_v} \\
	c_{ab}^l &:= \frac{1}{\sqrt{v_{a}^l v_{b}^l}} \mathbb{E} [Z_{j;a}^l Z_{j;b}^l] , \label{eqn:def_c}
\end{align}
where $Z_{j;a}^l$ is the $j$-th coordinate of the vector $\mathbf{Z}^l_a$ of pre-activations before layer $l$, when the 
input of the neural network is a data point $\mathbf{x}_a$. The expectation is computed over the full set of the 
parameters, i.e., weights and biases. So, we can interpret $v_{a}^l$ as the variance of the pre-activations of a data 
point $\mathbf{x}_a$, and $c_{ab}^l$ as the correlation between the pre-activations of two data points $\mathbf{x}_a$ 
and $\mathbf{x}_b$, over random initializations of the parameters, distributed independently in the following way:
\begin{alignat}{10}
	W_{ij}^l &\sim \mathrm{P}_w(\sigma_w) & \quad\text{ with } \quad&& \mathbb{E}[W_{ij}^l] &= 0 & \quad\text{ and }\quad && \mathrm{Var}(W_{ij}^l) &= \sigma_w^2, \label{eqn:dist_w} \\
	B_{i}^l &\sim \mathrm{P}_b(\sigma_b) & \quad\text{ with } \quad&& \mathbb{E}[B_{i}^l] &= 0 & \quad\text{ and }\quad && \mathrm{Var}(B_{i}^l) &= \sigma_b^2 \label{eqn:dist_b} .
\end{alignat}

\begin{remark}
	Since the parameters are sampled independently with zero-mean, then:
	\begin{align}
		\mathbb{E} [Z_{j_1;a}^l Z_{j_2;a}^l] &= v_{a}^l \delta_{j_1 j_2}, \\
		\frac{1}{\sqrt{v_{a}^l v_{b}^l}} \mathbb{E} [Z_{j_1;a}^l Z_{j_2;b}^l] &= c_{ab}^l \delta_{j_1 j_2}.
	\end{align}
	This is why Definitions~\eqref{eqn:def_v} and~\eqref{eqn:def_c} do not depend on $j$, and the crossed 
	terms in $j_1$ and $j_2$ are not worth studying (they are zero).
\end{remark}

\begin{remark}
	The distributions of $W_{ij}^l$ and $B_{i}^l$ considered by \cite{poole2016exponential} and \cite{schoenholz2016deep} are normal with zero-mean, that is:
	\begin{align*}
		\mathrm{P}_w(\sigma_w) = \mathcal{N}(0, \sigma_w^2) \quad \text{ and } \quad \mathrm{P}_b(\sigma_b) = \mathcal{N}(0, \sigma_b^2) .
	\end{align*}
	We loosen this assumption in~\eqref{eqn:dist_w} and~\eqref{eqn:dist_b}, where we assume that these random 
	variables are zero-mean with a variance we can control. Their theoretical claim remains valid under this broader
	assumption.
\end{remark}

\paragraph{Theoretical analysis.}
Given two fixed inputs $\mathbf{x}_a$ and $\mathbf{x}_b$, the goal of the EOC theory is to study the propagation of the correlation $c_{ab}^{l}$ of $Z_{j;a}^{l}$ and $Z_{j;b}^{l}$ which goes as follows
\citep{poole2016exponential,schoenholz2016deep}:
\begin{enumerate}
	\item build recurrence equations for $(c_{ab}^{l})_l$ of the form: $c_{ab}^{l + 1} = f(c_{ab}^{l})$; \label{pt:eoc_1}
	\item describe the dynamics of $(c_{ab}^{l})_l$; \label{pt:eoc_2}
	\item provide a procedure to compute the variance of the weights' and biases' distributions such that $(c_{ab}^{l})_l$ tends to $1$ with a sub-exponential rate (instead of an exponential rate). \label{pt:eoc_3}
\end{enumerate}

\paragraph{Point 1: recurrence equations.}
Point~\ref{pt:eoc_1} is achieved by using the Gaussian hypothesis for the pre-activations. That is, the distribution of 
the pre-activations $\mathbf{Z}^{l}$ is assumed to be Gaussian, whatever the layer $l$ and its width $n_l$. The 
recurrence equations define a variance map $\mathcal{V}$ and a correlation map $\mathcal{C}$ as follows:
\begin{align}
	v_{a}^{l + 1} &= \mathcal{V}(v_{a}^l | \sigma_w, \sigma_b) := \sigma_w^2 \int \phi\left(\sqrt{v_{a}^l}z\right)^2 
	\, \mathcal{D}z + \sigma_b^2 \label{eqn:rec_v}\\
	c_{ab}^{l + 1} &= \mathcal{C}(c_{ab}^l, v_{a}^{l}, v_{b}^{l} | \sigma_w, \sigma_b) 
		:= \frac{1}{\sqrt{v_{a}^{l + 1} v_{b}^{l + 1}}} \left[ \sigma_w^2 \int \phi\left(\sqrt{v_{a}^l} 
		z_1\right)\phi\left(\sqrt{v_{b}^l} z_2'\right) \, \mathcal{D}z_1 \mathcal{D}z_2 + \sigma_b^2 \right] , 
		\label{eqn:rec_c} \\
	 &\text{ where } z_2' := c_{ab}^l z_1 + \sqrt{1 - (c_{ab}^l)^2} z_2 , 
	\quad \text{ and } \quad\mathcal{D}z := \frac{1}{\sqrt{2 \pi}} \exp\left(-\frac{z^2}{2}\right) \, \mathrm{d}z . \label{eqn:def_Dz}
\end{align}

These equations are approximations of the true information propagation dynamics, which involves necessarily the number of neurons $n_l$ per layer. Actually, passing a Gaussian vector $\mathbf{Z}^l$ through a layer with random Gaussian weights and biases produces a pre-activation $\mathbf{Z}^{l + 1}$ with a distribution which is difficult to describe. On one side, as the dimension $n_l$ of the input $\mathbf{Z}^l$ tends to infinity, the Central Limit Theorem (CLT) applies, and the output $\mathbf{Z}^{l + 1}$ tends to become Gaussian.%
\footnote{Even if the coordinates $Z_j^l$ are dependent, the CLT is still valid, as proven by 
\cite{matthews2018gaussian} by using properties of exchangeable random variables \citep{de1937prevision}.}
The assumption that the components of $\mathbf{Z}^{l + 1}$ are Gaussian
is referred to as the \emph{Gaussian hypothesis}.
On the other side, with finite $n_l$, the tail of the distribution of $\mathbf{Z}^{l + 1}$ has been proven to become 
heavier than the Gaussian one, both theoretically \citep{vladimirova2019understanding,vladimirova2021bayesian} and 
experimentally (see Section~\ref{sec:prop:realistic}).

Finally, by assuming that $(v_{a}^l)_l$ and $(v_{b}^l)_l$ have already converged to the same limit $v^* \neq 0$ 
as $l \rightarrow \infty$, 
it is possible to rewrite Equation~\eqref{eqn:rec_c} in a nicer way:
\begin{align}
	c_{ab}^{l + 1} &= \mathcal{C}_*(c_{ab}^l) := \mathcal{C}(c_{ab}^l, v^*, v^* | \sigma_w, \sigma_b) . 
	\label{eqn:rec_c_simple}
\end{align}
Now that the dynamics of $(c_{ab}^l)_l$ is written in the form $c_{ab}^{l + 1} = 
\mathcal{C}_*(c_{ab}^l)$, it becomes sufficient to plot $\mathcal{C}_*$ to study its convergence. 
The two hypotheses made, i.e., the Gaussian one and the instant convergence of $(v_{a}^l)_l$ 
and $(v_{b}^l)_l$ to a unique nonzero limit $v^*$, are fundamental to obtaining the simple equation of 
evolution~\eqref{eqn:rec_c_simple}.

\paragraph{Point 2: dynamics of the correlation through the layers.}
Then, Point~\ref{pt:eoc_2} can be achieved. Now that the trajectory of $(c_{ab}^{l})_l$ is determined only by 
the function $\mathcal{C}_*$, it becomes easy to find numerically its limit $c^{*}$ and its rate of convergence. 
Specifically, we can distinguish three possible cases:
\begin{itemize}
	\item \emph{chaotic phase}: $\lim_{l \rightarrow \infty} c_{ab}^{l} = c^{*} < 1$. The correlation between 
	$Z_{j;a}^{l}$ and $Z_{j;b}^{l}$ tends to a constant that is strictly less than $1$. So, even if $Z_{j;a}^1$ and 
	$Z_{j;b}^1$ are highly correlated (which means that $\mathbf{x}_a$ and $\mathbf{x}_b$ are close to each other), 
	they tend to decorrelate when going deeper in the network;
	\item \emph{ordered phase}: $\lim_{l \rightarrow \infty} c_{ab}^{l} = c^{*} = 1$ with 
	$\mathcal{C}_*'(1) < 1$ (the prime here denotes the derivative of a function). The correlation between 
	$Z_{j;a}^{l}$ and $Z_{j;b}^{l}$ tends to $1$ with 
	an exponential rate, including when $Z_{j;a}^{1}$ and $Z_{j;b}^{1}$ are almost fully decorrelated;
	\item \emph{edge of chaos}: $\lim_{l \rightarrow \infty} c_{ab}^{l} = c^{*} = 1$ with 
	$\mathcal{C}_*'(1) = 1$. The correlation between $Z_{j;a}^{l}$ and $Z_{j;b}^{l}$ tends to $1$ 
	with a sub-exponential rate, including when $Z_{j;a}^{1}$ and $Z_{j;b}^{1}$ are almost fully decorrelated.
\end{itemize}

\paragraph{Point 3: best choices for the initialization distribution.}
\cite{poole2016exponential} and \cite{schoenholz2016deep} claim that pairs $(\sigma_w^2, \sigma_b^2)$ which lead either 
to the chaotic phase or the ordered phase should be avoided. In both cases, we expect that information contained in the 
propagated data (or the backpropagated gradients) would vanish at an exponential rate. So, we want to find pairs 
$(\sigma_w^2, \sigma_b^2)$ lying ``at the edge of chaos'', that is, making the sequence $(c_{ab}^{l})_l$ converge to 
$1$ at a sub-exponential rate.
The \emph{Edge of Chaos initializations} are the
initialization distributions of the weights and the biases such that the pair of variances 
$(\sigma_w^2, \sigma_b^2)$ lies at the Edge of Chaos.

\begin{remark}
	Even in the favorable edge of chaos configuration, the sequence of correlations $(c_{ab}^{l})_l$ tends to 
	$1$, whatever the data points $\mathbf{x}_a$ and $\mathbf{x}_b$. So a loss of information at initialization seems 
	unavoidable in very deep networks. 
	
	If one wants to create an initialization procedure with a smaller information loss, it becomes reasonable
	to consider data-dependent initialization schemes (i.e., a warm-up phase before training). 
	Such an initialization strategy has been sketched by \cite{mao2021neuron}, who make 
	use of the ``Information Bottleneck'' formalism \citep{tishby1999information, shwartz2017opening, 
	saxe2019information}.
\end{remark}

\paragraph{Strengths and weaknesses of the Edge of Chaos framework.}
One key feature of the Edge of Chaos framework is the simplicity of the recurrence 
equation~\eqref{eqn:rec_c_simple}: it involves only the correlation $c_{ab}^{l}$ as a variable, and all other 
parameters (such as $v_a^l$ and $v_b^l$) are assumed to be fixed once and for all. 
Notably, in Equation~\eqref{eqn:rec_c}, the computation of $c_{ab}^{l + 1}$ involves 
the distribution of the pre-activations outputted by layer $l$, which is assumed to be $\mathcal{D} = \mathcal{N}(0, 
1)$. 
In other words, the distribution $\mathcal{D}^l$ of the pre-activations $\mathbf{Z}^l$ is assumed to be constant 
and equal to $\mathcal{D}$.
However, in neural networks with finite widths, $\mathcal{D}^l$ is not constant and evolves according to a 
propagation equation:
\begin{align}
\mathcal{D}^{l + 1} \text{ is the distribution of } \mathbf{Z}^{l + 1} = \frac{1}{\sqrt{n_l}} \mathbf{W}^{l} 
\phi(\mathbf{Z}^{l}) + \mathbf{B}^{l}, \text{where } \mathbf{Z}^{l} \sim \mathcal{D}^l.
\end{align}

Such an equation involves a sum of products of random variables, which is usually difficult to keep track of.%
\footnote{For instance, a product of two Gaussian random variables is not Gaussian.}
Though \cite{noci2021precise} have proposed an analytical procedure to compute the $\mathcal{D}^l$ explicitly, 
it works only for 
networks with $\mathrm{ReLU}$ or linear activation functions, and involves Meijer G-functions 
\citep{erdelyi1953higher}, which are difficult to handle numerically.

\subsection{Results on realistic datasets} \label{sec:prop:realistic}

As far as we know, there does not exist any experimental result about the propagation of the correlations
with a non-synthetic dataset and a finite-width neural network.
We propose to visualize in Figure~\ref{fig:corr_main} the propagation of correlation $c^l_{ab}$ with dataset 
CIFAR-10 (results on MNIST are reported in Appendix~\ref{app:expe:extra:MNIST}), in the case of the multilayer perceptron with various numbers of neurons per layer $n_l$ (i.e., 
widths). Then, we show in Figure~\ref{fig:prop:relu_tanh} the distance between the standardized distribution of the 
pre-activations and the standard Gaussian $\mathcal{N}(0, 1)$. 

\paragraph{Propagation of the correlations.}
First, we have sampled randomly $10$ data points in each of the $10$ classes of the CIFAR-10 dataset, that is 
$100$ in total for each dataset. Then, for each tested neural network (NN) architecture, we repeated $n_{\mathrm{init}} 
= 1000$ times the following operation: (i) sample the parameters according to the EOC;%
\footnote{For the $\tanh$ activation function, a study of the EOC can be found in \cite{poole2016exponential}. For 
$\relu$, the EOC study is more subtle and can be found in \cite{hayou2019impact}.}
(ii) propagate the $100$ data points in the NN. 
Thereafter, for each pair $(\mathbf{x}_a,\mathbf{x}_b)$ of the selected $100$ data points, we have computed the 
empirical correlation $c_{ab}^l$ between the obtained pre-activations, averaged over the $n_{\mathrm{init}}$ samples. 
Finally, we have averaged the results over the classes: the matrix $C_{pq}^l$ plotted in Figure~\ref{fig:corr_main} 
shows the mean of the correlation $c_{ab}^l$ for data points $\mathbf{x}_a$ and $\mathbf{x}_b$ belonging respectively 
to classes $p$ and $q$ in $\{0, \cdots , 9\}$.%
\footnote{Correlations $c_{ab}^l$ with $a = b$ have been excluded from 
the computation to show the intra-class correlation between \emph{different} samples.} 
Only the experiments with CIFAR-10 are reported in Figure~\ref{fig:corr_main}; the results on MNIST, which are similar, 
are reported in Figure~\ref{fig:corr_suppl} in Appendix~\ref{app:expe:extra:MNIST}. 

In accordance with the theory of the EOC, we observe in Figure~\ref{fig:corr_main}
that the average correlation between pre-activations tends to $1$, except in the case $\phi = \relu$ and 
$n_l = 10$. In this case, it is not even clear that the sequences of correlations $(C_{pq}^l)_l$ converge at all, since 
some inter-class correlations are lower at $l = 30$ than $l = 10$, while we expected them to grow until $1$.
There is also a difference between activation functions $\mathrm{tanh}$ and $\mathrm{ReLU}$: the convergence to $1$ 
seems to be much quicker with $\phi = \mathrm{ReLU}$ than with  $\phi = \mathrm{tanh}$, when $n_l = 100$.

We observe that the rate of convergence of $(C^l_{pq})_l$ towards $1$ not only varies with the NN 
width $n_l$ but also varies in different directions depending on the activation function. When $n_l$ grows from $10$ 
to $100$, the convergence of $(C_{pq}^l)_l$ to $1$ slows down with $\phi = \mathrm{tanh}$, while it accelerates with 
$\phi = \mathrm{ReLU}$. Since this striking inconsistency with the EOC theory is related to $n_l$, it is due 
to the ``infinite-width'' approximation, precisely made to eliminate the dependency on $n_l$ in the recurrence 
Equations~\eqref{eqn:rec_v} and~\eqref{eqn:rec_c}, and consequently simplify them.

\begin{figure}[p!]
	\setlength\tabcolsep{5pt}
	\begin{tabular}{cccccc}
		& $\phi$ & input & $l = 10$ & $l = 30$ & $l = 50$ \\
		\makecell[b]{\parbox[t]{2mm}{\multirow{2}{*}{\rotatebox[origin=c]{90}{$n_l = 10$ neurons per layer}}}\\ \\ \\ \\ \\} & \makecell[b]{$\mathrm{ReLU}$ \\ \\ \\ \\} &
		\includegraphics[page=1,width=.20\linewidth]{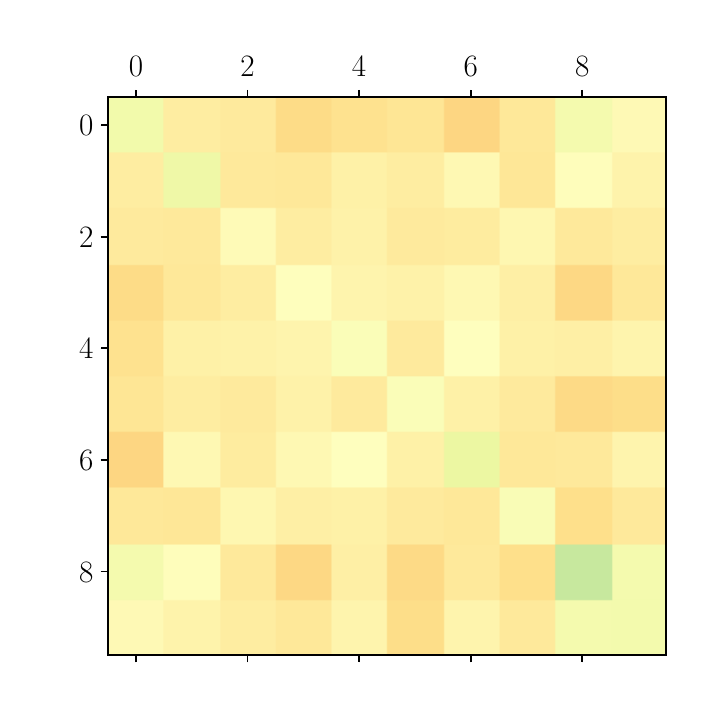}
		 &
		\includegraphics[page=11,width=.20\linewidth]{preactiv-init-cifar10-02/StatsEcab_wth-10_act-relu_the-0.00_sampler-normal_.pdf}
		 &
		\includegraphics[page=31,width=.20\linewidth]{preactiv-init-cifar10-02/StatsEcab_wth-10_act-relu_the-0.00_sampler-normal_.pdf}
		 &
		\includegraphics[page=51,width=.20\linewidth]{preactiv-init-cifar10-02/StatsEcab_wth-10_act-relu_the-0.00_sampler-normal_.pdf}
		 \\
		& \makecell[b]{$\mathrm{tanh}$ \\ \\ \\ \\} &
		\includegraphics[page=1,width=.20\linewidth]{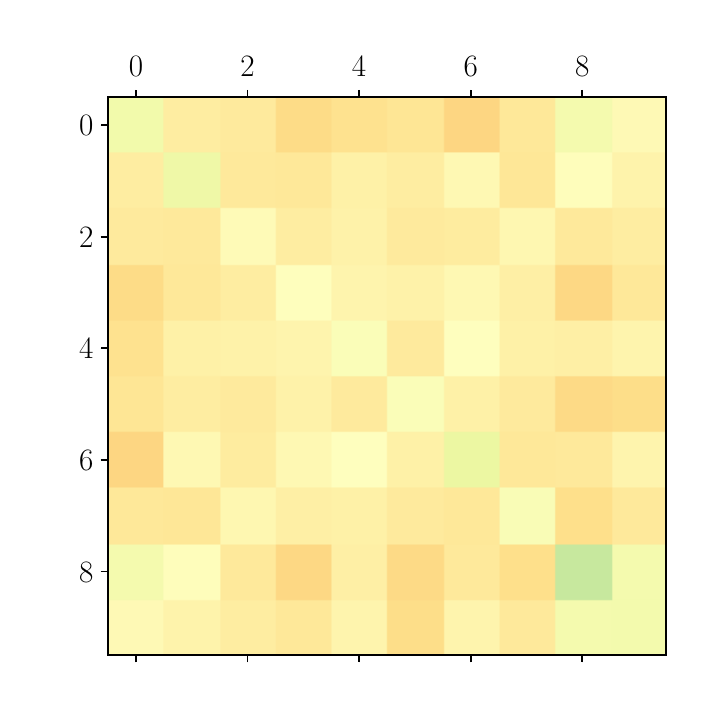}
		 & 
		\includegraphics[page=11,width=.20\linewidth]{preactiv-init-cifar10-02/StatsEcab_wth-10_act-tanh_the-0.00_sampler-normal_.pdf}
		 &
		\includegraphics[page=31,width=.20\linewidth]{preactiv-init-cifar10-02/StatsEcab_wth-10_act-tanh_the-0.00_sampler-normal_.pdf}
		 &
		\includegraphics[page=51,width=.20\linewidth]{preactiv-init-cifar10-02/StatsEcab_wth-10_act-tanh_the-0.00_sampler-normal_.pdf}
		 \\
        \cmidrule{3-6}
		\makecell[b]{\parbox[t]{2mm}{\multirow{2}{*}{\rotatebox[origin=c]{90}{$n_l = 100$ neurons per layer}}}\\ \\ \\ \\ \\} & \makecell[b]{$\mathrm{ReLU}$ \\ \\ \\ \\} &
		\includegraphics[page=1,width=.20\linewidth]{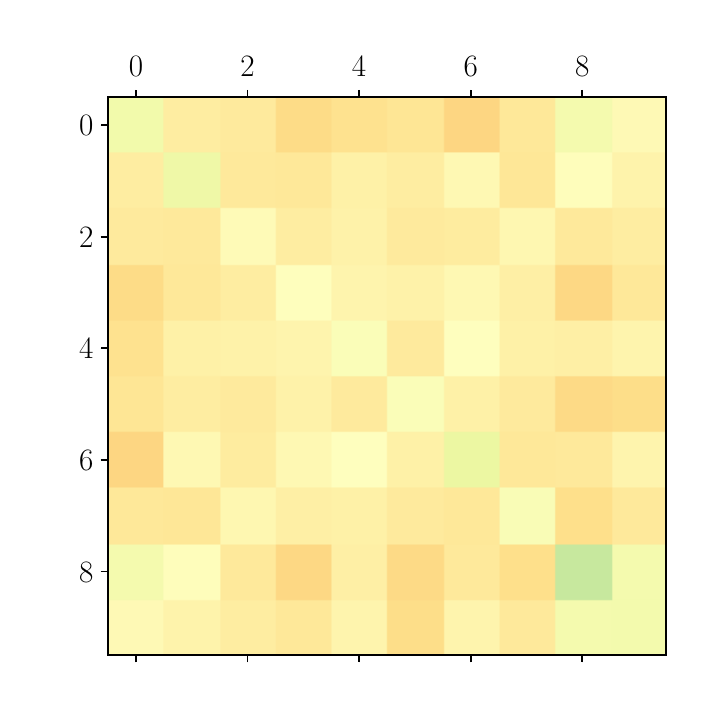}
		 & 
		\includegraphics[page=11,width=.20\linewidth]{preactiv-init-cifar10-02/StatsEcab_wth-100_act-relu_the-0.00_sampler-normal_.pdf}
		 &
		\includegraphics[page=31,width=.20\linewidth]{preactiv-init-cifar10-02/StatsEcab_wth-100_act-relu_the-0.00_sampler-normal_.pdf}
		 &
		\includegraphics[page=51,width=.20\linewidth]{preactiv-init-cifar10-02/StatsEcab_wth-100_act-relu_the-0.00_sampler-normal_.pdf}
		 \\
		& \makecell[b]{$\mathrm{tanh}$ \\ \\ \\ \\} &
		\includegraphics[page=1,width=.20\linewidth]{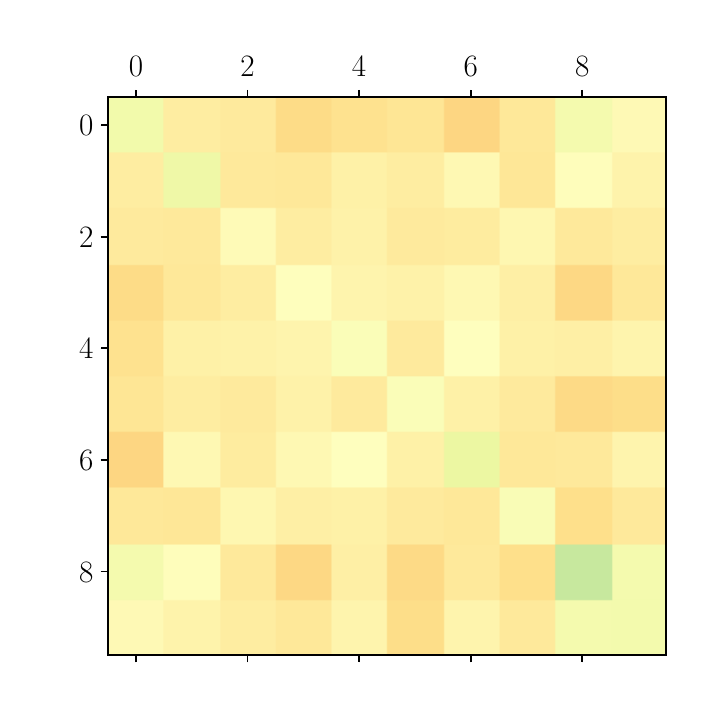}
		 & 
		\includegraphics[page=11,width=.20\linewidth]{preactiv-init-cifar10-02/StatsEcab_wth-100_act-tanh_the-0.00_sampler-normal_.pdf}
		 &
		\includegraphics[page=31,width=.20\linewidth]{preactiv-init-cifar10-02/StatsEcab_wth-100_act-tanh_the-0.00_sampler-normal_.pdf}
		 &
		\includegraphics[page=51,width=.20\linewidth]{preactiv-init-cifar10-02/StatsEcab_wth-100_act-tanh_the-0.00_sampler-normal_.pdf}
		 
	\end{tabular}
	\begin{landscape}
		\hspace*{4mm}
		\begin{subfigure}{1.\linewidth}
				\includegraphics[width=.75\linewidth]{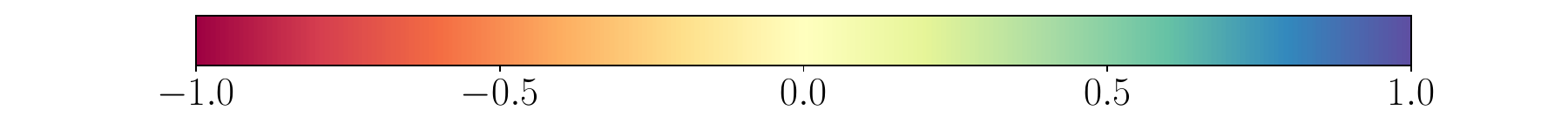}
				\caption*{\hspace*{-53mm}Correlation}
		\end{subfigure}
		\vspace*{-5mm}
	\end{landscape}
	\caption{Propagation of correlations $c^l_{ab}$ in a multilayer perceptron with activation function $\phi \in \{\mathrm{tanh}, \mathrm{ReLU}\}$ and inputs sampled from the CIFAR-10 dataset. According to the EOC claims, they should tend to 1 as $l \rightarrow \infty$. 
	Each plot displays a $10 \times 10$ matrix $C_{pq}^l$ whose entries are the average correlation between the 
	pre-activations propagated by samples from classes $p,q\in \{0, \cdots , 9\}$, at the input and right after layers 
	$l \in \{10, 30, 50\}$. See Fig.~\ref{fig:corr_suppl}, App.~\ref{app:expe:extra:MNIST}, for results on 
	MNIST.} \label{fig:corr_main}
\end{figure}

\begin{remark}
	According to the framework of the EOC, the inputs $\mathbf{x}_a$ and $\mathbf{x}_b$ are assumed to be fixed. So, it 
	is improper to define a correlation $c^0_{ab}$ between $\mathbf{x}_a$ and $\mathbf{x}_b$.
	However, when considering the correlation $c_{ab}^1$ of the pre-activations right after the first layer, the empirical correlation between inputs $\mathbf{x}_a$ and $\mathbf{x}_b$ appears naturally:
	\begin{align}
		c_{ab}^1 = \sigma_w^2 \cdot \frac{\mathbf{x}_a^T \mathbf{x}_b}{n} + \sigma_b^2 ,
	\end{align}
	where $\hat{c}_{ab}^0 := \mathbf{x}_a^T \mathbf{x}_b / n$ plays the role of an empirical correlation between $\mathbf{x}_a$ and $\mathbf{x}_b$, assuming that the empirical mean and variance of both $\mathbf{x}_a$ and $\mathbf{x}_b$ are respectively $0$ and $1$.%
	\footnote{Usually, this assumption does not hold exactly: it is common to normalize the entire dataset in 
		such a way that the whole set of the features of all training data points has empirical mean $0$ and 
		variance $1$, but not each data point individually. See also Definition~\ref{def:whole_dataset}.}
\end{remark}

\paragraph{Propagation of the distances to the Gaussian distribution.}
We test in Figure~\ref{fig:prop:relu_tanh} the Gaussian hypothesis in a multilayer perceptron of $L = 100$ layers,
with a constant width $n_l \in \{10, 100, 1000\}$, and an activation function $\phi \in \{\relu, \tanh\}$. 
We propagate a single point sampled from the CIFAR-10 dataset, and we compute the empirical distribution of the 
pre-activations by drawing $10000$ samples of the parameters of the neural network. 

In Figure~\ref{fig:prop:relu_tanh}, we plot the Kolmogorov--Smirnov statistic of the standardized pre-activations
for each layer, and we compare it to a threshold corresponding to a $p$-value of $0.05$.
Thus, according to Figure~\ref{fig:prop:relu_tanh}, the Gaussian hypothesis is rejected with a $p$-value of $0.05$ for 
the $\relu$ activation 
function in all considered setups ($n_l \in \{10 ,100, 1000\}$), and it is rejected too for the $\tanh$ activation 
function in the narrow network setup ($n_l = 10$).

We provide all the details about the Kolmogorov--Smirnov test in Appendix~\ref{app:ks_test}, and additional 
details and 
experiments in the multilayer perceptron setup in Section~\ref{sec:expe:multilayer}.

\begin{figure}[h!]
	\includegraphics[width=1\linewidth]{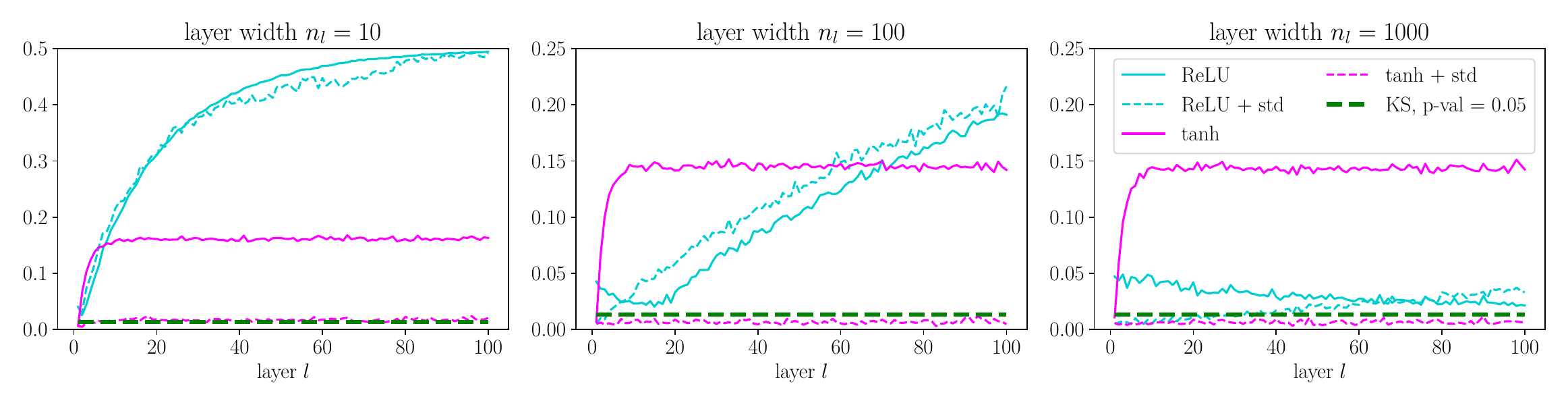}
	\caption{Evolution of the $\mathcal{L}^{\infty}$-distance between 
			the standard Gaussian $\mathcal{N}(0, 1)$ 
			and: (solid line) the empirical CDF of the pre-activations;
			(dotted line) the the empirical CDF of the standardized pre-activations.
			This distance should remain close to zero if the pre-activations are Gaussian.
		The empirical CDF of $\mathcal{D}^l$ have been computed with $10000$ samples.
		The green dotted line corresponds to the threshold given by the Kolmogorov--Smirnov test 
		with a $p$-value of $0.05$: 
		any point above it corresponds to a distribution for which the Gaussian hypothesis should be rejected with a 
		$p$-value of $0.05$.} 
	\label{fig:prop:relu_tanh}
\end{figure}

\subsection{Convergence of the sequence of variances $(v_{a}^l)_l$} \label{sec:prop:counter}

\paragraph{Multiple stable limits.} 
As reminded in Point~\ref{pt:eoc_1}, Section~\ref{sec:prop:prop}, it is a key assumption of the EOC 
formalism to assume that, whatever the starting point $v_{a}^0$, the sequence $(v_{a}^l)_l$ converges to the 
same limit $v^*$ as $l \rightarrow \infty$. 
As far as we know, no configuration where the map $\mathcal{V}$ 
has two nonzero stable points or more has been encountered in past works. Moreover, it is believed that such 
configurations do not exist, as stated for example by \cite{poole2016exponential}: ``for monotonic nonlinearities 
[$\phi$], this length map [$\mathcal{V}$] is a monotonically increasing, concave 
function whose intersections with the unity line determine its fixed points.''

In this subsection, we build a \emph{monotonic} activation function for which the $\mathcal{V}$ map is not 
concave and admits an \emph{infinite} number of stable fixed points.

\begin{definition}[Activation function $\varphi_{\delta, \omega}$]
\label{def:activ-phi}
	For $\delta \in [0 ,1]$ and $\omega > 0$ two real numbers, define:
	\begin{align}
		\varphi_{\delta, \omega}(x) &:= x \exp\left( \frac{\delta}{\omega} \sin( \omega \ln|x|) \right) .
	\end{align}
\end{definition}

It is easy to prove that for all $\delta \in [0 ,1]$ and $\omega > 0$:
\begin{enumerate}
	\item $\varphi_{\delta, \omega}(0) = 0$, by continuity;
	\item $\varphi_{\delta, \omega}$ is odd;
	\item $\varphi_{\delta, \omega}$ is strictly increasing;
	\item the map\footnote{Notation $\mathcal{D}z$ is defined in Eqn.~\eqref{eqn:def_Dz}.} $v \mapsto \int 
	\varphi_{\delta, \omega}(\sqrt{v}z)^2 \, \mathcal{D}z$ is $\mathcal{C}^1$ 
	and strictly increasing.
\end{enumerate}

\begin{definition}[Stable fixed points]
	Let $(u_n)_n$ be a sequence defined by recurrence:
	\begin{align}
		u_{n + 1} = f(u_n) \quad \text{with} \quad
		u_0 \in \mathbb{R} .
	\end{align}
	For any starting point $a$, we denote by $(u_n(a))_n$ the sequence defined as above, with $u_0 = a$.
	
	We say that $u^*$ is a \emph{stable fixed point} of $f$ if $f(u^*) = u^*$ and if there exists an open ball
	$\mathcal{B}(u^*, \epsilon)$ centered in $u^*$ of radius $\epsilon > 0$ such that:
	\begin{align}
		\forall a \in \mathcal{B}(u^*, \epsilon), \quad u_n(a) \underset{n \rightarrow 
		\infty}{\longrightarrow} u^* .
	\end{align}
\end{definition}

\begin{proposition}
	If $f(u^*) = u^*$ and $f$ is $\mathcal{C}^1$ in a neighborhood of $u^*$ with $f'(u^*) \in (-1, 1)$,
	then $u^*$ is a stable fixed point of $f$.
\end{proposition}

\begin{proposition} \label{prop:counter}
	For any $\delta \in (0, 1]$ and $\omega > 0$, let us use the activation function $\phi = \varphi_{\delta, 
	\omega}$ of Definition~\ref{def:activ-phi}. 
	We consider the sequence $(v^l)_l$ defined by:
	\begin{align}
		\forall l \geq 0, \quad v^{l + 1} = \sigma_w^2 \int \varphi_{\delta, 
		\omega}\left(\sqrt{v^l}z\right)^2 \, \mathcal{D}z + \sigma_b^2 , \quad \text{with} \quad
		v^0 \in \mathbb{R}^+_* . \label{eqn:prop:counter}
	\end{align}
	Then there exist $\sigma_w > 0$, $\sigma_b \geq 0$, and a strictly increasing sequence of stable fixed 
	points $(v_k^*)_{k \in \mathbb{Z}}$ of the recurrence equation~\eqref{eqn:prop:counter}.
\end{proposition}

The proof can be found in Appendix~\ref{app:proof_prop_counter}.

\begin{remark}
	In short, Proposition~\ref{prop:counter} ensures that there exists an infinite number of possible 
	(nonzero) limits for the sequence $(v_{a}^l)_l$, depending on $v_{a}^0$. Consequently, the proposed 
	activation 
	functions $\varphi_{\delta, \omega}$ are counterexamples to the claim of \cite{poole2016exponential}. 
\end{remark}

\paragraph{Plots.}
Figure~\ref{fig:perlog} shows the shape of several activation functions $\varphi_{\delta, \omega}$ for various $\omega$, along with their $\mathcal{V}$ maps. We have chosen $\delta = 0.99<1$ to ensure that $\varphi_{\delta, \omega}$ is a strictly increasing function.%
\footnote{We have chosen $\delta$ close to $1$ to obtain a function $\varphi_{\delta, \omega}$ that is strongly 
nonlinear, but a bit lower than $1$ to ensure that $\varphi_{\delta, \omega}'$ remains strictly 
positive, in order to prevent the training process from being stuck.}
We have chosen $\sigma_b^2 = 0$ and $\sigma_w^2 = \sigma_{\omega}^2$, computed as indicated in 
Appendix~\ref{app:sigma_omega}.

In Figure~\ref{fig:perlog:act}, the proposed activation functions exhibit reasonable properties: they are non-linear,  
differentiable at each point (excluding $0$), and  remain dominated by a linear function. However, we expect that as 
$\omega$ grows, $\varphi_{\delta, \omega}$ should become closer and closer to the identity function,%
\footnote{As $\omega \rightarrow \infty$, $\varphi_{\delta, \omega}$ converges pointwise to the identity function.} 
which is not desirable for the activation function of a NN.

In Figure~\ref{fig:perlog:Vv}, it is clear that the function $\varphi_{\delta, \omega}$ with $\omega = 6$ is a 
counterexample to the claim of \cite{poole2016exponential}: two nonzero stable points appear. So, in that case, 
depending on the square norm $v_{a}^0$ of the input, the variance $v_{a}^l$ of the pre-activations may converge to 
different values. For instance, for $\omega = 6$ and an input with square norm $v_{a}^0$ around the unstable point at 
$v \approx 2.3$, it may converge either to $v^* \approx 0.8$ or $v^* \approx 6.5$.

Also, we observe that for $\omega = 6$, the variance map $\mathcal{V}$ tends to be closer to the identity function than 
for smaller $\omega$. Thus, we expect the sequence $(v_{a}^l)_l$ to converge at a slower rate with $\omega = 6$ than 
with $\omega = 2$.

In Figure~\ref{fig:perlog:Cc}, all the configurations lie in the chaotic phase. Since all the correlation maps 
$\mathcal{C}$ are below the identity function, the sequence of correlations $(c_{ab}^l)_l$ always tends to $0$. 
However, the plots are close to the identity, 
so $(c_{ab}^l)_l$ varies very slowly, and we expect that the correlation between data 
points propagates into the NN with little deformation. Despite being not perfect and lying in the chaotic phase, this 
configuration roughly preserves  the input correlations between data points (without performing a pre-training phase), 
which is a desirable property at initialization: information propagates with little deformation to the output, and the 
error can be backpropagated to the first layers.

\begin{figure}[p!]
	\begin{subfigure}{.5\linewidth}
		\includegraphics[width=1.05\linewidth]{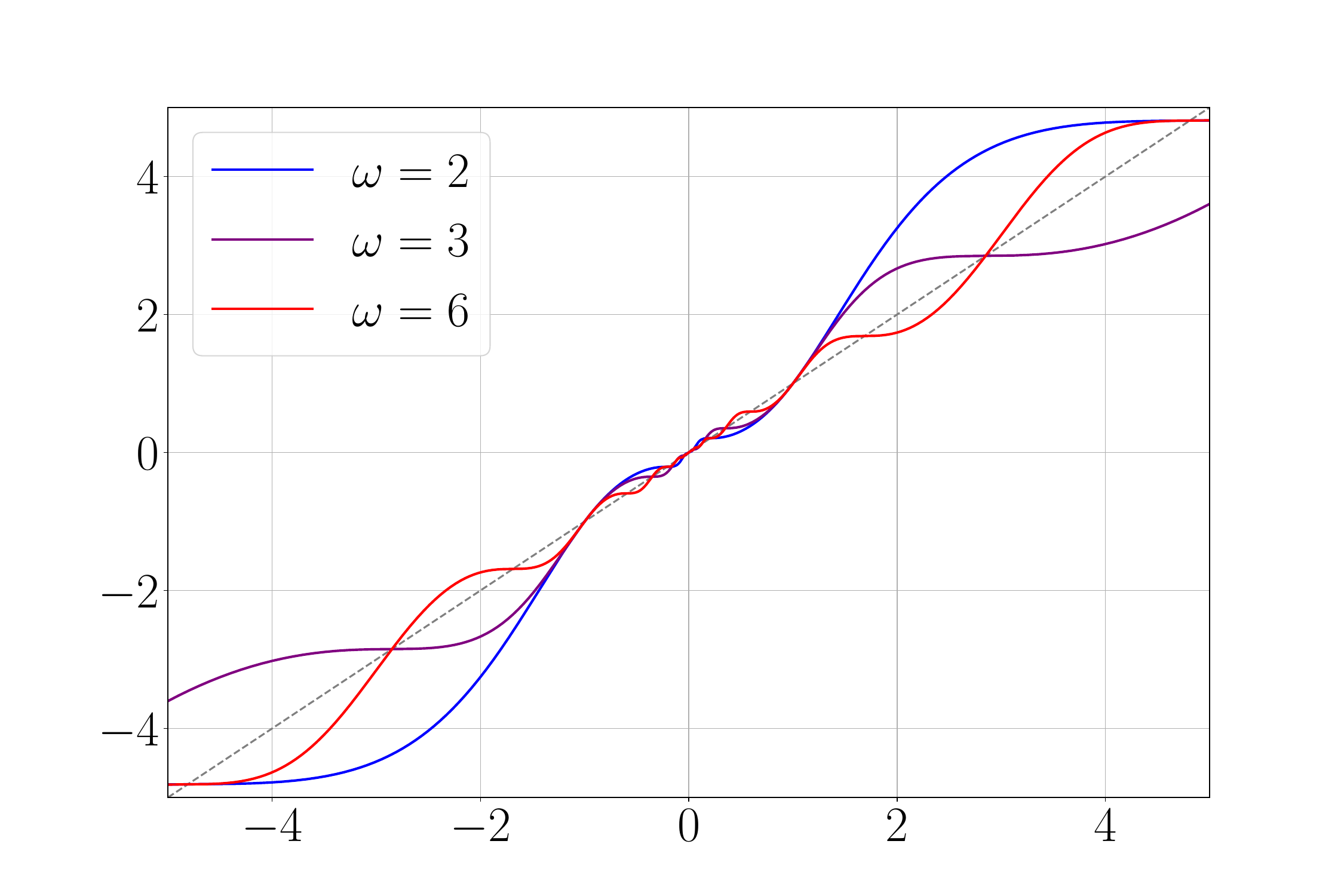}
		\caption{Activation function $\varphi_{\delta, \omega}$.} \label{fig:perlog:act}
	\end{subfigure}
	\begin{subfigure}{.5\linewidth}
		\includegraphics[width=1.05\linewidth]{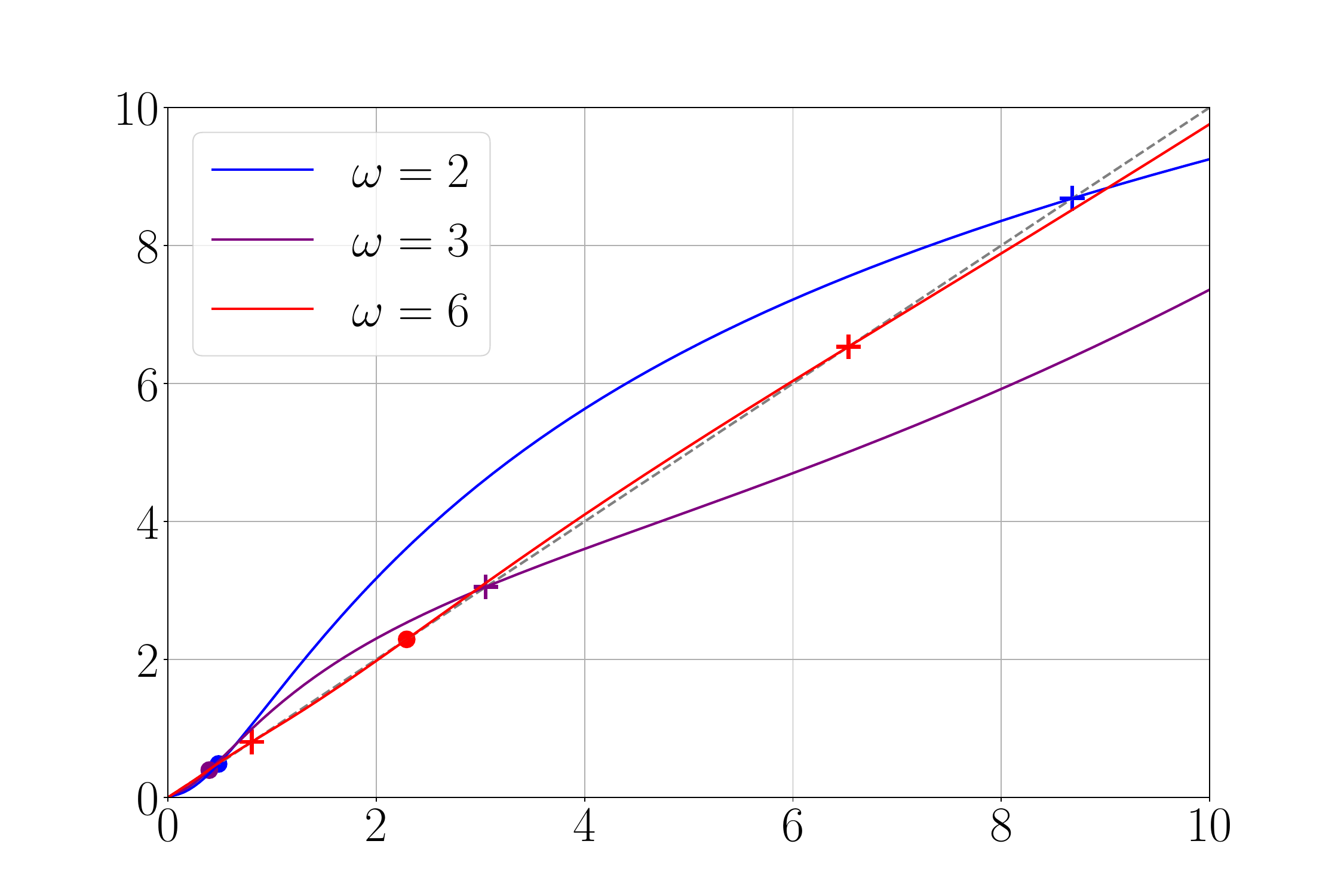}
		\caption{Variance map $\mathcal{V}(\cdot | \sigma_w, \sigma_b = 0)$, linear scale.} \label{fig:perlog:Vv}
	\end{subfigure}

	\begin{subfigure}{.5\linewidth}
		\includegraphics[width=1.05\linewidth]{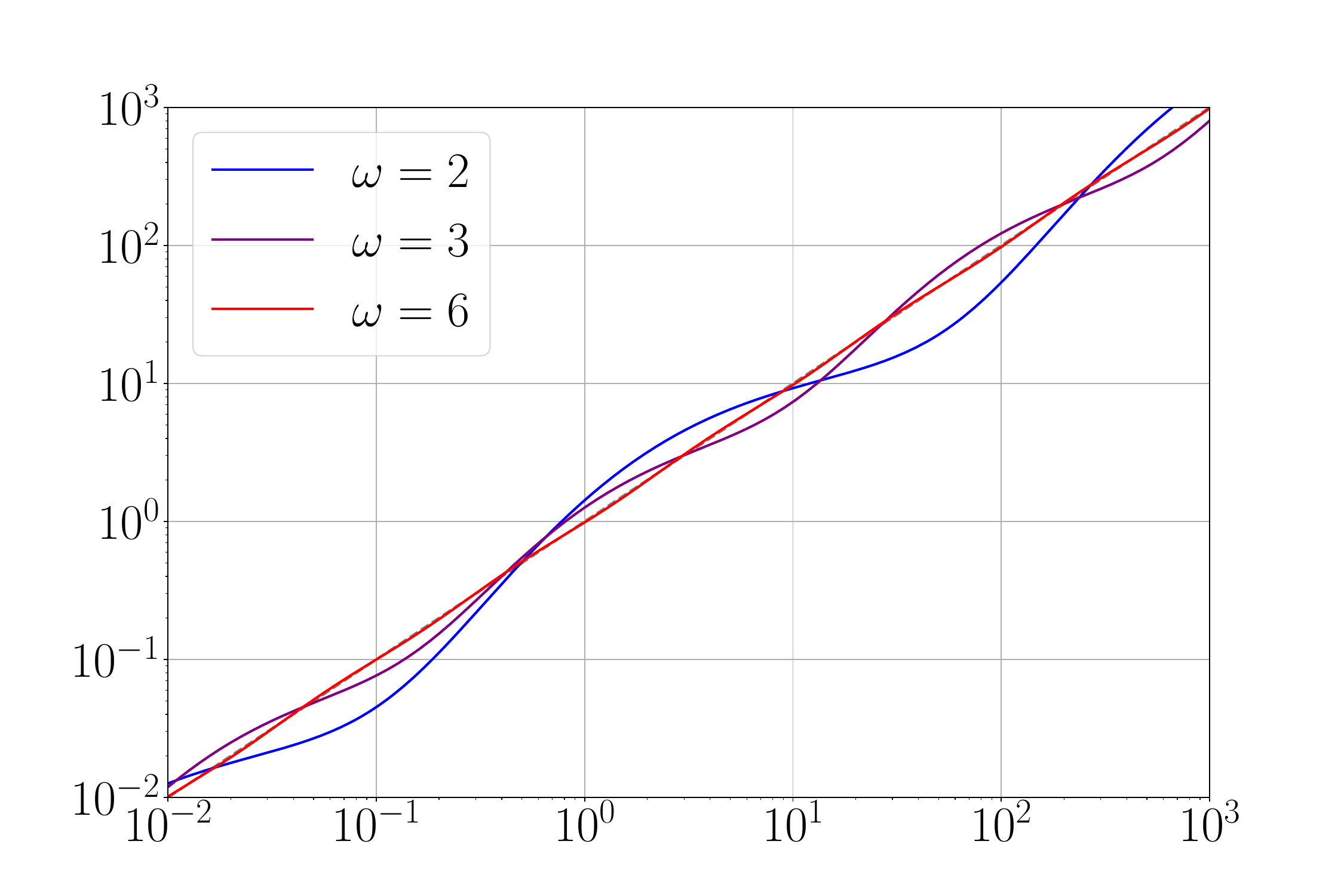}
		\caption{Variance map $\mathcal{V}(\cdot | \sigma_w, \sigma_b = 0)$, log-log scale.} \label{fig:perlog:Vv_log}
	\end{subfigure}
	\begin{subfigure}{.5\linewidth}
		\includegraphics[width=1.05\linewidth]{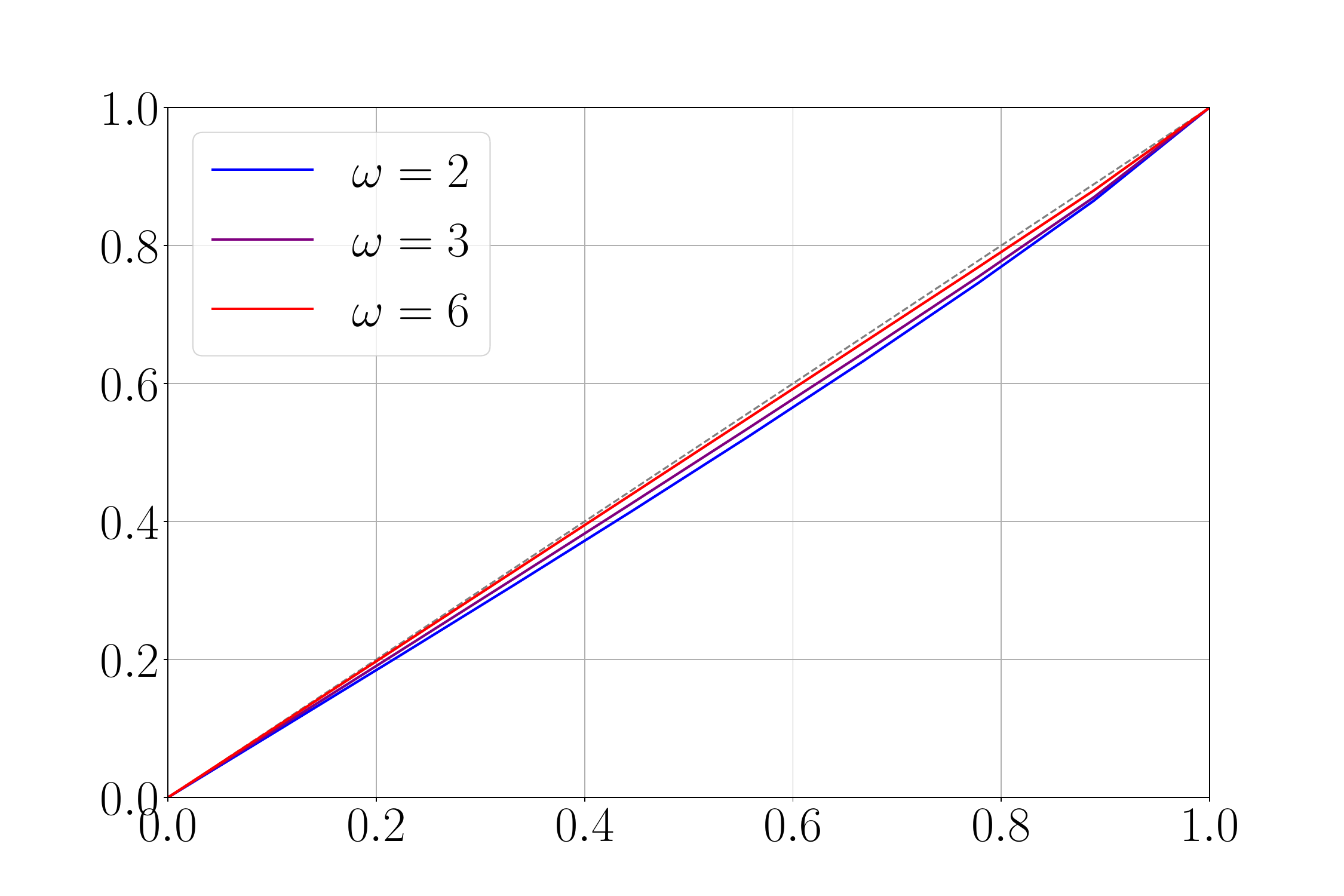}
		\caption{Correlation map $\mathcal{C}(\cdot, v^*, v^* | \sigma_w, \sigma_b)$.} \label{fig:perlog:Cc}
	\end{subfigure}
	\caption{Properties of the proposed counterexamples $\varphi_{\delta, \omega}$ represented in Fig.~\ref{fig:perlog:act} for $\omega\in\{2,3,6\}$ and $\delta = 0.99$. Their variance map $\mathcal{V}$ has an infinite number of fixed points. \\
		In Fig.~\ref{fig:perlog:Vv}, stable points are marked by crosses ($+$), and unstable points by bullets 
		($\bullet$), when they are away from $0$: two stable points appear for $\omega = 6$ (in red).
		As established in Proposition~\ref{prop:counter}, $\sigma_w$ is tuned for every $\omega$ in such a way that $\mathcal{V}$ crosses the identity function an infinite number of times (not visible on the figure). \\
		In Fig.~\ref{fig:perlog:Vv_log} (log-log scale), it is clearer that $\mathcal{V}$ has an infinite number of 
		fixed points, due to regular oscillations (in log-log scale) below and above the identity. \\
		In Fig.~\ref{fig:perlog:Cc}, we show that as $\omega$ grows, the correlation map $\mathcal{C}$ becomes closer to the identity function, which means that the correlation between data points tends to propagate perfectly. \\
		Note: since an infinite number of stable fixed points are available, we have arbitrarily picked one for each 
		$\omega$, denoted by $v^*$. This choice does not affect the plot of the correlation map $\mathcal{C}$, due to 
		the very specific structure of $\varphi_{\delta, \omega}$.
	} \label{fig:perlog}
\end{figure}

\subsection{Maintaining a property of the pre-activations during propagation} \label{sec:prop:commut}

To conclude this section and introduce the next one, we propose a common representation of the various 
methods used to build initialization distributions for the weights $(\mathbf{W}^l)_l$ and biases $(\mathbf{B}^l)_l$.

Several initialization methods 
\citep{glorot2010understanding,he2015delving,poole2016exponential,schoenholz2016deep} 
are based on the same principle: initialization should be done in such a way that some characteristic 
$\kappa^l$ of the distribution of $\mathbf{Z}^l$ is preserved during propagation (e.g., $\kappa^l = 
\mathrm{Var}(\mathbf{Z}^l)$). Intuitively, any change between $\kappa^l$ and $\kappa^{l + 1}$ reflects a loss 
of information between $\mathbf{Z}^l$ and $\mathbf{Z}^{l + 1}$, which impacts propagation or backpropagation. 
For instance, when $\mathrm{Var}(\mathbf{Z}^l) \rightarrow 0$, the network output tends to 
become deterministic, and when $\mathrm{Var}(\mathbf{Z}^l) \rightarrow \infty$, the output tends to forget the 
operations made by the first layers (i.e., the gradients vanish during backpropagation).

More generally, we denote by $\mathcal{D}^l$ the 
distribution associated to the pre-activations $\mathbf{Z}^{l}$, by $\mathcal{T}(\cdot ; n_l, \mathrm{P}_l, 
\phi^l) =: \mathcal{T}_l[\mathrm{P}_l](\cdot)$ the transformation of $\mathcal{D}^l$
performed by layer $l$, that is $\mathcal{D}^{l+1} = \mathcal{T}_l[\mathrm{P}_l](\mathcal{D}^l)$ (where $\mathrm{P}_l$ is the initialization distribution of $(\mathbf{W}^l, 
\mathbf{B}^l)$ and $\phi^l$ is the activation function at layer $l$), and by $\kappa^l := 
\chi(\mathcal{D}^l)$ the characteristic of the distribution 
$\mathcal{D}^l$ we are interested in. 
Then, according to a heuristic of ``information preservation'', it is assumed 
that the sequence $(\kappa^l)_l$ must remain constant, and the 
initialization distributions $(\mathrm{P}_l)_l$ are built accordingly. In some cases, it is possible to build 
a map $\tilde{\mathcal{T}}_{l}{[\mathrm{P}_{l}]}$, so that each $\kappa^{l + 1}$ can be built out of its 
predecessor $\kappa^{l}$, without using all the information we may have on the $(\mathcal{D}^l)_l$.

We summarize this way of building initialization procedures in
Figure~\ref{fig:commut}, and we show how it applies to well-known initialization procedures in 
Table~\ref{tbl:commut}.

\begin{table}[h]
	\begin{center}
	\caption{Examples of $\mathcal{D}^l$, $\phi^l$ and $\chi$ in various setups. 
		Notations: for any vector $\mathbf{x} \in \mathbb{R}^n$, its empirical mean is $\bar{\mathbb{E}} \, \mathbf{x} 
		= \frac{1}{n} \sum_{i = 1}^n x_i$ and is empirical variance is $\overline{\mathrm{Var}} \, \mathbf{x} = 
		\frac{1}{n - 1} \sum_{i = 1}^n (x_i - \bar{\mathbb{E}}\, \mathbf{x})^2$.} 
	\label{tbl:commut}
	{\renewcommand{\arraystretch}{1.15}
	\begin{tabular}{cccccc}
		\toprule
		Method & $\mathcal{D}^l$ & $\phi^l$ & $\chi$ & $\tilde{\mathcal{T}}_l[\mathrm{P}_l](\kappa)$ & Assumption \\
		\midrule
		\citeauthor{glorot2010understanding} & distr.\ of $\mathbf{Z}^l$ & $\mathrm{Id}$ & $\mathrm{Var}$ & $\sigma_w^2 \kappa^2 + 
		\sigma_b^2$ & -- \\
		\citeauthor{he2015delving} & distr.\ of $\mathbf{Z}^l$ & $\mathrm{ReLU}$ & $\mathrm{Var}$ & $\frac{1}{2} \sigma_w^2 \kappa^2 
		+ \sigma_b^2$ & -- \\
		\citetalias{poole2016exponential}, \citetalias{schoenholz2016deep} & distr.\ of $(\mathbf{Z}^l_{a}, 
		\mathbf{Z}^l_{b})$ & $\phi$ & $\mathrm{corr}$ & $\mathcal{C}_*(\kappa)$ & 
		\begin{tabular}{@{}c@{}}$Z^l_{j;a}$, $Z^l_{j;b}$ Gaussian,  \\ $v_{a}^l = v_{b}^l = v^*$\end{tabular} \\
		\midrule
		Ours & distr.\ of $\mathbf{Z}^l$ & $\phi_{\theta}^{\actop}$ & $\mathrm{Id}$ & 
		$\mathcal{T}_l[\mathrm{P}_l](\kappa)$ & 
		\begin{tabular}{@{}c@{}}$\bar{\mathbb{E}} \, \mathbf{x}_a = \bar{\mathbb{E}} \, \mathbf{x}_b = 0$, \\ $\overline{\mathrm{Var}} \, \mathbf{x}_a = \overline{\mathrm{Var}} \, \mathbf{x}_b = 1$\end{tabular} \\
		\bottomrule
	\end{tabular}
	}
	\end{center}
\end{table}

\begin{figure}[h!]
	\begin{center}
		\adjustbox{scale=1.25,center}{%
		\begin{tikzcd}[column sep=4.5em]
			\cdots \arrow[r, "\mathcal{T}_{l - 1}{[\mathrm{P}_{l - 1}]}"]
			& \mathcal{D}^l \arrow[d, "\chi"] \arrow[r, "\mathcal{T}_l{[\mathrm{P}_l]}"] 
			& \mathcal{D}^{l + 1} \arrow[d, "\chi"] \arrow[r, "\mathcal{T}_{l + 1}{[\mathrm{P}_{l + 1}]}"] 
			& \cdots \\ 
			\cdots \arrow[r, dashed, "\tilde{\mathcal{T}}_{l - 1}{[\mathrm{P}_{l - 1}]}"]
			& \kappa^l \arrow[r, dashed, "\tilde{\mathcal{T}}_{l}{[\mathrm{P}_{l}]}"]
			& \kappa^{l + 1} \arrow[r, dashed, "\tilde{\mathcal{T}}_{l + 1}{[\mathrm{P}_{l + 1}]}"]
			& \cdots 
		\end{tikzcd}
		}
		\caption{Building process of the initialization distributions $\mathrm{P}_l$ of 
		the parameters $(\mathbf{W}^l, \mathbf{B}^l)$: (i) the pre-activations-related distribution 
		$\mathcal{D}^l$ passes through a map $\mathcal{T}_l{[\mathrm{P}_l]}$ and becomes $\mathcal{D}^{l + 
		1}$; (ii) some statistical characteristic $\kappa^l$ of $\mathcal{D}^l$ can be computed with a function 
		$\chi$: $\kappa^l = \chi(\mathcal{D}^l)$; (iii) we tune the $(\mathrm{P}_l)_l$ in order to make the 
		sequence $(\kappa^l)_l$ constant.} 
			\label{fig:commut}
	\end{center}
\end{figure}

\begin{remark}
	We can use Figure~\ref{fig:commut} to build new initialization distributions: first, we choose a 
	statistical property of $\mathbf{Z}^l$, which determines $\mathcal{D}^l$ and $\chi$; then, we build a 
	framework in which $\chi(\mathcal{D}^l)$ can be easily computed for every $l$ (e.g., we choose a specific 
	activation function, or we make simplifying assumptions).
\end{remark}

In the following section, we aim to impose Gaussian pre-activations through a specific activation function 
$\phi_{\theta}^{\actop}$ and initialization distribution $\mathrm{P}_{\theta}$. It implies that we would 
preserve perfectly
the distribution $\mathcal{D}^l$ itself: our characteristic is $\chi(\mathcal{D}^l) = \mathcal{D}^l$. That 
way, \emph{all the statistical properties of $\mathcal{D}^l$ are preserved} during propagation.

\section{Imposing Gaussian pre-activations} \label{sec:imposing_th}

In this section, we propose a family of pairs $(\mathrm{P}_{\theta}, \phi_{\theta}^{\actop})$, where 
$\mathrm{P}_{\theta}$ is the distribution of the weights at initialization, $\phi_{\theta}^{\actop}$
is the activation function, and $\theta \in (2, \infty)$ is a parameter, such that the pre-activations $Z_{j}^l$ are 
$\mathcal{N}(0, 1)$ at any layer $l$. Imposing such pre-activations is a way to meet two goals.

First, in Section~\ref{sec:prop:realistic}, we have shown that the Gaussian hypothesis is 
not fulfilled in the case of realistic datasets propagated into a simple multilayer perceptron, and we have 
recalled in the Introduction
that the tails of the pre-activations tend to become heavier when information propagates in a neural network.
By imposing Gaussian pre-activations, we ensure that the Gaussian hypothesis is true, which reconciles the results 
provided in the EOC setup (see Eqn.~\eqref{eqn:rec_v} 
and~\eqref{eqn:rec_c}) and the experiments.%
\footnote{There exists another way to solve this problem: use propagation equations which would take into account the 
	sequence $(n_l)_l$ of layer widths, that is, adopt a non-asymptotic setup, 
	contrary to the process leading to Eqn.~\eqref{eqn:rec_v} and~\eqref{eqn:rec_c}. 
	However, taking into account the whole sequence $(n_l)_l$ would lead to recurrence equations that are far less easy 
	to use than Eqn.~\eqref{eqn:rec_c_simple}. Moreover, a precise characterization of the distributions 
	$(\mathcal{D}^l)_l$ of the pre-activations $(Z^l_1)_l$ may be very difficult since they would not be Gaussian 
	anymore.}

Second, as we recalled in the Introduction, many initialization procedures are based on the preservation of some 
characteristic of the distribution of the pre-activations (see Table~\ref{tbl:commut} and Figure~\ref{fig:commut}).
Usual characteristics are the variance and the 
correlation between data points. By imposing Gaussian pre-activations, we would ensure that \emph{the whole} 
distribution is propagated, and not only one of its characteristics.

Besides, we provide a set of constraints, Constraints~\ref{constr_gaussian}, \ref{constr_2}, \ref{constr_2b}, and \ref{constr_3}, 
that the activation function and the initialization procedure should 
fulfill in order to maintain Gaussian pre-activations at each layer. 

\paragraph{Summary.} Formally, we aim to find a family of pairs $(\mathrm{P}_{\theta}, 
\phi_{\theta}^{\actop})$ such that:
\begin{align*}
\left.
\begin{array}{lc}
Z^{l}_j  \sim \mathcal{N}(0, 1) &\text{i.i.d.} \\
W^{l}_{ij}  \sim \mathrm{P_{\theta}} &\text{i.i.d.} \\
B^{l}_i  = 0  &
\end{array}
\right\rbrace
\Rightarrow 
Z^{l + 1}_i := \frac{1}{\sqrt{n_l}} \mathbf{W}_{i \cdot}^{l} \phi_{\theta}^{\actop}(\mathbf{Z}^{l}) + 
\mathbf{B}^{l} 
\sim \mathcal{N}(0, 1) ,
\end{align*}
where $\mathbf{W}_{i \cdot}^{l}$ is the $i$-th row of the matrix $\textbf{W}^l$. 
In other words, the pre-activations $Z^{l + 1}_i$ remain Gaussian for all $l$.

As a result of the present section, we make the following proposition for $(\mathrm{P}_{\theta}, 
\phi_{\theta}^{\actop})$:
\begin{itemize}
	\item $\mathrm{P_{\theta}}$ is the symmetric Weibull distribution $\mathcal{W}(\theta, 1)$ with scale parameter 1 and shape (or tail) parameter $\theta$, which is obtained by symmetrizing a standard Weibull distribution from $\mathbb{R}^+$ to $\mathbb{R}$ (Generalized Weibull-Tail distributions are detailed in Definition \ref{def:gwt}). The symmetric Weibull distribution $\mathcal{W}(\theta, 1)$ has the following CDF:
	\begin{align}
	F_W(t) &= \frac{1}{2} + \frac{1}{2}\mathrm{sgn}(t) \exp\left( 
	-|t|^{\theta} \right) ; \label{eqn:weibull}
	\end{align}
	\item $\phi_{\theta}^{\actop}$ is computed to ensure that $Z^{l + 1}_i := \frac{1}{\sqrt{n_l}} 
	\mathbf{W}_{i \cdot}^{l} \phi_{\theta}^{\actop}(\mathbf{Z}^{l}) + \mathbf{B}^{l}$ is Gaussian $\mathcal{N}(0, 1)$. 
	In short, the family of functions $\phi_{\theta}^{\actop}$ contains (Fig. \ref{fig:plot_act}):
	\begin{enumerate}[$\bullet$]
		\item a sub-family $(\phi_{\theta}^{\acto})_{\theta}$ of \emph{odd} activation functions, 
		spanning a range of functions from a $\tanh$-like 
		function (as $\theta \rightarrow 2^+$) to the identity function (as $\theta \rightarrow \infty$),
		\item a sub-family $(\phi_{\theta}^{\actp})_{\theta}$ of \emph{positive} activation functions, 
		spanning a range of functions from a sigmoid-like 
		function (as $\theta \rightarrow 2^+$) to a softplus-like function%
		\footnote{The softplus function is defined by: $x \mapsto \log(1 + \exp(x))$.} 
		(as $\theta \rightarrow \infty$).
	\end{enumerate}
	We shall see that $\phi_{\theta}^{\actp}$ preserves the 
	Gaussianity of pre-activations better than $\phi_{\theta}^{\acto}$.
\end{itemize}

In order to obtain this result, we:
\begin{enumerate}
	\item reduce and decompose the initial problem (Section~\ref{sec:solve:decomposing});
	\item find constraints on the initialization distribution of the parameters to justify our choice 
	$\mathrm{P}_{\theta} = \mathcal{W}(\theta, 1)$ (Section~\ref{sec:solve:why_weibull});
	\item compute the distribution $\mathrm{Q}_{\theta}$ of $\phi_{\theta}^{\actop}(Z^l_{j})$ 
	we must choose to ensure 
	Gaussian  pre-activations $Z_i^{l + 1}$, given an initialization distribution $\mathrm{P}_{\theta}$	(Section~\ref{sec:solve:product});
	\item build $\phi_{\theta}^{\acto}$ and $\phi_{\theta}^{\actp}$ 
	from $\mathrm{Q}_{\theta}$ (Section~\ref{sec:solve:function}).
\end{enumerate}

\subsection{Decomposing the problem} \label{sec:solve:decomposing}

By combining Equations \eqref{eqn1:act} and \eqref{eqn1:preact}, the operation performed by each layer is:
\begin{align}
\mathbf{Z}^{l + 1} &= \frac{1}{\sqrt{n_l}} \mathbf{W}^{l} \phi(\mathbf{Z}^{l}) + \mathbf{B}^{l} . 
\label{eqn2:total}
\end{align}

In this subsection, we show that finding the distribution of the weights $\mathbf{W}^{l}$ and the activation function 
$\phi$ in order to have:
\begin{align}
	\forall l \in [1, L], \forall j \in [1, n_l], \quad Z^{l}_j \sim \mathcal{N}(0, 1) ,
\end{align}
can be done if we manage to get:
\begin{align}
	Z := W \phi(X) \sim \mathcal{N}(0, 1), \quad \text{with } X \sim \mathcal{N}(0, 1),
\end{align}
by tuning the distribution of $W$ and the activation function $\phi$.

\paragraph{$Z^{l + 1}_i$ as a sum of Gaussian random variables.} 
First, we focus on the operation made by one layer: if each layer transforms Gaussian inputs 
$Z^{l}_j$ into pre-activations $Z^{l + 1}_i$ that are Gaussian too, then we can ensure that the 
pre-activations remain Gaussian after each layer.
Thus, it is sufficient to solve the problem for one layer. After renaming the variables as $Z \leftarrow Z^{l + 1}_i, 
	W_j \leftarrow W^l_{ij} , 
	X_j \leftarrow Z^l_j , 
	B \leftarrow B_i^l ,
	n \leftarrow n_l$, we have:
\begin{align}
	Z = \frac{1}{\sqrt{n}} \sum_{j = 1}^{n} W_j \phi(X_j) + B. \label{eqn3:simple}
\end{align}
In the rest of this subsection, we assume that the $X_j \sim \mathcal{N}(0, 1)$ are independent.
We discuss the independence hypothesis in Remark~\ref{rem:indep_gaussian} and Appendix~\ref{app:indep_preact}.
We want to build an activation function $\phi$, and distributions for $(W_j)_j$ and $B$ such that $Z 
\sim \mathcal{N}(0, 1)$.

Second, we narrow our search space. According to Equation~\eqref{eqn3:simple}, $Z$ is the sum of a random 
variable $B$ and a number $n$ of i.i.d.\ random variables $W_j \phi(X_j) / \sqrt{n}$. 
Since $Z$ must be $\mathcal{N}(0, 1)$ whatever the value of $n$, it is both convenient and sufficient to check that 
each summand in the right-hand side of Equation~\eqref{eqn3:simple} is Gaussian, that is: 
\begin{align}
	B \sim \mathcal{N}(0, \sigma_b^2), \quad W_j \phi(X_j) \sim 
\mathcal{N}(0, 1 - \sigma_b^2),
\end{align}
with $\sigma_b \in (0, 1)$. In that case, we have $Z \sim \mathcal{N}(0, 1)$. For the sake of simplicity, we 
assume that $B = 0$ with probability $1$, so that we just have to ensure that, for all $j$, $W_j \phi(X_j) 
\sim \mathcal{N}(0, 1)$. If one wants to deal with nonzero bias $B \sim \mathcal{N}(0, \sigma_b^2)$, it is 
sufficient to scale the random variables $W_j \phi(X_j)$ accordingly.

To summarize, we have chosen to build $Z \sim \mathcal{N}(0, 1)$ by ensuring that $W_j \phi(X_j) \sim 
\mathcal{N}(0, 1)$. 
With $B = 0$, this choice is formally imposed by this straightforward proposition.
\begin{proposition}[Lévy-Cramér Theorem] \label{prop:gaussian_sum}
	Let $(Z_j)_j$ be a sequence of $n$ i.i.d.\ random variables. Let $Z = \frac{1}{\sqrt{n}} \sum_{j = 1}^n 
	Z_j$.
	If $Z$ is $\mathcal{N}(0, 1)$, then the distribution of each $Z_j$ is also $\mathcal{N}(0, 1)$.
\end{proposition}
\begin{proof}
	Let $\psi_Z(x) := \mathbb{E}[e^{i Z x}]$ be the characteristic function of the distribution of $Z$.
	Besides, the $(Z_j)_j$ are i.i.d.\ and $Z = \frac{1}{\sqrt{n}} \sum_{j = 1}^n Z_j$, so:
	\begin{align}
	\psi_Z(x) = \exp\left(- \frac{x^2}{2}\right) \quad \text{and} \quad \psi_Z(x) = \left[ 
	\psi_{\frac{Z_1}{\sqrt{n}}}(x) \right]^n .
	\end{align}
	This proves that $\psi_{Z_1}(x) = e^{-x^2/2}$. So, for all $j$ in $[1, n]$, $Z_j 
	\sim \mathcal{N}(0, 1)$.
\end{proof}

As a result, we obtain the first constraint.
\begin{tcolorbox}[colback=blue!10!white,size=title]
\begin{constraint} \label{constr_gaussian}
	If, for all $l$, the weights $(W_{ij}^l)_{ij}$ are i.i.d.\ and independent from the pre-activations 
	$(X_j^l)_j$, which are also supposed to be i.i.d., then we must ensure that:
	\begin{align*}
	\forall l, i, j, \quad W_{ij}^l \phi(X_j^l) \sim \mathcal{N}(0, 1) .
	\end{align*}
\end{constraint}
\end{tcolorbox}

\begin{remark} \label{rem:indep_gaussian}
	The hypothesis of independent inputs $(X_j^l)_j$ truly holds only for the second layer.%
	\footnote{The inputs of the first layer are deterministic.}
	Although the hypothesis of independent $(X_j^l)_j$ is common 
	(see Section \ref{sec:common-assumptions}), it is mostly unrealistic in practice. 
	So, we propose in Appendix~\ref{app:indep_preact} an empirical study of this hypothesis, 
	in order to identify in which cases the dependence between the inputs of one layer 
	damages the Gaussianity of its outputted pre-activations.
\end{remark}

\paragraph{New formulation of the problem.} We have proven that, to ensure that $Z \sim \mathcal{N}(0, 1)$, it 
is sufficient to solve the following problem:
\begin{align}
	\text{find } \mathrm{P} \text{ and } \phi \text{ such that:} \quad X \sim \mathcal{N}(0, 1) \text{ and } W 
	\sim \mathrm{P} \Rightarrow W \phi(X) \sim \mathcal{N}(0, 1) . \label{eqn:problem_main}
\end{align}
In the following subsections, we build a family $\mathcal{P}$ of initialization distributions (Section~\ref{sec:solve:why_weibull}) such that, for any 
$\mathrm{P} \in \mathcal{P}$, there exists a function $\phi$ such that $(\mathrm{P}, \phi)$ is a 
solution to~\eqref{eqn:problem_main}.
We decompose the remaining problem into two parts, by introducing an intermediary random 
variable $Y = \phi(X)$:
\begin{itemize}
	\item for a distribution $\mathrm{P}$, deduce $\mathrm{Q}$ s.t.:
	$W \sim \mathrm{P}, Y \sim \mathrm{Q} \; \Rightarrow \; W Y =: G \sim \mathcal{N}(0, 1)$ (Section~\ref{sec:solve:product});
	\item for a distribution $\mathrm{Q}$, find a function $\phi$ s.t.:
	$X \sim \mathcal{N}(0, 1) \; \Rightarrow \; Y = \phi(X) \sim \mathrm{Q}$ (Section~\ref{sec:solve:function}).
\end{itemize}

\subsection{Why initializing the weights $W$ according to a symmetric Weibull distribution?} 
\label{sec:solve:why_weibull}

We are looking for a family $\mathcal{P}$ of distributions such that, for any $\mathrm{P} \in \mathcal{P}$, there 
exists $\mathrm{Q}$ such that:
\begin{align}
W \sim \mathrm{P}, Y \sim \mathrm{Q} \; \Rightarrow \; W Y =: G \sim \mathcal{N}(0, 1) .
\end{align}

Therefore, the family $\mathcal{P}$ is subject to several constraints.
In this subsection, we present two results indicating that a subset of the family of Weibull distributions
is a good choice for $\mathcal{P}$:
\begin{enumerate}
	\item the density of $W$ at $0$ should be $0$;
	\item $W$ should be a generalized Weibull-tail random variable (see Section~\ref{sec:constraint_2} or  
	\citealp{vladimirova2021bayesian}) with parameter $\theta \in (2, \infty)$. 
\end{enumerate}
In the process, we are able to gather information about the distribution of $|Y|$, 
namely its density at $0$ and the leading power of the $\log$ of its survival function at infinity, respectively:
\begin{align}
f_{|Y|}(0) &= \sqrt{\frac{2}{\pi}} \left[\int_{0}^{\infty} \frac{f_{|W|}(t)}{t} \, \mathrm{d}t\right]^{-1}, \\
\log S_{|Y|}(y) &\propto - y^{1/\left(\frac{1}{2} - \frac{1}{\theta}\right)} .
\end{align}

As a conclusion of this subsection, we consider that the distribution $\mathrm{P} = \mathrm{P}_{\theta}$ of $W$ should 
lie in the following subset of the family of symmetric Weibull distributions (defined at Eqn.~\eqref{eqn:weibull}):
\begin{align*}
\mathcal{P} &:= \{ \mathcal{W}(\theta, 1) : \theta \in \Theta \} , \\
\Theta &:= (2, \infty) .
\end{align*}

\subsubsection{Behavior near $0$} \label{sec:constraint_1}

Since the product $G = W Y$ is meant to be distributed according to $\mathcal{N}(0, 1)$, then we must have $f_{|G|}(0) = \sqrt{\frac{2}{\pi}} \in (0, \infty)$, which is impossible for several choices of distributions for $W$.

\begin{proposition}[Density of a product of random variables at $0$] \label{thm:product_main}
	Let $W, Y$ be two independent non-negative random variables and $Z = W Y$. Let $f_W, f_Y, f_Z$ be their respective 
	densities.
	Assuming that $f_Y$ is continuous at $0$ with $f_Y(0) > 0$, we have:
	\begin{align}
	\text{if } \quad \lim_{w \rightarrow 0} \int_{w}^{\infty}  \frac{f_W(t)}{t} \, \mathrm{d}t = \infty, \quad \text{ 
	then } \quad \lim_{z \rightarrow 0} f_Z(z) = \infty .
	\end{align}
	Moreover, if $f_Y$ is bounded:
	\begin{align}
	\text{if } \quad \int_{0}^{\infty}  \frac{f_W(t)}{t} \, \mathrm{d}t < \infty, \quad \text{ then } \quad 
	f_Z(0) = f_Y(0) \int_{0}^{\infty} \frac{f_W(t)}{t} \, \mathrm{d}t . \label{eqn:product_main_eqn2}
	\end{align}
\end{proposition}

The proof can be found in Appendix~\ref{app:dem_product_main}.

\begin{corollary} \label{thm:product_corr}
	If $f_Y$ and $f_W$ are continuous at $0$ with $f_Y(0) > 0$ and $f_W(0) > 0$, then:
	\begin{align}
	\lim_{z \rightarrow 0} f_Z(z) = \infty .
	\end{align}
\end{corollary}

According to Corollary~\ref{thm:product_corr}, it is impossible to obtain a Gaussian $G$ by multiplying two random 
variables $W$ and $Y$ whose densities are both continuous and nonzero at $0$. So, if we want to manipulate continuous 
densities, we must have either $f_W(0) = 0$ or $f_Y(0) = 0$.

Let us assume that $f_Y(0) = 0$.
We want $Y$ to be the image of $X \sim \mathcal{N}(0, 1)$ through the function $\phi$, where $f_X(0) > 0$. 
So, in order to obtain $Y$ with a zero density at $0$, it is necessary to build a function $\phi$ with 
$\phi'(0) = \infty$ (see Lemma~\ref{lem:infinity_at_zero} in Appendix~\ref{app:vertical}), which is usually not desirable for an activation 
function of a neural network for training stability reasons.%
\footnote{If $\phi'(0) = \infty$ and $\phi$ is $\mathcal{C}^1$ on $\mathbb{R}^*$, then numerical instabilities 
	may occur during training: if a pre-activation $Z_j^l$ approaches $0$ too closely, $\phi'(Z_j^l)$ can explode 
	and damage the training. These instabilities can be handled by gradient clipping 
	\citep{pascanu2013difficulty}.}
So, it is preferable to design $W$ such that $f_{W}(0) = 0$.

\begin{tcolorbox}[colback=blue!10!white,size=title]
\begin{constraint} \label{constr_2}
	To avoid activation functions with a vertical tangent at $0$, the density of the initialization distribution of a 
	weight $W^l_{ij}$ must be $0$ at $0$:
	\begin{align*}
		\forall l, i, j, \quad f_{W_{ij}^l}(0) = 0 .
	\end{align*}
\end{constraint}
\end{tcolorbox}
\vspace{0.2cm}

\begin{remark}
	In the common case of neural networks with activation function $\phi = \tanh$ and weights $W$ initialized according 
	to a Gaussian distribution, if we assume that the Gaussian hypothesis is true, then $f_Y(0) > 0$ and $f_Z(0) > 0$. 
	Thus, Corollary~\ref{thm:product_corr} applies and the density of $Z$ is infinite at $0$.
	
	If $\phi = \mathrm{Id}$, $Z$ is the product of two independent $\mathcal{N}(0, 1)$, whose density is well-known:
	\begin{align}
	f_{Z}(z) &= \frac{K_0(|z|)}{\pi},
	\end{align}
	where $K_0$ is the modified Bessel function of the second kind, which 
	tends to infinity at $0$,
	which illustrates Corollary~\ref{thm:product_corr}.%
	\footnote{Though, even if each $W_j \phi(X_j)$ has an infinite density at $0$, the density at $0$ of the 
		weighted sum $Z = \frac{1}{\sqrt{n}} \sum_{j = 1}^n W_j \phi(X_j) + B$ may be finite. 
		For instance, it occurs when all the $W_j$ and $\phi(X_j)$ are i.i.d.\ and Gaussian.
		But in this case, even if $f_Z(0) < \infty$, it is impossible to recover a Gaussian pre-activation (see Prop.\ 
		\ref{prop:gaussian_sum}).}
\end{remark}

Finally, if Constraint~\ref{constr_2} holds and we want $f_{|G|}(0) = \sqrt{\frac{2}{\pi}}$, then, according to 
Equation~\eqref{eqn:product_main_eqn2}, the following constraint must hold.
\begin{tcolorbox}[colback=blue!10!white,size=title]
\begin{constraint} \label{constr_2b}
	The density of $Y$ at $0$ must have a specific value depending on the distribution of $W$:
	\begin{align*}
		f_{|Y|}(0) = \sqrt{\frac{2}{\pi}} \left[\int_{0}^{\infty} \frac{f_{|W|}(t)}{t} \, \mathrm{d}t\right]^{-1} .
	\end{align*}
\end{constraint}
\end{tcolorbox}

\subsubsection{Behavior of the tail} \label{sec:constraint_2}

We use the results of~\cite{vladimirova2021bayesian} on the \emph{generalized Weibull-tail distributions} and start by recalling useful definitions and properties.

\begin{definition}[Slowly varying function]
	A measurable function $f : (0, \infty) \rightarrow (0, \infty)$ is said to be \emph{slowly varying} if:
	\begin{align}
	\forall a > 0, \quad \lim_{x \rightarrow \infty} \frac{f(ax)}{f(x)} = 1 .
	\end{align}
\end{definition}

\begin{definition}[Generalized Weibull-Tail ($\mathrm{GWT}$) distribution, \citealp{vladimirova2021bayesian}] \label{def:gwt}
	A random variable $X$ is called \emph{generalized Weibull-tail} with parameter 
	$\theta > 0$, or $\mathrm{GWT}(\theta)$, if its survival function $S_X$ is bounded 
	in the following way: 
	\begin{align}
	\forall x > 0, \quad \exp\left(-x^{\theta} f_1(x) \right) \leq S_X(x) \leq \exp\left(-x^{\theta} f_2(x) \right) ,
	\end{align}
	where $f_1$ and $f_2$ are slowly-varying functions and $\theta > 0$.
\end{definition}

\begin{proposition}[\citealp{vladimirova2021bayesian}, Thm.\ 2.2] \label{thm:product_tail}
	The product of two independent non-negative 
	random variables $|W|$ and $|Y|$ which are respectively $\mathrm{GWT}(\theta_W)$ and $\mathrm{GWT}(\theta_Y)$
	is $\mathrm{GWT}(\theta)$, with $\theta$ such that:
	\begin{align}
	\frac{1}{\theta} = \frac{1}{\theta_W} + \frac{1}{\theta_Y} .
	\end{align}
\end{proposition}

We recall that, in our case, $|G| = |W| \cdot |Y|$ is the absolute value of a Gaussian random variable. 
So $|G|$ is $\mathrm{GWT}(2)$.
Thus, if we assume that $|W|$ and $|Y|$ are respectively $\mathrm{GWT}(\theta_W)$ and $\mathrm{GWT}(\theta_Y)$, 
then we have:
\begin{align}
\frac{1}{2} = \frac{1}{\theta_W} + \frac{1}{\theta_Y} .
\end{align}

Therefore we have the following constraint.
\begin{tcolorbox}[colback=blue!10!white,size=title]
\begin{constraint} \label{constr_3}
	The weights $W$ are $\mathrm{GWT}(\theta)$ with $\theta \in \Theta = (2, \infty)$.
\end{constraint}
\end{tcolorbox}

\subsubsection{Conclusion}

Constraints~\ref{constr_2} and~\ref{constr_3} indicate that the distribution $\mathrm{P}$ of the weights $W$:
\begin{itemize}
	\item[(i)] should have a density $f_{W}$ such that $f_W(0) = 0$;
	\item[(ii)] should be $\mathrm{GWT}(\theta)$ with $\theta \in (2, \infty)$.
\end{itemize}
A simple choice for $\mathrm{P}$ matching these two conditions is: $\mathrm{P} = \mathrm{P}_{\theta} = 
\mathcal{W}(\theta, 1)$ with $\theta \in (2, \infty)$, where $\mathcal{W}(\theta, 1)$ is the symmetric Weibull 
distribution, defined in Equation~\eqref{eqn:weibull}. 
Thus, we ensure that $f_{W}(0) = 0$ and $W$ is 
generalized Weibull-tail with a parameter $\theta$ easy to control (see remark below).

\begin{remark} \label{rem:weibull_gwt}
	If $W \sim \mathcal{W}(\theta, 1)$, then $W$ is $\mathrm{GWT}(\theta)$.
\end{remark}

\subsection{Obtaining the distribution of the activations $Y$} \label{sec:solve:product}

Now that the distribution $\mathrm{P}$ of $W$ is supposed to be symmetric Weibull, that is, $\mathrm{P} = 
\mathrm{P}_{\theta} = \mathcal{W}(\theta, 1)$, we are able to look for \emph{odd}
and \emph{positive} activation functions $\phi_{\theta}^{\actop}$
such that:
\begin{align}
W \sim \mathcal{W}(\theta, 1), X \sim \mathcal{N}(0, 1) \; \Rightarrow \; 
W \phi_{\theta}^{\actop}(X) \sim \mathcal{N}(0, 1) . \label{eqn:def_phi_theta}
\end{align}

As a first step, we look for a distribution $\mathrm{Q}_{\theta}$ such that:
\begin{align}
	W \sim \mathrm{P}_{\theta}, Y \sim \mathrm{Q}_{\theta} \; \Rightarrow \; W Y =: G \sim \mathcal{N}(0, 1) .
\end{align}
In order to ``invert'' this equation, it is natural to make use of the Mellin transform.
A comprehensive and historical work about Fourier and Mellin transforms can be found in 
\cite{titchmarsh1937introduction}, 
and a simple application to the computation of the density of the product of two random variables can be found 
in \cite{epstein1948some}.

However, the technique involving the Mellin transform is very difficult to use in this case, 
both analytically and numerically.
Details about the Mellin transform and these difficulties can be found in 
Appendix~\ref{app:mellin}.

\paragraph{Computation of $f_{|Y|}$: hand-designed parameterized function.} 
Thus, inspired by the shape of $f_{|Y|}$ computed via 
the numerical inverse Mellin transform (see Fig.~\ref{fig:mellin_ck_inverse}, App.~\ref{app:mellin:numerical}), we 
build an 
approximation of $f_{|Y|}$ from the family of functions $\{g_{\alpha, \gamma, \lambda_1, \lambda_2} : \alpha, \gamma, 
\lambda_1, \lambda_2 > 0 \}$ with:
\begin{align}
	g_{\Lambda}(x) := g_{\alpha, \gamma, \lambda_1, \lambda_2}(x) := \gamma \alpha \frac{x^{\alpha - 
	1}}{\lambda_1^{\alpha}}  
	\exp\left( - \frac{x^{\alpha}}{\lambda_1^{\alpha}} \right)
	+ \sqrt{\frac{2}{\pi}}\frac{1}{\Gamma\left(1 - \frac{1}{\theta} \right)} \exp\left( - 
	\frac{x^{\theta'}}{\lambda_2^{\theta'}} \right) ,
\end{align}
where $\theta'$ is the conjugate of $\theta$: $\frac{1}{\theta} + \frac{1}{\theta'} = \frac{1}{2}$,
and $\Lambda := (\alpha, \gamma, \lambda_1, \lambda_2)$.
It is clear that, whatever the parameters, $g_{\Lambda}(0) = \sqrt{\frac{2}{\pi}}\left[\Gamma\left(1 - \frac{1}{\theta} 
\right)\right]^{-1}$, which is exactly Constraint~\ref{constr_2b}. Moreover, when $\alpha = 0$, $g$ matches also 
Constraint~\ref{constr_3}.

Then, we optimize the vector of parameters $\Lambda$ with respect to the following loss:
\begin{align}
	\ell(\Lambda) &:= \| \hat{F}_{\Lambda} - F_{|G|} \|_{\infty} \\
	\hat{F}_{\Lambda}(z) &:= \int_{0}^{\infty} F_{|W|}\left(\frac{z}{t}\right) g_{\Lambda}(t)  \, \mathrm{d} t
	\quad \text{(see Eqn.~\eqref{eqn:mellin:cdf})},
\end{align}
where $\hat{F}_{\Lambda}$ is meant to approximate the CDF of the absolute value of a Gaussian $\mathcal{N}(0, 1)$. 
The integral is computed numerically. 
For the loss, we have chosen to compute the $\mathcal{L}^{\infty}$-distance between two CDFs, in order to be consistent 
with the Kolmogorov--Smirnov test we perform in Section~\ref{sec:expe:ks}. Optimization details can be found in 
Appendix~\ref{app:numerical_inv}.

\subsection{Obtaining the activation function $\phi_{\theta}^{\actop}$} \label{sec:solve:function}

In the preceding section, we have computed the distribution of $|Y|$. 
We are restricting ourselves to respectively symmetrical and positive $Y$, whose distributions are 
respectively denoted by $\mathrm{Q}_{\theta}^{\acto}$ and $\mathrm{Q}_{\theta}^{\actp}$
(``$\mathrm{o}$'' for ``odd'' and ``$\mathrm{p}$'' for ``positive'').
Now, we want to build the activation function $\phi_{\theta}^{\actop}$, in order to transform a
pre-activation $G \sim \mathcal{N}(0, 1)$ into an activation $Y = \phi_{\theta}^{\actop}(G)$ distributed
according to $\mathrm{Q}_{\theta}^{\actop}$:
\begin{align*}
	\text{if } G \sim \mathcal{N}(0, 1), \quad \text{then } \phi_{\theta}^{\actop}(G) 
	\sim \mathrm{Q}_{\theta}^{\actop} .
\end{align*}

To compute $\phi_{\theta}^{\actop}$, we will use the following proposition:

\begin{proposition} \label{prop:invF}
	Let $X$ be a random variable such that $F_X$ is strictly increasing on $\mathbb{R}$. 
	Let $\mathrm{Q}$ be a distribution without atoms such that $F_{\mathrm{Q}}$ 
	is strictly increasing on $S := F_{\mathrm{Q}}^{-1}((0, 1))$. 
	Then:
	\begin{align}
	\text{with} \quad \phi(x) := F_{\mathrm{Q}}^{-1}(F_X(x)), \quad 
	\text{we have:} \quad \phi(X) \sim \mathrm{Q}.
	\end{align}
\end{proposition}
\begin{proof}
	Let $\phi(x) := F_{\mathrm{Q}}^{-1}(F_X(x))$, which is a strictly increasing bijection from $\mathbb{R}$ to
	$S$. 
	
	For any $y \in S$, we have:
	\begin{align}
	\mathbb{P}(\phi(X) \leq y) = \mathbb{P}(X \leq \phi^{-1}(y)) = F_X(\phi^{-1}(y)) 
	= F_X(F_X^{-1}(F_{\mathrm{Q}}(y))) = F_{\mathrm{Q}}(y).
	\end{align}
	
	For $y \notin S$, since $S$ is connected, we have either $y \leq \inf S$ or $y \geq \sup S$:
	\begin{itemize}
		\item if $y \notin S$ with $y \leq \inf S$, then $\mathbb{P}(\phi(X) \leq y) 
		\leq \mathbb{P}(\phi(X) \notin S) = 0$;
		\item if $y \notin S$ with $y \geq \sup S$, 
		then $\mathbb{P}(\phi(X) \leq y) \geq \mathbb{P}(\phi(X) \leq \sup S) 
		\geq \mathbb{P}(\phi(X) \in S) = 1$.
	\end{itemize}

	So, for all $y \in \mathbb{R}$, $\mathbb{P}(\phi(X) \leq y) = F_{\mathrm{Q}}(y)$.
\end{proof}

\paragraph{Odd activation function $\phi_{\theta}^{\acto}$.}
We build a symmetric distribution for $Y$:
\begin{align}
F_Y^{\acto}(t) := \frac{1}{2} + \frac{1}{2} \mathrm{sgn}(t) \int_0^{|t|} f_{|Y|}(y) \, \mathrm{d}y .
\end{align}

Since $F_G$ and $F_Y^{\acto}$ are strictly increasing on $\mathbb{R}$, we can use 
Proposition~\ref{prop:invF}:
\begin{align}
\phi_{\theta}^{\acto}(t) := (F_Y^{\acto})^{-1}(F_G(t)).
\end{align}

\paragraph{Positive activation function $\phi_{\theta}^{\actp}$.}
We build a distribution for $Y$ with support in $\mathbb{R}^+$:
\begin{align}
F_Y^{\actp}(t) := \int_0^{\max(0, t)} f_{|Y|}(y) \, \mathrm{d}y .
\end{align}

Since $F_G$ is strictly increasing on $\mathbb{R}$ 
and $F_Y^{\actp}$ is strictly increasing on 
$S := (F_Y^{\actp})^{-1}((0, 1)) = \mathbb{R}^+ \backslash \{0\}$, we can use 
Proposition~\ref{prop:invF}:
\begin{align}
\phi_{\theta}^{\actp}(t) := (F_Y^{\actp})^{-1}(F_G(t)).
\end{align}

\begin{remark} \label{rem:sym_phi}
	There are many other possible choices for the distribution $\mathrm{Q}$ of $Y$,
	resulting in various activation functions other than $\phi_{\theta}^{\actop}$.
	However, among all the possible distributions $\mathrm{Q}$, 
	we identify two ``natural'' usable solutions: 
	$\mathrm{Q}$ symmetric (leading to an odd activation function) 
	or $\mathrm{Q}$ with support in $\mathbb{R}^+$ (leading to a positive activation function).
\end{remark}

\subsection{Results and limiting cases}

We have plotted in Figure~\ref{fig:plot_act} the different distributions related to the computation
of $\phi_{\theta}^{\actop}$ and the functions $\phi_{\theta}^{\actop}$ themselves.
The family of the $\phi_{\theta}^{\acto}$ is a 
continuum spanning unbounded functions from the $\tanh$-like function $\phi_{2^+}$
to the identity function,
while the family of the $\phi_{\theta}^{\actp}$ is a 
continuum spanning functions from a sigmoid-like function
to a softplus-like function, all of which tending to $0$ in $- \infty$ and 
to $+ \infty$ in $+ \infty$.
As a reminder, $\mathrm{softplus}(x) = \log(1 + \exp(x))$.

\begin{figure}[h!]
	{\centering
    \begin{subfigure}{.48\linewidth}
    	\includegraphics[width=1.05\linewidth]{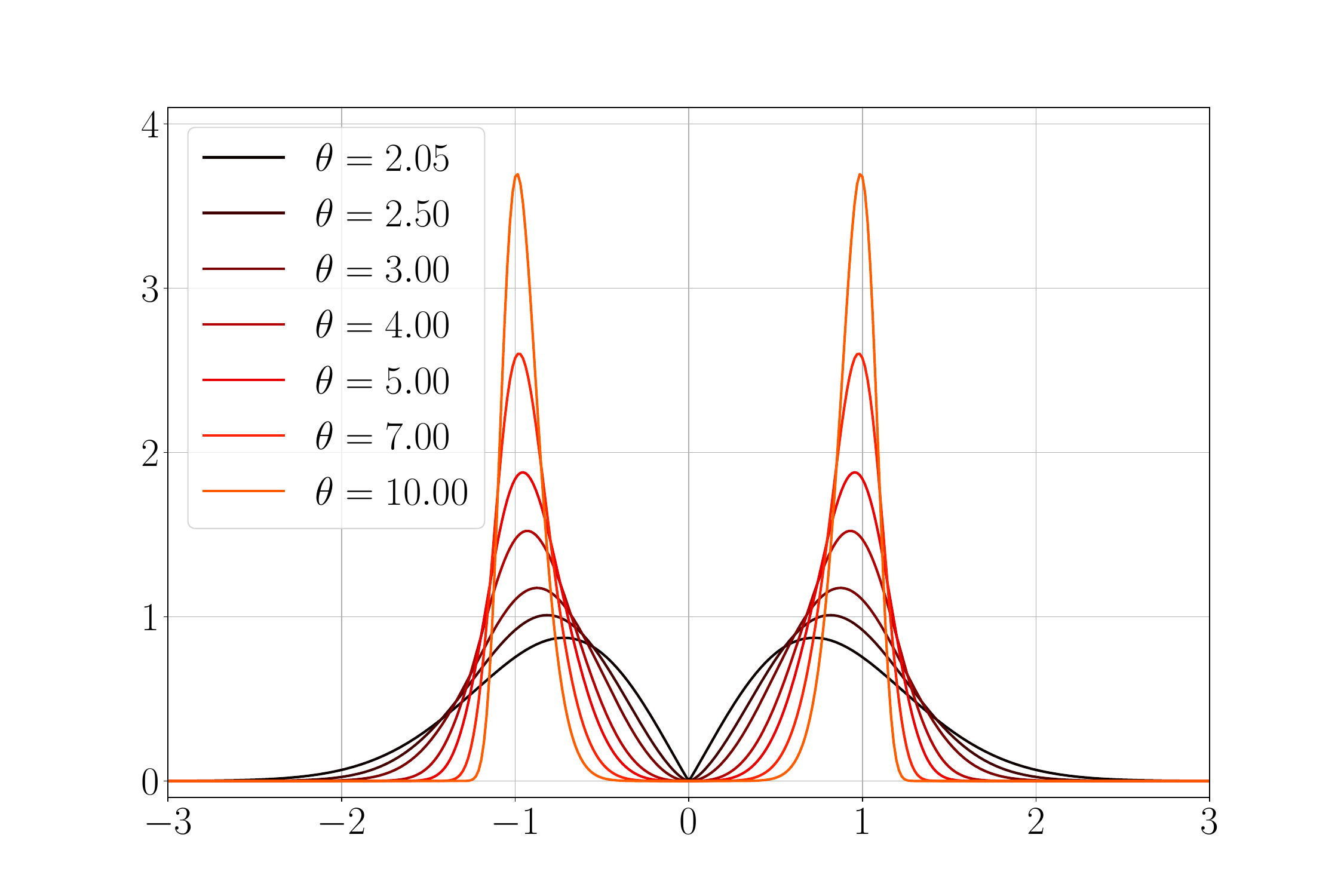}
    	\subcaption{Initialization distribution $\mathrm{P}_{\theta}$ of the weights: symmetric Weibull 
    		$\mathcal{W}(\theta, 1)$.} \label{fig:plot_act:dens_Weibull}
    \end{subfigure}\hfill
    \begin{subfigure}{.48\linewidth}
    	\includegraphics[width=1.05\linewidth]{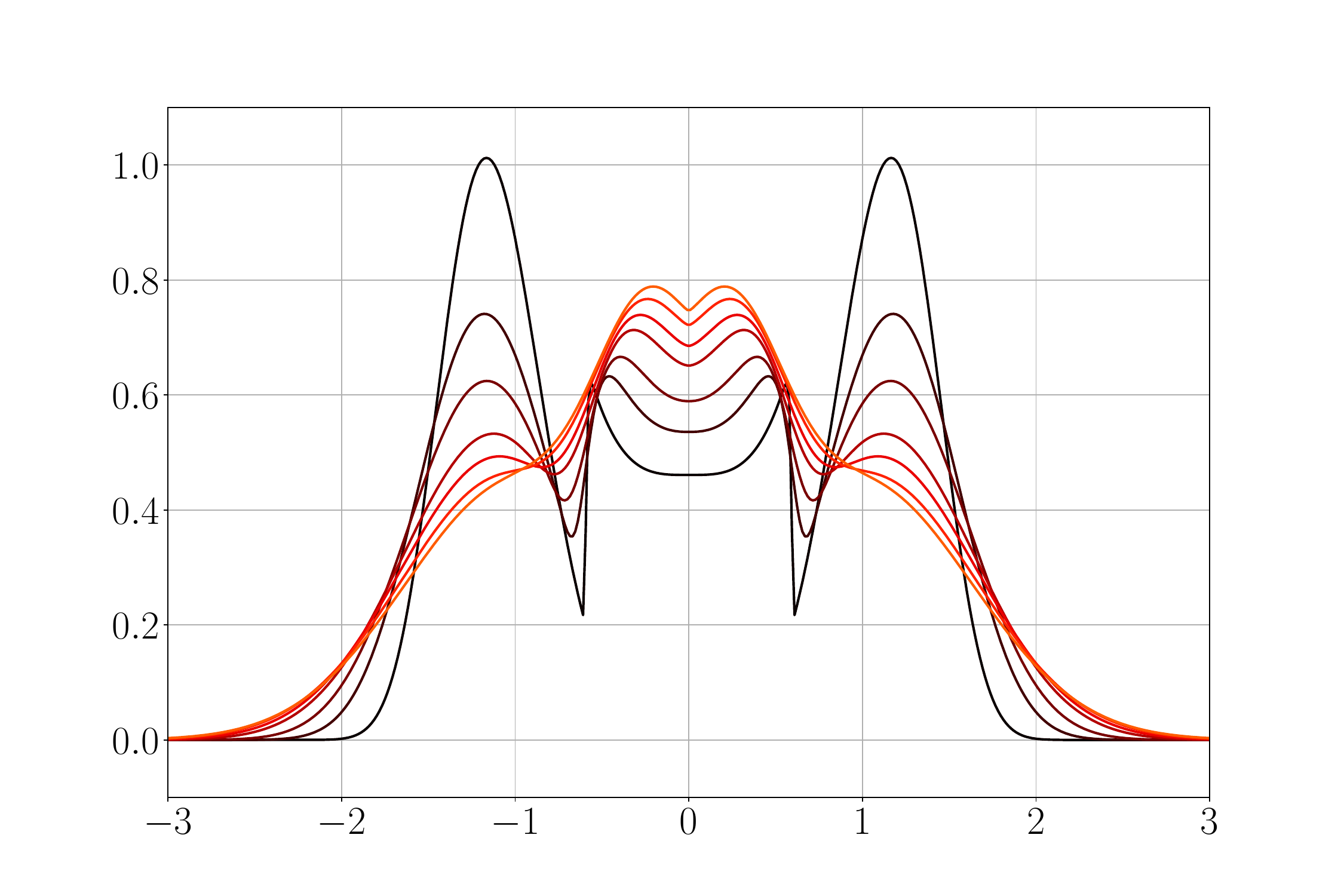}
    	\subcaption{Estimated density $f_Y$ of the distribution $\mathrm{Q}_{\theta}$ of $Y$. \phantom{Estimated density}} 
    	\label{fig:plot_act:dens_Y}
    \end{subfigure}

    \begin{subfigure}{.48\linewidth}
		\includegraphics[width=1.05\linewidth]{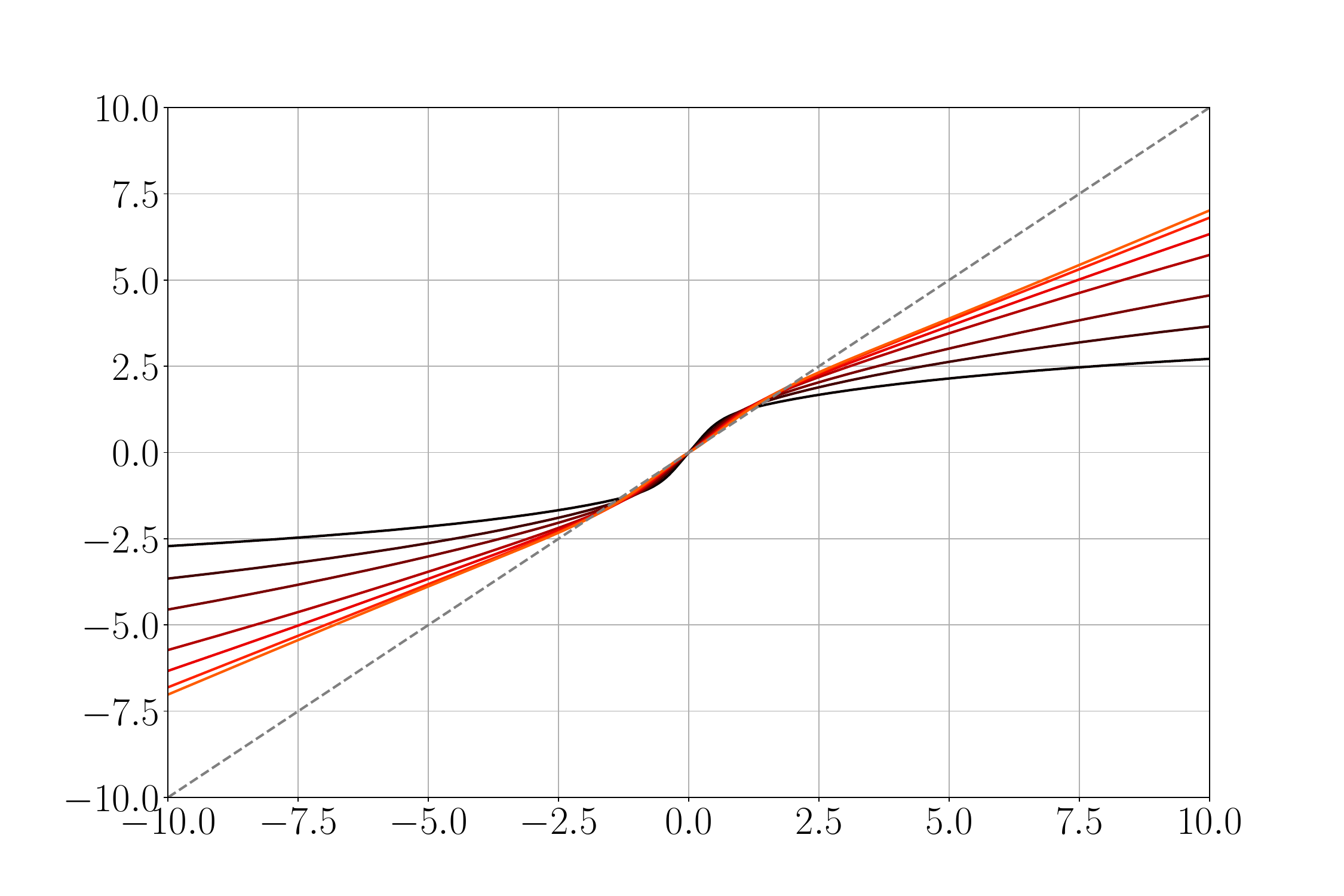}
		\subcaption{Activation function $\phi_{\theta}^{\acto}$.} \label{fig:plot_act:act}
	\end{subfigure}\hfill
	\begin{subfigure}{.48\linewidth}
		\includegraphics[width=1.05\linewidth]{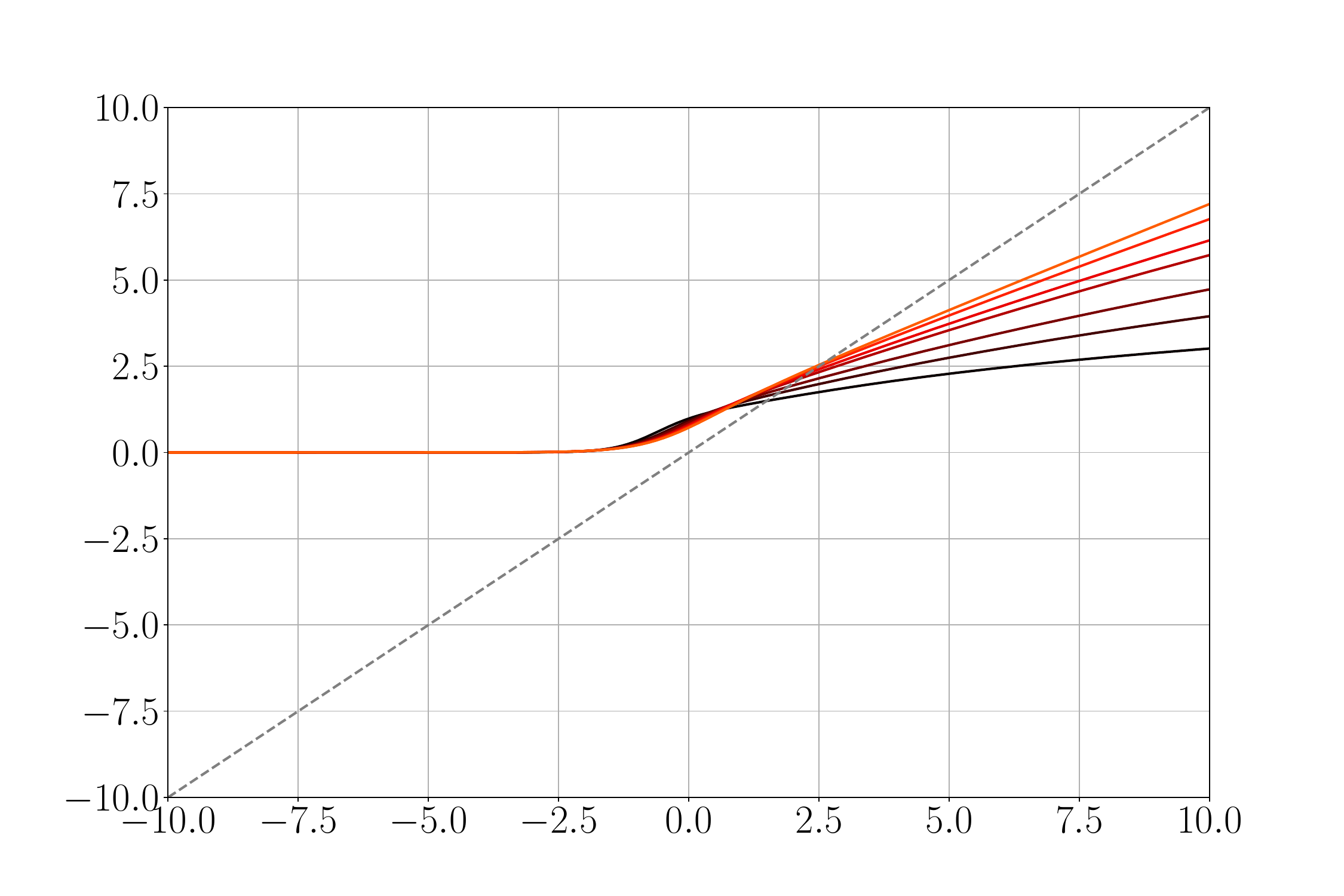}
		\subcaption{Activation function $\phi_{\theta}^{\actp}$.} \label{fig:plot_act:act_zoom}
	\end{subfigure}
\caption{How to build a random variable $W Y = W \phi_{\theta}^{\actop}(X) =: G \sim \mathcal{N}(0, 1)$, where $X \sim 
\mathcal{N}(0, 1)$? (a) Choose the distribution $\mathrm{P}_{\theta}$ of $W$, then (b) deduce the distribution 
$\mathrm{Q}_{\theta}$ of $Y$, and finally (c-d) find $\phi_{\theta}^{\acto}$
and $\phi_{\theta}^{\actp}$.} \label{fig:plot_act}}
\end{figure}

Our construction degenerates into the following two extreme cases at the boundaries of the parameter space 
$\Theta = (2,\infty)$:
\begin{itemize}
\item when $\theta \rightarrow \infty$, we have $\mathrm{P}_{\theta} \overset{d}{\longrightarrow} \mathcal{R}$ and:
	\begin{enumerate}[$\bullet$]
		\item $\phi_{\theta}^{\acto} \rightarrow \mathrm{Id}$ pointwise,
		\item $\phi_{\theta}^{\actp} \rightarrow \phi_{\infty}^{\actp}$ pointwise, 
		with: $\phi_{\infty}^{\actp} : x \mapsto \sqrt{2} \mathrm{erf}^{-1}(\frac{1}{2}
		+ \frac{1}{2} \mathrm{erf}(\frac{x}{\sqrt{2}}))$;
	\end{enumerate}
\item when $\theta \rightarrow 2$, we have $\mathrm{P}_{\theta} \overset{d}{\longrightarrow} \mathcal{W}(2, 1)$;
\end{itemize}
where $\mathcal{R}$ is the Rademacher distribution ($\xi \sim \mathcal{R} \Leftrightarrow \mathbb{P}(\xi = \pm 
1) = \frac{1}{2}$) and $\mathrm{Id}$ is the identity function.

In the limiting case $\theta \rightarrow \infty$, we initialize the weights at $\pm 1$, which corresponds 
to binary ``weight quantization'', used in neural networks 
compression \citep{pouransari2020least}, 
and we use a linear activation function, commonly used in theoretical analyses of neural 
networks \citep{arora2019implicit}.
In the limiting case $\theta \rightarrow 2$, we recover weights with Gaussian tails,
with a $\tanh$-like or $\mathrm{sigmoid}$-like activation function.

\section{Experiments} \label{sec:experiments}

In this section, we test the Gaussian hypothesis with $\relu$, $\tanh$ 
and our activation functions $\phi_{\theta}^{\actop}$, 
after one layer (Section~\ref{sec:expe:ks}) and after several layers (Section~\ref{sec:expe:multilayer}). 
Then, we plot the Edge of Chaos graphs $(\sigma_b, \sigma_w)$, which are exact with $\phi = \phi_{\theta}^{\actop}$ 
(Section~\ref{sec:expe:eoc})
in the case of finite $n_l$ with independent pre-activations.
Finally, in Section~\ref{sec:expe:training}, we show the training trajectories of 
LeNet-type networks and multilayer perceptrons, when using $\tanh$, $\relu$, and various activation functions we have 
proposed. 

In the following subsections, when we use $\phi = \tanh$ or $\relu$, we initialize the 
weights according to a Gaussian at the EOC. This is, for $\tanh$: $\sigma_b^2 = 0.013$, $\sigma_w^2 = 1.46$; and for
$\relu$: $\sigma_b^2 = 0$, $\sigma_w^2 = 2$.
When we use $\phi = \phi_{\theta}^{\actop}$, we initialize the weights according to $\mathcal{W}(\theta, 1)$.

\paragraph{Notation for the activation functions.} We recall that: $\phi$ denotes an arbitrary activation function; 
$\varphi_{\delta, \omega}$ denotes the function defined in Definition~\ref{def:activ-phi}; 
$\phi_{\theta}^{\actop}$ denotes the odd or the positive function verifying Eqn.~\eqref{eqn:def_phi_theta}.

\subsection{Testing the Gaussian hypothesis: synthetic data, one layer} \label{sec:expe:ks}

Most importantly, we must experimentally verify that our family of initialization distributions and activation functions 
$\{(\mathrm{P}_{\theta}, \phi_{\theta}^{\actop}) : 
\theta \in (2, \infty) \}$
can reliably produce Gaussian pre-activations.

\paragraph{Framework.} First, we test our setup in the one-layer neural network case with 
synthetic inputs. More formally, we consider a $\mathcal{N}(0, 1)$ \emph{pre-input}%
\footnote{Such a pre-input plays the role of the pre-activation outputted by a hypothetical preceding layer.}
$\mathbf{Z} \in \mathbb{R}^{n}$, which is meant to be first transformed by the activation function 
$\phi_{\theta}^{\actop}$ (hence the name ``\emph{pre}-input''), then multiplied by a matrix of weights $\mathbf{W} \in 
\mathbb{R}^{1 \times n}$.
This one-neuron layer outputs a scalar $Z'$:
\begin{align}
	Z' := \frac{1}{\sqrt{n}} \mathbf{W} \phi(\mathbf{Z}), \quad \text{ with } Z_j \sim \mathcal{N}(0, 
	1) \text{ and } W_{1j} \sim \mathrm{P}_{\theta} \text{ for all } j.
\end{align}
We want to check that the distribution $\mathrm{P}'$ of $Z'$ is equal to $\mathcal{N}(0, 1)$.

\paragraph{Experimental results.}
For that, we use of the Kolmogorov--Smirnov (KS) test \citep{kolmogoroff1941confidence,smirnov1948table} (see 
Appendix~\ref{app:ks_test}).
We perform the KS test within two setups: with and without preliminary standardization of the sets of samples.
With a preliminary standardization, we perform the test on $(\bar{Z}'_1, \cdots , \bar{Z}'_s)$:
\begin{align}
	\bar{Z}_k' = \frac{Z_k' - \bar{\mu}}{\bar{\sigma}}, \quad \bar{\mu} = \frac{1}{s} \sum_{k = 1}^s Z'_k, \quad
	\bar{\sigma}^2 = \frac{1}{s - 1} \sum_{k = 1}^{s} \left(Z_k' - \bar{\mu}\right)^2 .
\end{align}
For the sake of simplicity, let us denote by $\hat{F}_{Z'} := F_s$ the empirical CDF of $Z'$, computed with the $s$ 
data samples $(Z_1', \cdots , Z_s')$, and let $\hat{F}_{\bar{Z}'}$ be the empirical CDF of standardized $Z'$, computed 
with $(\bar{Z}_1', \cdots , \bar{Z}_s')$.

We have plotted in Figure~\ref{fig:ks_test} the KS statistic of the distribution of the output $Z'$, when using our 
activation functions $\phi_{\theta}^{\actop}$, $\tanh$ and $\relu$. 
Our sample size is $s = 10^7$.
A small KS statistic corresponds to a configuration where $Z'$ is close to being $\mathcal{N}(0, 1)$. If a point is 
above the KS threshold (green line, dotted), then the Gaussian hypothesis is rejected with $p$-value $0.05$.

When we perform standardization (Fig.~\ref{fig:ks_test:with_std}), the neurons using $\phi_{\theta}^{\actop}$ 
output always a pre-activation $Z'$ that is closer to $\mathcal{N}(0, 1)$ than with $\relu$ or $\tanh$. 
But, despite this advantage, the Gaussian hypothesis should be rejected with $\phi_{\theta}^{\actop}$ 
when the neuron has a very small number of inputs ($n < 30$). 

In Figure~\ref{fig:ks_test:no_std}, we compare directly the distribution of $Z'$ to $\mathcal{N}(0, 1)$. 
This test is harder than testing the Gaussian hypothesis because the variance of $Z'$ must be equal to $1$. 
We observe that when using $\phi_{\theta}^{\actop}$, the KS statistic remains above the threshold (while it is still 
below $10^{-2}$, and even below $6 \cdot 10^{-3}$ for $n \geq 3$).
This result is due to the fact that our 
computation of $\phi_{\theta}^{\actop}$ is only approximate (see Section~\ref{sec:solve:function}). 

\begin{remark}
	Our sample size ($s = 10^7$) is very large, which lowers the threshold of
	rejection of the Gaussian hypothesis. We have chosen this large $s$ to reduce the noise of the KS statistics.
	If we had chosen $s = 18000$, a threshold close to $10^{-2}$ would have resulted, which is higher than any of 
	the KS statistics computed with $\phi = \phi_{\theta}^{\actop}$. One will also note that, in 
	Section~\ref{sec:expe:multilayer}, we use only $s = 10000$ samples to keep a reasonable computational cost.
\end{remark}

\begin{figure}[h!]
	\begin{subfigure}{.5\linewidth}
		\includegraphics[width=1\linewidth]{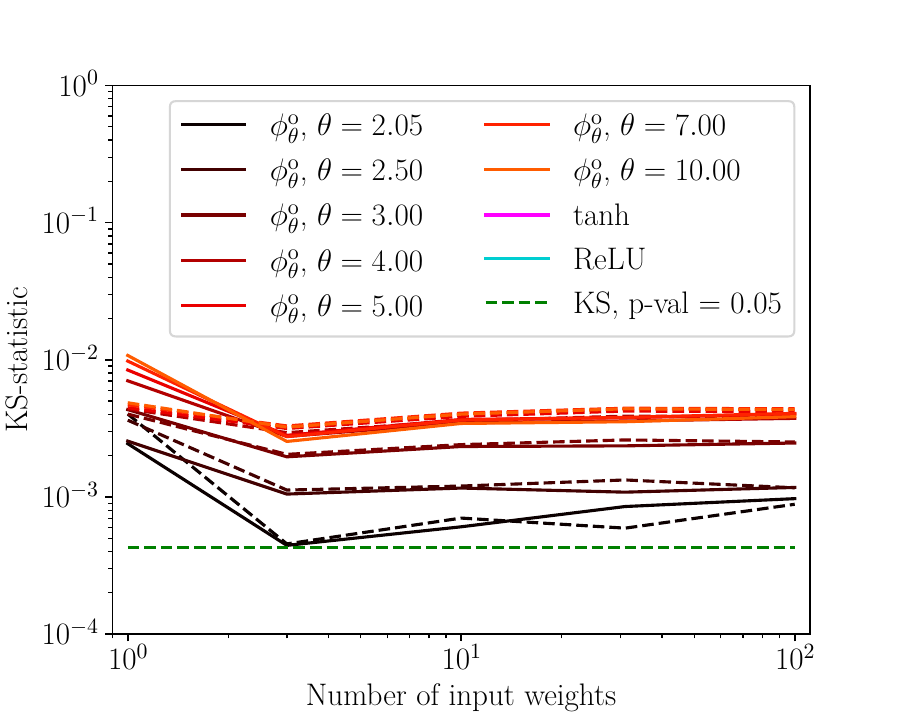}
		\subcaption{KS-statistic of $\hat{F}_{Z'}$.} 
		\label{fig:ks_test:no_std}
	\end{subfigure}
	\begin{subfigure}{.5\linewidth}
		\includegraphics[width=1\linewidth]{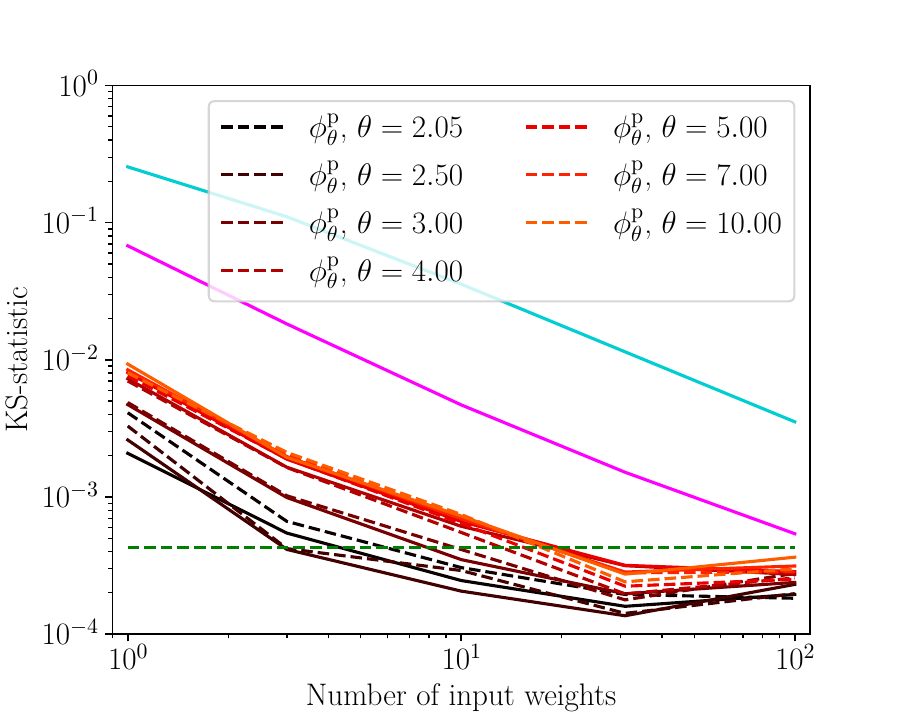}
		\subcaption{KS-statistic of $\hat{F}_{\bar{Z}'}$ (standardized samples).} 
		\label{fig:ks_test:with_std}
	\end{subfigure}
	\caption{Evolution of the KS statistic of the distribution of $Z'$ (Fig.~\ref{fig:ks_test:no_std}) and the 
		standardized distribution of $Z'$ (Fig.~\ref{fig:ks_test:with_std}), 
		with a number of inputs $n \in \{1, 3, 10, 30, 100\}$.}
	\label{fig:ks_test}
\end{figure}

\begin{remark}
	Since $\tanh$ and $\relu$ have not been designed such that $Z' \sim \mathcal{N}(0, 1)$, we did not plot 
	the related non-standardized KS statistics in Figure~\ref{fig:ks_test:no_std}.
	Anyway, given the standard deviation of $Z'$ when using $\phi = \tanh$ or $\relu$ (see Table~\ref{tab:ks_stddev}, 
	Appendix~\ref{app:expe:ks}), very large KS statistics are expected in this setup. 
\end{remark}

We discuss the limits of the KS test in Appendix~\ref{app:ks_test}.

\subsection{Testing the Gaussian hypothesis: CIFAR-10, multilayer perceptron} \label{sec:expe:multilayer}

Now, we test our setup on a multilayer perceptron with CIFAR-10, which is more realistic. 
We show in Figure~\ref{fig:testing_prop} how the distribution of the pre-activations propagates in a multilayer 
perceptron, for different layer widths $n_l \in \{10, 100, 1000\}$. 

\paragraph{Setup.}
Let $\mathcal{D}^l$ be the distribution of the pre-activation $Z^l_1$ after layer $l$.
For all $l$, let us define $(Z^l_{1;k})_{k \in [1, s]}$, a sequence of i.i.d.\ samples drawn
from $\mathcal{D}^l$.
The plots in Figure~\ref{fig:testing_prop} show the evolution of the $\mathcal{L}^{\infty}$ distance $\| \hat{F}_{Z^l_1} - F_G \|_{\infty}$ between the CDF
of $\mathcal{N}(0, 1)$, $F_G$, and the empirical CDF of $\mathcal{D}^l$, $\hat{F}_{Z^l_1}$, 
built with $s = 10000$ samples $(Z^l_{1;k})_{k \in [1, s]}$.

We have built the plots of Figure~\ref{fig:testing_prop} with the same input data point.%
\footnote{In the PyTorch implementation of the training set of CIFAR-10: data point \#47981 (class = plane).
This data point has been chosen randomly.}
See Appendix~\ref{app:expe:prop_comp} for a comparison of the propagation between different data points.
Also, the propagated data point has been normalized 
according to the whole 
training dataset, that is:
\begin{definition}[Input normalization over the whole dataset]
\label{def:whole_dataset}
	We build the normalized data point $\hat{\mathbf{x}}$:
	\begin{align}
	\hat{x}_{a;ij} := \frac{x_{a;ij} - \mu_i}{\sigma_i} , \quad
	\mu_i := \frac{1}{p_i d} \sum_{\mathbf{x} \in \mathbb{D}} \sum_{j = 1}^{p_i} x_{ij}, \quad
	\sigma_i^2 := \frac{1}{p_i d - 1} \sum_{\mathbf{x} \in \mathbb{D}} \sum_{j = 1}^{p_i} (x_{ij} - \mu_i)^2 ,
	\end{align}
	where $x_{a;ij}$ is the $j$-th component of the $i$-th channel of the input image $\mathbf{x}_a$, $p_i$ is the size 
	of the $i$-th channel, and $d$ is the size of the dataset $\mathbb{D}$.
\end{definition}

\begin{figure}[p!]
	\begin{subfigure}{1.\linewidth}
		\includegraphics[width=1.\linewidth]{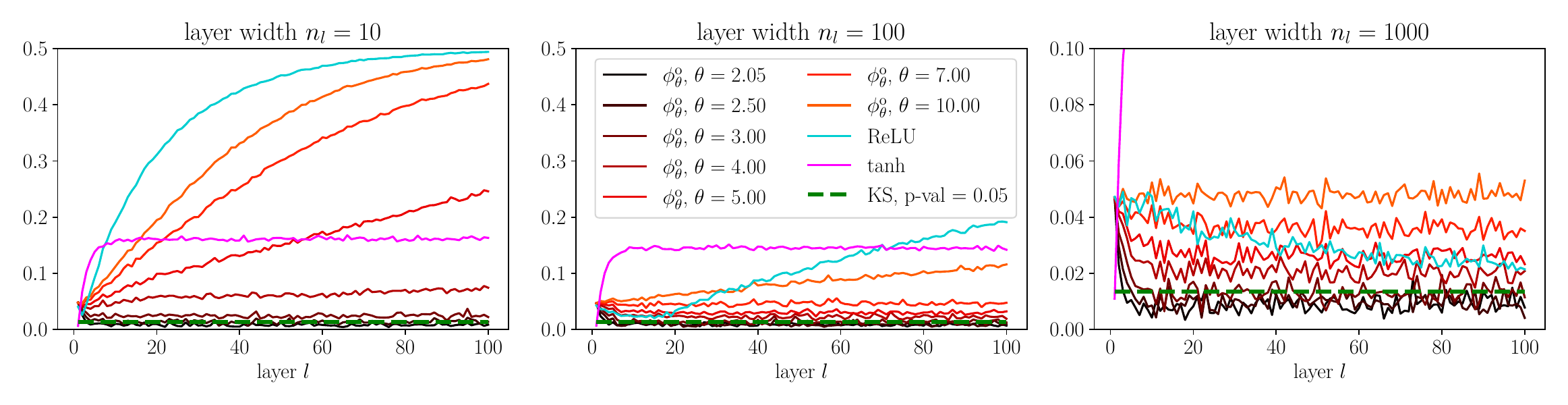}
		\subcaption{Distance $\|\hat{F}_{Z^l_1} - F_G\|_{\infty}$; 
			odd activation function $\phi^{\acto}_{\theta}$;
			inputs normalized over the dataset (Definition 
			\ref{def:whole_dataset}).}
		\label{fig:testing_prop:3}
	\end{subfigure}
	
	\begin{subfigure}{1.\linewidth}
		\includegraphics[width=1.\linewidth]{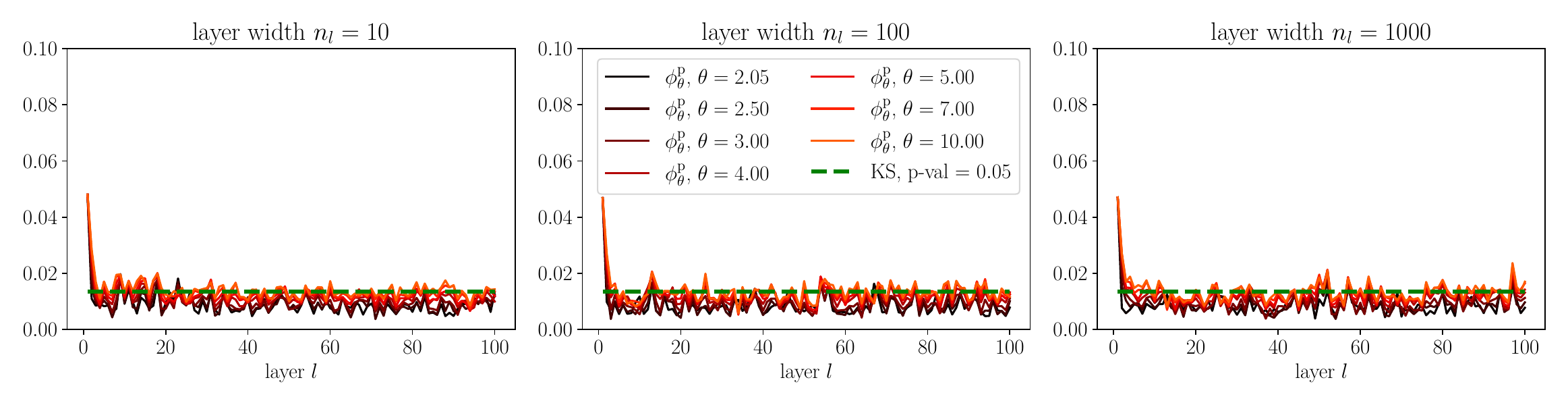}
		\subcaption{Distance $\|\hat{F}_{Z^l_1} - F_G\|_{\infty}$; 
			positive activation function $\phi^{\actp}_{\theta}$;
			inputs normalized over the dataset (Definition 
			\ref{def:whole_dataset}); \textbf{these figures have been zoomed around $0$}.} 
		\label{fig:testing_prop:1}
	\end{subfigure}
	
	\caption{Evolution of the distance between $\mathcal{D}^l$ and the standard Gaussian,
		during propagation where $l$ varies from $1$ to $100$. 
		If the pre-activations are Gaussian, these curves should remain close to zero.
		The dotted green line is the KS threshold: any point above it corresponds to a distribution 
		$\mathcal{D}^l$ for which the Gaussian hypothesis should be rejected with $p$-value $0.05$.
		Weight initialization is $\mathcal{W}(\theta, 1)$ when using $\phi = \phi^{\actop}_{\theta}$ and is Gaussian 
		according 
		to the EOC when using $\phi = \tanh$ or $\mathrm{ReLU}$.} 
	\label{fig:testing_prop}
\end{figure}

\paragraph{Results.}
First, in all cases, $\tanh$ leads to pre-activations that remain far from the standard Gaussian. 
Also, with $10$ and $100$ neurons per layer and $\relu$, the pre-activations tend to diverge from the standard Gaussian.

Second, our proposition of activation functions $\phi = \phi^{\actop}_{\theta}$ leads to various results, 
depending on the layer width $n_l$ and the subfamily to which it belongs, $(\phi^{\acto}_{\theta})_{\theta}$ 
or $(\phi^{\actp}_{\theta})_{\theta}$.
On one hand, the positive activation functions $\phi^{\actp}_{\theta}$ 
lead consistently to standard Gaussian pre-activations for all $\theta \in [2.05, 10]$
(Fig.\ \ref{fig:testing_prop:1}).
On the other hand, the odd activation functions $\phi^{\acto}_{\theta}$ 
lead consistently to Gaussian pre-activations only for small $\theta$
(Fig.\ \ref{fig:testing_prop:3}).

\paragraph{Conclusion.}
For all tested widths, combining a positive activation function $\phi_{\theta}^{\actp}$ 
and weights sampled from $\mathcal{W}(\theta, 1)$ leads to pre-activations that are close
to $\mathcal{N}(0, 1)$, and obviously closer to $\mathcal{N}(0, 1)$ than with $\tanh$ or $\relu$.
However, our proposition of odd activation functions $\phi_{\theta}^{\acto}$
leads to a distribution of pre-activations that drifts away from $\mathcal{N}(0, 1)$
along the propagation, at least for $\theta > 2.5$.
These results are confirmed with various inputs, as shown in Figures
\ref{fig:init:comparison_data_pts} and \ref{fig:init:comparison_data_pts_pos}
in Appendix \ref{app:expe:prop_comp}.

So, in the case $\phi = \phi_{\theta}^{\actp}$, 
$\mathcal{D}^l = \mathcal{N}(0, 1)$ is close to a 
stable fixed point of the recurrence relation $\mathcal{D}^{l + 1} = 
\mathcal{T}_l[\mathrm{P}_l](\mathcal{D}^l)$ with the notations of Section~\ref{sec:prop:commut} (see Fig.~\ref{fig:commut}).
But, in the case $\phi = \phi_{\theta}^{\acto}$, $\mathcal{D}^l = \mathcal{N}(0, 1)$
is not a stable fixed point.
This deviation from the expected result is likely to 
be due to the assumption of independent pre-activations, which we used 
in the derivation of both $\phi^{\acto}$ and $\phi^{\actp}$. This assumption is
discussed in Appendix \ref{app:indep_preact}, where we show 
that such dependence plays a more important role with $\phi^{\acto}$ than with $\phi^{\actp}$.

This search for stable fixed points in the recurrence relation $\mathcal{D}^{l + 1} = 
\mathcal{T}_l[\mathrm{P}_l](\mathcal{D}^l)$ is closely related to the discovery of stable fixed points for the sequence 
of variances $(v_{a}^l)_l$ and the sequence of correlations $(c^l_{ab})_l$ 
\citep{poole2016exponential,schoenholz2016deep}, and may be explored further in future works.%
\footnote{The fixed points of $\mathcal{D}^{l + 1} = \mathcal{T}_l[\mathrm{P}_l](\mathcal{D}^l)$
can also be seen as stationary distributions of the Markov chain $(Z^l_1)_l$, if the layer width $n_l$ and the 
initialization distribution $\mathrm{P}_l$ are constant.}

\subsection{Non-asymptotic Edge of Chaos} \label{sec:expe:eoc}

In Figure~\ref{fig:at_the_eoc}, we show the Edge of Chaos graphs for several activation functions: $\tanh$ and 
$\mathrm{ReLU}$ on one side, and our family $(\phi_{\theta}^{\actop})_{\theta}$ on the other side. We remind that each graph 
corresponds to a family of initialization standard deviations $(\sigma_w, \sigma_b)$ such that the sequence of 
correlations $(c_{ab}^l)_l$ converges to $1$ at a \emph{sub-exponential} rate (see Section~\ref{sec:prop:prop}, Point 
2).
Such choices ensure that the initial correlation between two inputs changes slowly so that the information contained 
in these inputs is lost at the slowest possible rate.

Instead of assuming that the pre-activations are Gaussian as an effect of the ``infinite-width limit'' and the Central 
Limit Theorem, we claim that, with our activation functions $\phi_{\theta}^{\acto}$, for any layer widths (including 
narrow layers and networks with various layer widths), the Edge of Chaos is \emph{non-asymptotic}. Therefore, 
the corresponding curves in Figure~\ref{fig:at_the_eoc} hold for realistic networks.

\begin{figure}[h!]
\begin{center}
    \includegraphics[width=.8\linewidth]{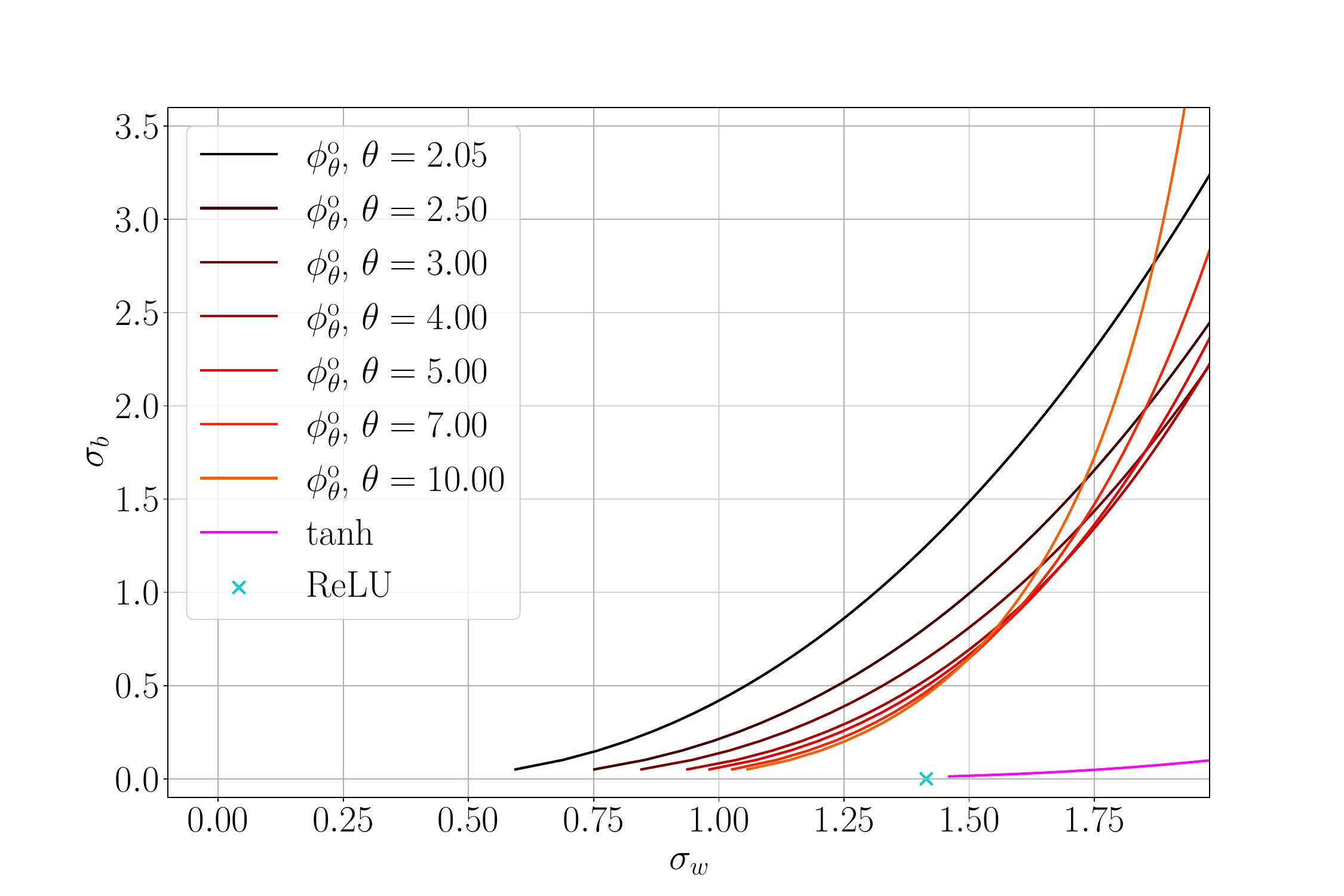}
\end{center}
	\caption{Edge of Chaos for several activation functions. The ordered phase (resp. chaotic) lies above (resp. below) each EOC curve.} 
	\label{fig:at_the_eoc}
\end{figure}

\begin{remark}
	For the $\mathrm{ReLU}$ activation function, the EOC graph reduces to one point. 
	One can refer to \citet[Section 3.1]{hayou2019impact} for a complete study of the EOC of ``$\relu$-like 
	functions'', that is, 
	functions that are linear on $\mathbb{R}^+$ and on $\mathbb{R}^-$ with possibly different factors.
\end{remark}

Also, as in Figure~\ref{fig:corr_main} (see Section~\ref{sec:prop:realistic}), 
we have plotted the propagation of the correlations $(C_{pq}^l)_l$ with $\phi = 
\phi_{\theta}^{\acto}$ and weights sampled from $\mathcal{W}(\theta, 1)$ in Appendix~\ref{app:expe:extra:phi}.

\subsection{Training experiments} \label{sec:expe:training}

Finally, we compare the performance of a trained neural network when using 
$\tanh$ and $\mathrm{ReLU}$, and our activation functions. 
Although the EOC framework (and ours) do not provide any quantitative
prediction about the training trajectories, it is necessary to analyze them in order to enrich the theory.

This section starts with a basic 
check of the training and test performances on a common task: training LeNet on CIFAR-10. Then, we challenge our activation 
function, along with $\tanh$ and $\mathrm{ReLU}$, by training on MNIST a diverse set of multilayer perceptrons, 
some of them being extreme (narrow and deep).

In the following, we train all the neural networks with the same optimizer, Adam, and the same learning rate $\eta = 
0.001$. 
We use a scheduler and an early stopping mechanism, respectively based on the training loss and the validation loss, 
the test loss not being used during training. We did not use data augmentation.
All the technical details are provided in Appendix~\ref{app:expe:training}.

\paragraph{LeNet-type networks.}
We provide here the results only for our odd activation functions $\phi_{\theta}^{\acto}$.
The results with the positive ones $\phi_{\theta}^{\actp}$ are provided in Appendix
\ref{app:expe:additional}. The results are very similar in both cases.

We consider LeNet-type networks \citep{lecun1998gradient}. 
They are made of two $(5 \times 5)$-convolutional layers, each of them followed by a 
$2$-stride average pooling, and then three fully-connected layers.
We denote by ``$6-16-120-84-10$'' a LeNet neural network with two convolutional layers outputting respectively $6$ and 
$16$ channels, and three fully connected layers having respectively $120$, $84$, and $10$ outputs (the final output of 
size $10$ is the output of the network).

We have tested LeNet with several sizes (see Figure~\ref{fig:training:lenet}). In 
Figure~\ref{fig:training:lenet:tr_nll}, we have plotted the training loss. Overall, $\mathrm{ReLU}$ and $\tanh$ perform 
well, along with some $\phi_{\theta}^{\acto}$ with small $\theta$ and $\varphi_{\delta,\omega}$, with $(\delta,\omega)= (0.99, 
2)$. In 
Figure~\ref{fig:training:lenet:ts_acc}, the results in terms of test accuracy are quite different: $\mathrm{ReLU}$ and 
$\tanh$ still achieve good accuracy, but the other functions achieving similar results on the training loss seem to 
be a bit behind.

So, in this standard setup, the functions we are proposing seem to make the neural network trainable and as 
expressive as with other activation functions, but with some overfitting. 
This is not surprising, since we are testing long-standing activation functions, $\mathrm{ReLU}$ and $\tanh$, 
which have been selected both for their ability to make the neural network converge quickly with good 
generalization, against functions we have designed only according to their ability to propagate information. 
Therefore, according to these plots, taking into account generalization may be the missing piece of our study.

\begin{figure}[h!]
	\begin{subfigure}{1.\linewidth}
		\includegraphics[width=1.\linewidth]{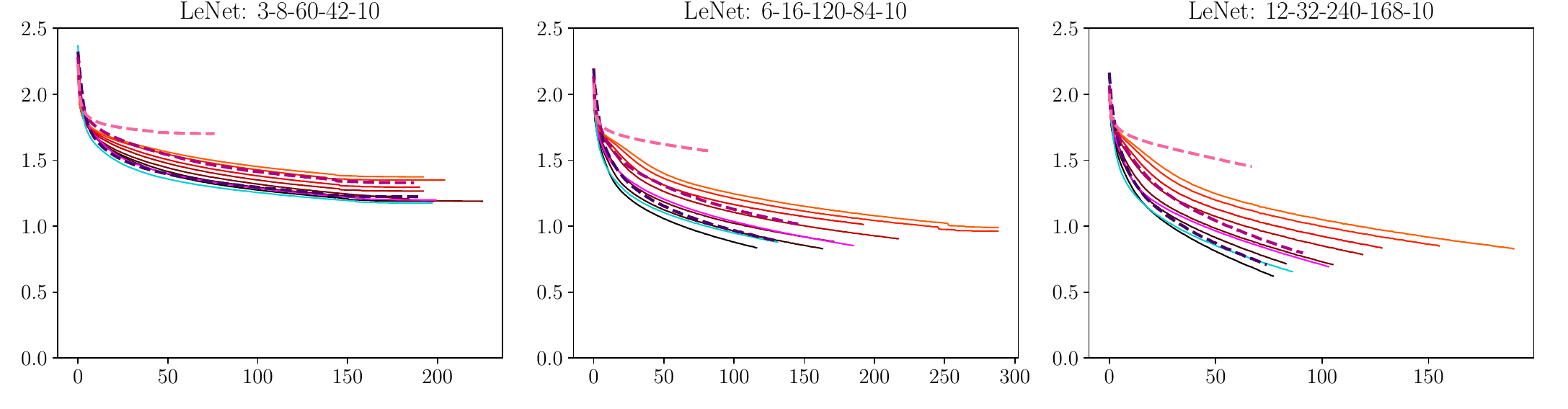}
		\subcaption{Training loss (negative log-likelihood).} 
		\label{fig:training:lenet:tr_nll}
	\end{subfigure}

	~~\\
	
	\begin{subfigure}{1.\linewidth}
		\includegraphics[width=1.\linewidth]{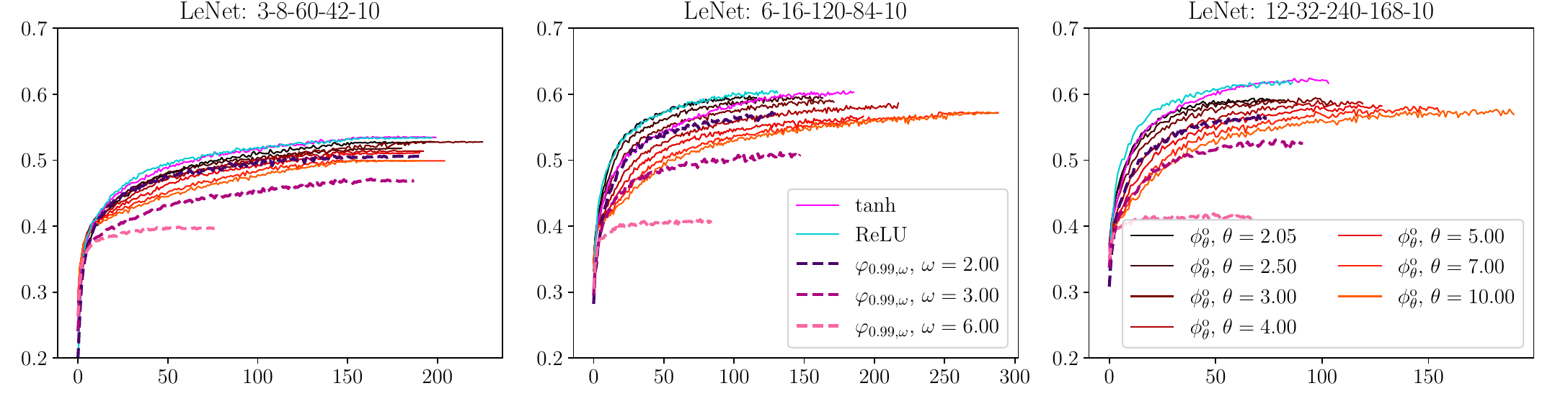}
		\subcaption{Test accuracy.} 
		\label{fig:training:lenet:ts_acc}
	\end{subfigure}
	
	\caption{Training curves for LeNet with 3 different numbers of neurons per layer.} 
	\label{fig:training:lenet}
\end{figure}

\paragraph{Multilayer perceptron.}
We provide here the results only for our odd activation functions $\phi_{\theta}^{\acto}$.
The results with the positive ones $\phi_{\theta}^{\actp}$ are provided in Appendix
\ref{app:expe:additional}. 
The results are very different compared to the previous experiment: training narrow and deep neural networks is much more difficult with the
positive activation functions than with the odd ones.

We have trained a family of multilayer perceptrons on MNIST. They have a constant width $n_l \in \{3, 10\}$ and a depth 
$L \in \{3, 10, 30\}$. So, extreme cases such as a narrow and deep
neural network ($n_l = 3$, $L = 30$) have been tested. 

A series of results are presented in Figure~\ref{fig:training:perc}. 
The training curves have been averaged over 5 experiments.
In terms of training, the strength of our activation functions $\phi_{\theta}$ is more visible in the case of narrow 
neural networks ($n_l = 3$): in general, the loss decreases faster and attains better optima with $\phi = \phi_{\theta}$ 
than with $\phi = \tanh$ or $\relu$. 

We also notice that training a narrow and deep neural network with a $\relu$ activation function is challenging.
This result is consistent with several observations we have made in Section~\ref{sec:prop:realistic}: 
in narrow $\relu$ networks, the sequence of correlations $(C_{pq}^l)_l$ fails to converge to $1$ 
(Fig.~\ref{fig:corr_main}), and the pre-activations are far from being Gaussian (Fig.~\ref{fig:prop:relu_tanh}).

Finally, the case of $\phi_{2.05}$ put aside, the results regarding the activation functions $\phi_{\theta}$ are consistent between 
the two runs. This is not the case with $\tanh$ and $\relu$.

\begin{figure}[h!]
    \includegraphics[width=1.\linewidth]{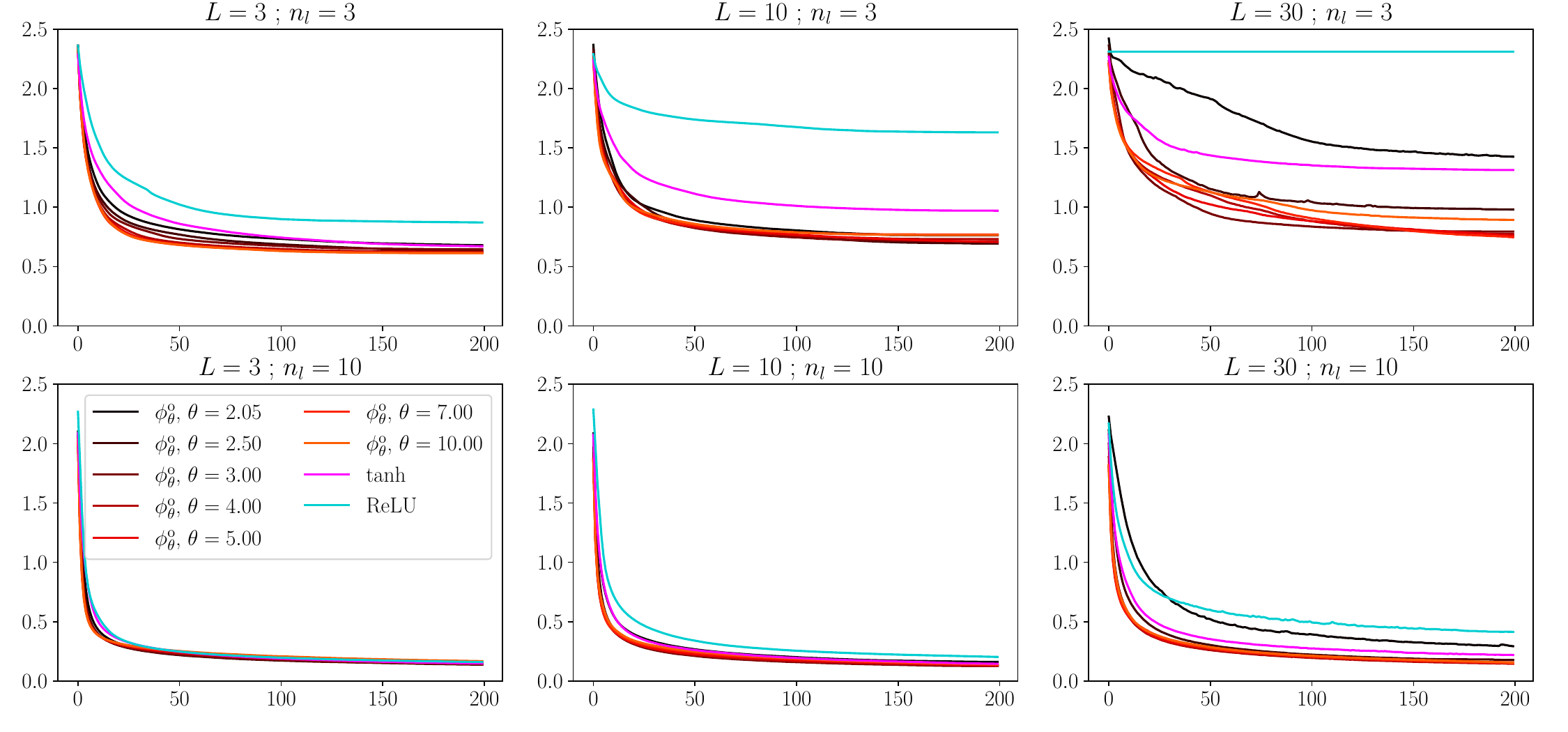}
	\caption{Training loss for a multilayer perceptron, narrow ($n_l \in \{3, 10\}$) and of various depths ($L \in \{3, 10, 30\}$).
    Training curves averaged over 5 experiments.} 
	\label{fig:training:perc}
\end{figure}

\section{Discussion} \label{sec:discussion}

\paragraph{Generality of Constraints~\ref{constr_gaussian} to~\ref{constr_3}.} 
To ensure that pre-activations follow a Gaussian $\mathcal{N}(0, 1)$ distribution and that
the weights in each layer are i.i.d. at initialization, the four constraints outlined in
Section~\ref{sec:propagating} must be satisfied. Relaxing these constraints requires either breaking
symmetries or addressing more complex problems.
Let us consider two examples. 

Example 1: instead of making all the pre-activations Gaussian, one may want to 
impose some other distribution. In this case, the problem to solve would be more difficult: we would have to decompose 
a non-Gaussian random variable $Z$ into a weighted sum $\frac{1}{\sqrt{n}} \sum_{j = 1}^{n} W_j \phi(X_j)$. In this 
case, we cannot use Proposition~\ref{prop:gaussian_sum}, and we would have to make $W_j \phi(X_j)$ belong to a family 
of random variables stable by multiplication by a constant and by sum, for arbitrary $n$. If we aim for a non-Gaussian 
$Z$, this task is much harder. For a study of infinitely wide neural networks going beyond Gaussian pre-activations, 
see \cite{peluchetti2020stable}.

Example 2: instead of assuming that all the weights of a given layer are i.i.d., one may want to initialize them with 
different distributions or to introduce a dependence structure between them. If the goal remains to obtain Gaussian 
pre-activations, this kind of generalization should be feasible without drastically changing the constraints we are 
proposing.

\paragraph{Hypothesis of independent pre-activations.} 
In the constraints, we have assumed that the inputs of any layer are independent. 
This assumption is discussed in Remark~\ref{rem:indep_gaussian} and in Appendix~\ref{app:indep_preact}.
According to the experimental results presented in Appendix~\ref{app:indep_preact}, 
we can build specific cases where the dependence between inputs breaks the Gaussianity of the outputted 
pre-activations, including with our pairs $(\phi_{\theta}^{\actop}, \mathrm{P}_{\theta})$.

However, the results are different with $\phi_{\theta}^{\acto}$ and $\phi_{\theta}^{\actp}$.
The difference between Figures \ref{fig:testing_prop:3} and \ref{fig:testing_prop:1}
is striking: with $\phi_{\theta}^{\acto}$, the dependence between pre-activations 
tends to damage the Gaussianity after a certain number of layers,
but with $\phi_{\theta}^{\actp}$, the Gaussianity of
the pre-activations is remarkably well-preserved during propagation.
The same kind of result is presented in Appendix \ref{app:indep_preact},
Figure \ref{fig:indep_preactiv}.

\paragraph{Other families of initialization distributions and activation functions.} 
Given the constraints we have derived, we have made a choice to obtain our family of initialization distributions 
and activation functions: we have decided that the weights should be sampled from a symmetric Weibull distribution 
$\mathcal{W}(\theta, 1)$.
We have made this choice because Weibull distributions meet immediately Constraint~\ref{constr_2}, 
and we can modulate their Generalized Weibull Tail parameter easily (see Remark~\ref{rem:weibull_gwt}).

One may propose another family of initialization distributions, as long as it meets the constraints. However, such a 
family should be selected wisely: if one chooses a distribution with compact support, or $\mathrm{GWT}(\theta)$ with 
$\theta \gg 2$, then the related activation function $\phi$ is likely to be almost linear. 
Intuitively, if we want $W \phi(X)$
to be $\mathcal{N}(0, 1)$ with $X \sim \mathcal{N}(0, 1)$ and very light-tailed $W$, 
then $\phi$ must ``reproduce'' the tail of its input $X$, so $\phi$ must be approximately linear around infinity.

\paragraph{Preserving a characteristic during propagation or imposing stable fixed points?}
Previous works have used both ideas: \cite{glorot2010understanding} aim at preserving the variance of the 
pre-activations, while \cite{poole2016exponential,schoenholz2016deep} aim at imposing a specific stable fixed point
for the correlation map $\mathcal{C}$. 
According to the results we have obtained in Section~\ref{sec:expe:ks}, it is possible to \emph{approximately} preserve 
the distribution of the pre-activations when passing through \emph{one} layer. However, according 
to the results of Section~\ref{sec:expe:multilayer}, a drift can appear after several layers. 
Therefore, when testing an initialization setup in the real world, with numerical errors and approximations, it is
necessary to check the stable fixed points of the monitored characteristic. However, this is not easy to do in 
practice: without the Gaussian hypothesis, it can be difficult to find the possible limits of $(c_{ab}^l)_l$.

\paragraph{Gaussian pre-activations and Neural Tangent Kernels.} 
With our pair of activation functions and 
initialization procedure, we have been able to obtain non-asymptotic Edge of Chaos, removing the infinite-width 
assumption.
Since this infinite-width assumption is also fundamental in works on the Neural Tangent Kernels (NTKs)  \citep{jacot2018neural}, would it be 
possible to obtain the same kind of results within the NTK setup? 
We do not think so. On the one hand, the infinite-width limit is used to end up with an NTK that is \emph{constant} 
during training, in order to provide exact equations of evolution of the trained neural network.
On the other hand, in our setup, we only ensure Gaussian pre-activations, and this does not imply that the NTK 
would be constant during training.
Nevertheless, with Gaussian pre-activations, we could expect an improvement in the convergence rate of the neural 
network towards a stacked Gaussian process, as the widths of the layers tend to infinity. 
Thus, the NTK regime would be easier to achieve.

\paragraph{A broadly useful technique.}
Our proposed setup can be readily adopted by practitioners interested in testing the Gaussian hypothesis 
on problems beyond those considered in this work. The code used here is available on GitHub: 
\url{https://github.com/p-wol/gaussian-preact}. One could use this code to train, evaluate, 
and compare neural networks built using our family of activation functions and initialization distributions 
against standard choices (such as tanh or ReLU with Gaussian EOC initialization), 
to assess the impact of enforcing Gaussian pre-activations.

\paragraph{Taking into account the generalization performance.}
In the EOC framework and ours, a common principle could be discussed and improved. Generalization 
performance is not taken into account at any step of the reasoning. As we have seen in the training experiments, 
the main difficulty in our setup is overfitting: as the training loss decreases without obstacles, 
the test loss ends up being worse than with $\relu$ and $\tanh$ activation functions.
To improve these generalization results, one can introduce a separation between the training set and a 
validation set into the framework.

\paragraph{Precise characterization of the pre-activations distributions $\mathcal{D}^l$.} 
Finally, finding an accurate and useful characterization of the  of the pre-activations distributions $\mathcal{D}^l$ remains an unsolved problem.
In the EOC line of work, the problem has been simplified by using the Gaussian 
hypothesis. However, as shown in the present work, such a simplification is too coarse and leads to inconsistent results.
Nevertheless, our approach has its own drawbacks. Namely, we can ensure Gaussian pre-activations only when using 
specific initialization distributions and activation functions.
Thus, we still lack a characterization of the distributions $\mathcal{D}^l$ that is applicable to widely-used 
networks without rough approximations, and that can be easily used to achieve practical goals, such as finding an optimal 
initialization scheme. From this perspective, we expect that the problem statement of
Section~\ref{sec:prop:commut} will lead to fruitful future research.

\paragraph{Is it desirable to have Gaussian pre-activations?}
We have presented several results that should help to answer this tough question, 
but it remains difficult to answer it definitively:
\begin{itemize}
	\item there is a paradox regarding the $\relu$ activation. 
	On the one hand, when $\relu$ is used in narrow and deep perceptrons, 
	the pre-activations are far from being Gaussian, the sequence of correlations does not converge to $1$, and training is 
	difficult and unstable. 
	On the other hand, in LeNet-type networks with $\relu$, training is easy, and the resulting networks generalize well;
	\item our activation functions $(\phi_{\theta}^{\acto})_{\theta}$ perform quite differently depending on the setup:
	with $\phi_{2.05}$, LeNet can achieve good training losses, but the training of narrow and deep perceptrons may fail.
	With $\phi_{10.00}$, we observe opposite results;
	\item when training narrow networks, the training curves are very different with $\phi_{\theta}^{\acto}$
	and $\phi_{\theta}^{\actp}$. They are difficult to train when using the functions $\phi_{\theta}^{\actp}$,
	as well as $\relu$ (whose pre-activations are far from being Gaussian). But training
	is feasible and leads often to relatively good results when using $\phi_{\theta}^{\acto}$.
\end{itemize}
Therefore, a temporary answer we can give is: with Gaussian pre-activations at initialization 
and odd activations functions, a neural network is likely to be trainable.

\paragraph{Is the ``Gaussian pre-activations hypothesis'' a myth or a reality?}
Our findings reveal that, when using the $\relu$ activation function in several practical scenarios, 
it is largely a myth. In contrast, with the $\tanh$ activation function, 
the hypothesis holds true depending on the network's width: 
it becomes a reality with a sufficiently large yet reasonable number of neurons per layer. 
To extend this reality to networks with any layer width, 
we established a set of constraints that neural network designs must satisfy and proposed solutions 
to meet these constraints. Consequently, we developed a family of activation functions 
$(\phi_{\theta}^{\actp})_{\theta}$ and initialization distributions $\mathrm{P}_{\theta}$ 
that fulfill these requirements, providing robust foundations for the Gaussian hypothesis 
and validating it across all tested cases.

\section*{Acknowledgements}

The project leading to this work has received funding from the European Research Council
(ERC) under the European Union's Horizon 2020 research and innovation program (grant
agreement No 834175) and from the French National Research Agency (ANR-21-JSTM-0001) in the framework of the 
``Investissements d'avenir'' program (ANR-15-IDEX-02).
This work was granted access to the HPC resources of IDRIS under the allocation
2024-AD011013762R2 made by GENCI.
We thank Thomas Dupic for proposing the trick involving the Laplace transform used in 
Appendix~\ref{app:proof_prop_counter}.
We would like to thank an anonymous reviewer for pointing out an error in a previous version of the paper,
and proposing an example that inspired Example~\ref{exa:indep_preactiv} (see App.~\ref{app:indep_preact}).

\bibliographystyle{tmlr}
\bibliography{ImposeGaussianPreactivations_TMLR}

\newpage

\appendix

\section{Activation function with infinite number of stable fixed points for $\mathcal{V}$} 

\subsection{Proof that $\mathcal{V}$ admits an infinite number of fixed points when using $\phi = \varphi_{\delta, 
\omega}$} \label{app:proof_prop_counter}

\begin{customprop}{\ref{prop:counter}}
	For any $\delta \in (0, 1]$ and $\omega > 0$, let us pose the activation function $\phi = \varphi_{\delta, 
		\omega}$. 
	We consider the sequence $(v^l)_l$ defined by:
	\begin{align}
	\forall l \geq 0, \quad v^{l + 1} &= \sigma_w^2 \int \varphi_{\delta, \omega}\left(\sqrt{v^l}z\right)^2 \, 
	\mathcal{D}z + \sigma_b^2 , \label{eqn:prop:counter_dem} \\
	v^0 &\in \mathbb{R}^+_* . \nonumber
	\end{align}
	Then there exists $\sigma_w > 0$, $\sigma_b \geq 0$, and a strictly increasing sequence of stable fixed 
	points $(v_k^*)_{k \in \mathbb{Z}}$ of the recurrence Equation~\eqref{eqn:prop:counter_dem}.
\end{customprop}

\begin{proof}
Let us define:
\begin{align}
\tilde{\mathcal{V}}(v) &:= \frac{1}{v}\mathcal{V}(v | \sigma_w = 1, \sigma_b = 0) , \nonumber \\
\text{so we have:} \qquad \mathcal{V}(v | \sigma_w, \sigma_b) &= \sigma_w^2 v \tilde{\mathcal{V}}(v) + 
\sigma_b^2 . \label{eqn:v_v_tilde}
\end{align}
In the following, we use a simplified notation: $\mathcal{V}(v) = \mathcal{V}(v | \sigma_w, \sigma_b)$.

Our goal is to find a sequence $(v_k^*)_{k \in \mathbb{Z}}$, $\sigma_w$ and $\sigma_b$ such that:
\begin{align*}
\forall k \in \mathbb{Z}, \quad \mathcal{V}(v_k^*) = v_k^* \quad \text{ and } 
\quad \mathcal{V}'(v_k^*) \in (-1, 1) ,
\end{align*}
which would ensure that all $v_k^*$ are stable fixed points of $\mathcal{V}$. 
In order to understand how to build the sequence $(v_k^*)_{k \in \mathbb{Z}}$, let us consider $v > 0$. 
Let $\sigma_b^2 = 0$ and $\sigma_w^2 = 1 / \tilde{\mathcal{V}}(v)$. So we have:
\begin{align*}
\mathcal{V}(v) = v .
\end{align*}
So, any $v > 0$ can possibly be a fixed point if we tune $\sigma_w$ accordingly. We just have to find a 
$v > 0$ such that: $\mathcal{V}'(v) \in (-1, 1)$, with:
\begin{align}
\mathcal{V}'(v) &= \frac{1}{\tilde{\mathcal{V}}(v)} \left( \tilde{\mathcal{V}}(v) + 
v\tilde{\mathcal{V}}'(v) \right)  
= 1 + \tilde{\mathcal{V}}_{\mathrm{e}}'(\ln(v)) , \label{eqn:v_prime_v_e}
\end{align}
where $\tilde{\mathcal{V}}_{\mathrm{e}} : r \mapsto \ln(\tilde{\mathcal{V}}(\exp(r)))$. 
So, knowing that $\tilde{\mathcal{V}}_{\mathrm{e}}$ is $\mathcal{C}^1$ and periodic, it is sufficient 
to prove that it is not constant to ensure that
we can extract one $v^*_0$ such that $\mathcal{V}'(v^*_0) \in (-1, 1)$. 
Then, by periodicity of $\tilde{\mathcal{V}}_{\mathrm{e}}$, we can build a sequence of stable fixed points 
$(v_k^*)_{k \in \mathbb{Z}}$.

We have: 
\begin{align*}
\tilde{\mathcal{V}}(v) &= \frac{1}{v} \int_{-\infty}^{\infty} \varphi_{\delta, \omega}(\sqrt{v}z)^2 \, 
\mathcal{D}z 
= \int_{-\infty}^{\infty} z^2 \exp\left( 2 \frac{\delta}{\omega} \sin( \omega \ln|\sqrt{v}z|) \right) \, 
\mathcal{D}z \\
&= 2 \int_0^{\infty} z^2 \exp\left( 2 \frac{\delta}{\omega} \sin( \omega \ln(\sqrt{v}z)) \right) \, 
\mathcal{D}z 
= 2 \int_0^{\infty} z^2 \exp\left( 2 \frac{\delta}{\omega} \sin\left( \frac{\omega}{2} \ln v + \omega \ln 
z\right) \right) \, \mathcal{D}z .
\end{align*}

\begin{lemma} \label{lem:non_constant}
	$\tilde{\mathcal{V}}$ is not constant.
\end{lemma}
\begin{proof}
	\begin{align*}
	\tilde{\mathcal{V}}(v) &= \frac{1}{\sqrt{2 \pi}} \frac{2}{v^{3/2}} \int_0^{\infty} z^2 \exp\left( 2 
	\frac{\delta}{\omega} \sin( \omega \ln z) \right) \exp\left( - \frac{z^2}{2 v} \right) \, \mathrm{d}z \\
	&= \frac{1}{\sqrt{2 \pi}} \frac{2}{v^{3/2}} \int_0^{\infty} \frac{z}{2 \sqrt{z}} \exp\left( 2 
	\frac{\delta}{\omega} \sin( \omega \ln \sqrt{z}) \right) \exp\left( - \frac{z}{2 v} \right) \, 
	\mathrm{d}z \\
	&= \frac{1}{\sqrt{2 \pi}} v^{-3/2} \mathcal{L} \left[\varphi_{\delta, 
		\omega}(\sqrt{z})\right]\left(\frac{1}{2 v}\right) ,
	\end{align*}
	where $\mathcal{L}$ is the Laplace transform.
	
	Let us suppose that $\tilde{\mathcal{V}}$ is constant: $\forall v > 0, \tilde{\mathcal{V}}(v) = c$. So, if 
	we pose $v \leftarrow \frac{1}{2 v}$, we have:
	\begin{align*}
	\forall v > 0, \quad c = \frac{2}{\sqrt{\pi}} v^{3/2} \mathcal{L} \left[\varphi_{\delta, 
		\omega}(\sqrt{z})\right](v),
	\end{align*}
	that is:
	\begin{align*}
	\mathcal{L} \left[\varphi_{\delta, \omega}(\sqrt{z})\right](v) = \frac{c \sqrt{\pi}}{2} v^{-3/2} .
	\end{align*}
	
	Since the function $z \mapsto \varphi_{\delta, \omega}(\sqrt{z})$ is continuous on $\mathbb{R}^+$, then,
	almost everywhere \citep[see Thm. 22.2,][]{billingsley2008probability}:
	\begin{align*}
	\varphi_{\delta, \omega}(\sqrt{z}) &= \frac{c \sqrt{\pi}}{2} \mathcal{L}^{-1}[v^{-3/2}](z) = c \sqrt{z},
	\end{align*}
	which is impossible for $\delta \in (0, 1]$. Hence the result.
\end{proof}

The function $\tilde{\mathcal{V}}_{\mathrm{e}} : r \mapsto \ln(\tilde{\mathcal{V}}(\exp(r)))$ is 
continuous and $\frac{4 \pi}{\omega}$-periodic:
\begin{align*}
\tilde{\mathcal{V}}_{\mathrm{e}}(r) = \ln\left[2 \int_0^{\infty} z^2 \exp\left( 2 
\frac{\delta}{\omega} 
\sin\left( \frac{\omega}{2} r + \omega \ln z\right) \right) \, \mathcal{D}z \right] ,
\end{align*}
so $\tilde{\mathcal{V}}_{\mathrm{e}}$ is lower and upper bounded and reach 
its bounds (and, by Lemma~\ref{lem:non_constant}, these bounds are different). We define:
\begin{align*}
\tilde{\mathcal{V}}_{\mathrm{e}}^+ &= \max \tilde{\mathcal{V}}_{\mathrm{e}} &
r_0^+ &= \inf \left\{ r > 0 : \tilde{\mathcal{V}}_{\mathrm{e}}(r)
= \tilde{\mathcal{V}}_{\mathrm{e}}^+ \right\}, \\
\tilde{\mathcal{V}}_{\mathrm{e}}^- &= \min \tilde{\mathcal{V}}_{\mathrm{e}} &
r_0^- &= \inf \left\{ r > r_0^+ : \tilde{\mathcal{V}}_{\mathrm{e}}(r)
= \tilde{\mathcal{V}}_{\mathrm{e}}^- \right\} .
\end{align*}
By continuity, $\tilde{\mathcal{V}}_{\mathrm{e}}(r_0^+) = \tilde{\mathcal{V}}_{\mathrm{e}}^+$ and $ 
\tilde{\mathcal{V}}_{\mathrm{e}}(r_0^-) = \tilde{\mathcal{V}}_{\mathrm{e}}^-$.
Since $\tilde{\mathcal{V}}^+ > \tilde{\mathcal{V}}^-$ and $\tilde{\mathcal{V}}_{\mathrm{e}}$ is 
$\mathcal{C}^1$, then there exists $r_0^* \in (r_0^+, r_0^-)$ such that 
$\tilde{\mathcal{V}}_{\mathrm{e}}'(r_0^*) \in (-2, 0)$. 

Since $\tilde{\mathcal{V}}_{\mathrm{e}}$ is $\frac{4 \pi}{\omega}$-periodic, we can define a sequence 
$(r_k^*)_{k \in \mathbb{Z}}$ such that:
\begin{align*}
r_k^* &:= r_0^* + \frac{4 k \pi}{\omega} \\
\tilde{\mathcal{V}}_{\mathrm{e}}(r_k^*) &= \tilde{\mathcal{V}}_{\mathrm{e}}(r_0^*) =: 
\tilde{\mathcal{V}}_{\mathrm{e}}^0 \\
\tilde{\mathcal{V}}_{\mathrm{e}}'(r_k^*) &= \tilde{\mathcal{V}}_{\mathrm{e}}'(r_0^*) \in (-2, 0) .
\end{align*}

So, by using Eqn.~\eqref{eqn:v_v_tilde} and Eqn.~\eqref{eqn:v_prime_v_e} with $\sigma_b^2 = 0$ and 
$\sigma_w^2 = 1/\exp(\tilde{\mathcal{V}}_{\mathrm{e}}^0)$:
\begin{align*}
v_k^* &:= \exp(r_k), \\
\mathcal{V}(v_k^*) &= \frac{1}{\exp(\tilde{\mathcal{V}}_{\mathrm{e}}^0)} v_k^* 
\tilde{\mathcal{V}}(v_k^*) = v_k^* \\
\mathcal{V}'(v_k^*) &= 1 + \tilde{\mathcal{V}}_{\mathrm{e}}'(r_0^*) \in (-1, 1) .
\end{align*}
Thus, $(v_k^*)_{k \in \mathbb{Z}}$ is a sequence of stable fixed points of $\mathcal{V}(\cdot | \sigma_w, 
\sigma_b )$ for well-chosen $\sigma_w$ and $\sigma_b$.
\end{proof}

\subsection{Practical computation of $\sigma_w^2$} \label{app:sigma_omega}

We propose a practical method to ensure that $\mathcal{V}(\cdot | \sigma_w, \sigma_b = 0)$ admits an infinite number of 
stable fixed points when using activation function $\phi = \varphi_{\delta, \omega}$.

In order to achieve this goal, we build $\sigma_w^2 = \sigma_{\omega}^2$ in the following way:
\begin{align*}
	\mathcal{V}_{\mathrm{low}} &:= 2 \int_{0}^{\infty} z^2 \exp\left(-2\frac{\delta}{\omega} \sin(\omega \ln 
	(z)) \right) \, \mathcal{D}z \\
	\mathcal{V}_{\mathrm{upp}} &:= 2 \int_{0}^{\infty} z^2 \exp\left(2 \frac{\delta}{\omega} \sin(\omega \ln 
	(z)) \right) \, \mathcal{D}z , \\
	\sigma_{\omega}^2 &:= \left[\frac{\mathcal{V}_{\mathrm{low}} + \mathcal{V}_{\mathrm{upp}}}{2}\right]^{-1} .
\end{align*}

In practice, for $\delta = 0.99$, we obtain $\sigma_{\omega}$ for $\omega \in \{2, 3, 6\}$:
\begin{align*}
	\sigma_{2} \approx 0.879, \quad \sigma_{3} \approx 0.945, \quad \sigma_{6} \approx 0.987 .
\end{align*}

\section{Discussion about the independence of the pre-activations} \label{app:indep_preact}

In Proposition~\ref{prop:gaussian_sum} and Constraint~\ref{constr_gaussian}, we assume that, for any layer $l$, 
its inputs $(Z^{l}_j)_j$ are independent.
This does not hold in full generality:
\begin{example} \label{exa:indep_preactiv}
	Let $X \sim \mathcal{N}(0, 1)$ be some random input of a two-layer neural network. We perform the following 
	operation:
	\begin{align*}
		Z = \frac{1}{\sqrt{2}} \left[W^2_1 \phi(W_1^1 X) + W^2_2 \phi(W_2^1 X)\right] ,
	\end{align*}
	where $(W_1^1, W_2^1, W_1^2, W_2^2)$ be i.i.d.\ random variables samples from some distribution $\mathrm{P}$, 
	and $\phi$ is some activation function. 
	
	Let $\mathrm{P} = \mathcal{R}$, the Rademacher distribution, i.e., if $W \sim \mathcal{R}$, then $W = \pm 1$ with 
	probability
	$1/2$. Let $\phi = \mathrm{Id}$. Then we have: $Y_1 := W_1^1 X \sim \mathcal{N}(0, 1)$ and 
	$Y_2 := W_2^1 X \sim \mathcal{N}(0, 1)$. But they are not independent: 
	\begin{align*}
		W^2_1 W_1^1 X + W^2_2 W_2^1 X &= (W_1' + W_2') X ,
	\end{align*}
	where $W^2_1 W_1^1$ and $W^2_2 W_2^1$ are two independent Rademacher random variables. So, 
	$Z = 0$ with probability $1/2$. So, $Z$ is not Gaussian.
\end{example}

In this example, we build a non-Gaussian random variable with a minimal neural network, in which we construct
two dependent random variables. So, we should pay attention to this phenomenon when propagating the pre-activations
in a neural network.

One should note that the structure of dependence of $W^2_1 W_1^1 X$ and $W^2_2 W_2^1 X$ does not involve 
their correlation (which is zero), and yet breaks the Gaussianity of their sum. So, in order to obtain 
a theoretical result about the distribution of $Z$, we should study finer aspects of the dependence 
structure.

But, on the practical side, we want to answer the question: to which extent does the relation of dependence between
the pre-activations $(Z^{l}_j)_j$ affect the Gaussianity of $(Z^{l + 1}_i)_i$?

In order to answer this question, we propose a series of experiments on a two-layer neural network. We consider 
a vector of inputs $\mathbf{X} \in \mathbb{R}^{n_0}$, where the $(X_j)_{1 \leq j \leq n_0}$ are 
$\mathcal{N}(0, 1)$ and i.i.d. The scalar outputted by the network is:
\begin{align*}
	Z = \frac{1}{\sqrt{n_1}} \mathbf{W}^2 \phi\left(\frac{1}{\sqrt{n_0}} \mathbf{W}^1 \mathbf{X}\right),
\end{align*}
where the weights $\mathbf{W}^1 \in \mathbb{R}^{n_1 \times n_0}$ and $\mathbf{W}^2 \in \mathbb{R}^{1 \times n_1}$ 
are i.i.d.

We test this setup with different initialization distributions $\mathrm{P}$ and activation functions $\phi$:
\begin{itemize}
	\item usual ones: $\phi = \tanh$ or $\relu$, $\mathrm{P} = \mathcal{N}(0, \sigma_w^2)$, 
	where $\sigma_w^2$ is such that the pair $(\sigma_w^2, \sigma_b^2 = 0.01)$ lies at the EOC;
	\item ours: $\phi = \phi_{\theta}$, $\mathrm{P} = \mathrm{P}_{\theta} = \mathcal{W}(\theta, 1)$.
\end{itemize}

We study three cases:
\begin{itemize}
	\item unfavorable case: $n_0 = 1$, various $n_1 \in [1, 10]$; \\
	the intermediary features $\mathbf{W}^1 \mathbf{X}$ are weakly ``mixed'', 
	so it is credible that they lead to an output that is far from being Gaussian;
	\item favorable case: $n_1 = 2$, various $n_0 \in [1, 10]$;%
	\footnote{The case $n_1 = 1$ is trivial: there is no sum of dependent random variables in 
	the second layer.}%
	\\
	the intermediary features $\mathbf{W}^1 \mathbf{X}$ are ``mixed'' with an increasing rate as $n_0$ increases;
	\item same-width case: $n_0 = n_1$.
\end{itemize}

\begin{figure}[h!]
	\includegraphics[width=1.\linewidth]{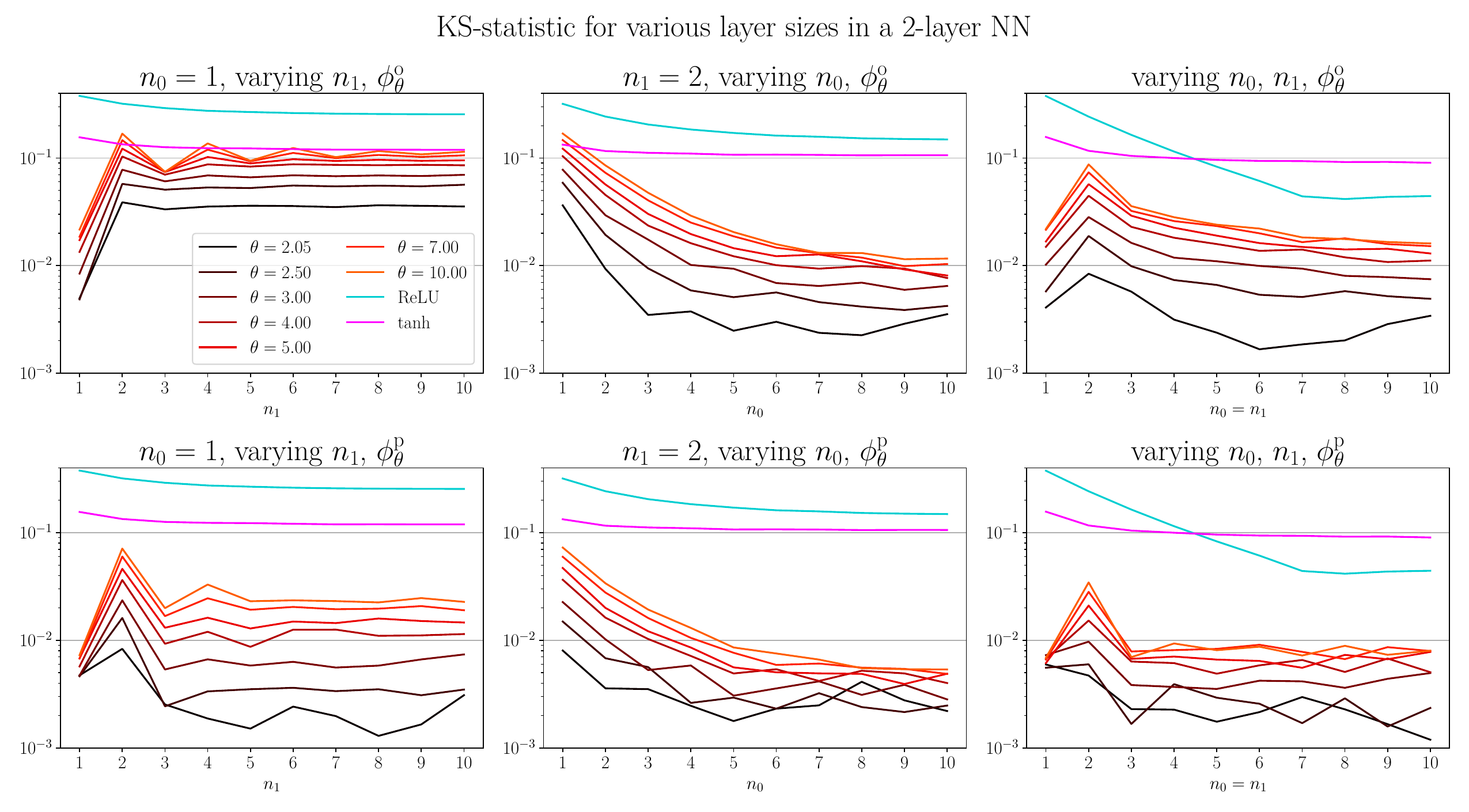}
	
	\caption{Evolution of the distance of the standardized distribution of $Z$ to the $\mathcal{N}(0, 1)$
			according to the Kolmogorov--Smirnov statistic.
		Weight initialization is $\mathcal{W}(\theta, 1)$ when using $\phi = \phi_{\theta}$ and is Gaussian according 
		to the EOC when using $\phi = \tanh$ or $\mathrm{ReLU}$.
		For each point, we have computed the KS-statistic over $200\, 000$ samples of 
		$(\mathbf{X}, \mathbf{W}^1, \mathbf{W}^2)$.
		In the first row, we have plotted the results for the \textbf{odd} activation functions $\phi^{\acto}_{\theta}$. 
		In the second row, we have plotted the results for the \textbf{positive} activation functions 
		$\phi^{\actp}_{\theta}$.} 
	\label{fig:indep_preactiv}
\end{figure}

According to Figure~\ref{fig:indep_preactiv}:
\begin{itemize}
	\item unfavorable case (1st graph, $n_0 = 1$): $Z$ is far from being Gaussian for $\relu$, $\tanh$, 
	and the $\phi^{\acto}_{\theta}$, but is much better for the $\phi^{\actp}_{\theta}$, especially 
	with a small $\theta$;
	\item favorable case (2nd graph, $n_1 = 2$): $Z$ is closer to be Gaussian with 
	our setups $(\mathrm{P}_{\theta}, \phi^{\actop}_{\theta})$ than with $\phi = \tanh$ or $\relu$, 
	especially with larger $n_1$;
	\item same-width case (3rd graph, $n = n_0 = n_1$): the larger the width $n$, 
	the closer $Z$ is to being Gaussian. For a fixed $n$, the distribution of $Z$ is closer to a Gaussian
	with smaller $\theta$. When we use $(\mathrm{P}_{\theta}, \phi^{\actop}_{\theta})$, 
	we are close to the performance of $\tanh$ or better.
\end{itemize}

The case $n_0 = n_1$ is the most realistic one: usually, the sizes of the layers of a neural network 
are of the same order of magnitude. In this case, our setup is better than or equivalent to the ones 
with $\tanh$ or $\relu$.

\paragraph{Mixing of inputs and interference phenomenon.}
We see in Figure~\ref{fig:indep_preactiv} that, in every graph, setups with small $\theta$ lead to better results 
than the ones with large $\theta$. 
This observation could be explained by Example~\ref{exa:indep_preactiv}. In this example, the inputs are weakly 
``mixed'': since $X_0$ is multiplied by a Rademacher random variable, it is possible to partially reconstruct 
$X_0$ after the first layer, and then build destructive interference 
(leading to an output $Z = 0$ half of the time).

So, if we want to avoid this ``interference'' phenomenon, we should use initialization distributions and 
activation functions such that every layer ``mixes'' strongly the inputs. 
Notably, initialization distributions should be far from being a combination of Dirac distributions. Typically, 
$\mathrm{P}_{\theta = 10}$ is close to the Rademacher distribution (see Fig.~\ref{fig:plot_act:dens_Weibull}).

Finally, this phenomenon can be greatly reduced by using our family $(\phi^{\actp}_{\theta})_{\theta}$ of
positive activation functions (see second row of Figure \ref{fig:indep_preactiv}).

\section{Constraints on the product of two random variables} \label{app:dem_product_main}

\begin{customprop}{\ref{thm:product_main}}[Density of a product of random variables at $0$]
Let $W, Y$ be two independent non-negative random variables and $Z = W Y$. Let $f_W, f_Y, f_Z$ be their respective 
density.
Assuming that $f_Y$ is continuous at $0$ with $f_Y(0) > 0$, we have:
\begin{align}
\text{if } \quad \lim_{w \rightarrow 0} \int_{w}^{\infty}  \frac{f_W(t)}{t} \, \mathrm{d}t = \infty, \quad \text{ 
	then } \quad \lim_{z \rightarrow 0} f_Z(z) = \infty . \label{dem:eqn1}
\end{align}
Moreover, if $f_Y$ is bounded:
\begin{align}
\text{if } \quad \int_{0}^{\infty}  \frac{f_W(t)}{t} \, \mathrm{d}t < \infty, \quad \text{ then } \quad 
f_Z(0) = f_Y(0) \int_{0}^{\infty} \frac{f_W(t)}{t} \, \mathrm{d}t . \label{dem:eqn2}
\end{align}
\end{customprop}

\begin{proof}
Let $z, z_0 > 0$:
\begin{align*}
f_Z(z) &= \int_{0}^{\infty} f_Y(t) \frac{1}{t} f_W\left( \frac{z}{t} \right) \, \mathrm{d}t \\
&\geq \int_{0}^{z_0} f_Y(t) \frac{1}{t} f_W\left( \frac{z}{t} \right) \, \mathrm{d}t 
\geq \inf_{[0, z_0]} f_Y \cdot \int_{0}^{z_0} \frac{1}{t} f_W\left( \frac{z}{t} \right) \, \mathrm{d}t 
\geq \inf_{[0, z_0]} f_Y \cdot \int_{z/z_0}^{\infty} \frac{f_W(t)}{t} \, \mathrm{d}t .
\end{align*}
Let us take $z_0 = \sqrt{z}$. We have:
\begin{align*}
f_Z(z) &\geq \inf_{[0, \sqrt{z}]} f_Y \cdot \int_{\sqrt{z}}^{\infty} \frac{f_W(t)}{t} \, \mathrm{d}t .
\end{align*}
Then we take the limit $z \rightarrow 0$, hence:
\begin{itemize}
	\item if $\int_{0}^{\infty} \frac{f_W(t)}{t} \, \mathrm{d}t = \infty$, then: $\lim_{z \rightarrow 0} f_Z(z) = 
	\infty$, which achieves~\eqref{dem:eqn1};
	\item if $\int_{0}^{\infty} \frac{f_W(t)}{t} \, \mathrm{d}t < \infty$, then: $f_Z(0) \geq f_Y(0) \int_{0}^{\infty} 
	\frac{f_W(t)}{t} \, \mathrm{d}t$, which achieves one half of~\eqref{dem:eqn2};
\end{itemize}

Let us prove the second half of~\eqref{dem:eqn2}. Let $z, z_0 > 0$:
\begin{align*}
f_Z(z) &= \int_{0}^{\infty} f_Y(t) \frac{1}{t} f_W\left( \frac{z}{t} \right) \, \mathrm{d}t \\
&= \int_{0}^{z_0} f_Y(t) \frac{1}{t} f_W\left( \frac{z}{t} \right) \, \mathrm{d}t + \int_{z_0}^{\infty} f_Y(t) 
\frac{1}{t} f_W\left( \frac{z}{t} \right) \, \mathrm{d}t \\
&\leq \sup_{[0, z_0]} f_Y \cdot \int_{z/z_0}^{\infty} \frac{f_W(t)}{t} \, \mathrm{d}t + \int_{1}^{\infty} f_Y\left(z_0 
t\right) \frac{1}{t} f_W\left( \frac{z}{z_0 t} \right) \, \mathrm{d}t
\end{align*}

Let $z_0 = \sqrt{z}$. We have:
\begin{align*}
f_Z(z) &\leq \sup_{[0, \sqrt{z}]} f_Y \cdot \int_{\sqrt{z}}^{\infty} \frac{f_W(t)}{t} \, \mathrm{d}t + 
\int_{1}^{\infty} f_Y\left(\sqrt{z} t\right) \frac{1}{t} f_W\left( \frac{\sqrt{z}}{t} \right) \, \mathrm{d}t ,
\end{align*}
where:
\begin{align*}
\int_{1}^{\infty} f_Y\left(\sqrt{z} t\right) \frac{1}{t} f_W\left( \frac{\sqrt{z}}{t} \right) \, \mathrm{d}t
\leq \|f_Y\|_{\infty} \int_{1}^{\infty} \frac{1}{t} f_W\left( \frac{\sqrt{z}}{t} \right) \, \mathrm{d}t 
\leq \|f_Y\|_{\infty} \int_{0}^{\sqrt{z}} \frac{f_W\left( t \right)}{t} \, \mathrm{d}t .
\end{align*}

According to the hypotheses, we have, as $z \rightarrow 0$:
\begin{align*}
\sup_{[0, \sqrt{z}]} f_Y &\rightarrow f_Y(0), &
\int_{\sqrt{z}}^{\infty} \frac{f_W(t)}{t} \, \mathrm{d}t &\rightarrow \int_{0}^{\infty} \frac{f_W(t)}{t} \, \mathrm{d}t ,
&
\int_{0}^{\sqrt{z}} \frac{f_W\left( t \right)}{t} \, \mathrm{d}t &\rightarrow 0 ,
\end{align*}
hence the result.
\end{proof}

\section{Activation functions with vertical tangent at $0$} \label{app:vertical}

In the following lemma, we show that if we want the activation $Y$ to have a density that is $0$ at $0$, then the
activation function $\phi$ should have a vertical tangent at $0$. $G$ plays the role of pre-activation.

\begin{lemma}\label{lem:infinity_at_zero}
	Let $\phi$ be a function transforming a Gaussian random variable $G \sim \mathcal{N}(0, 1)$ into a symmetrical 
	random variable $Y$ with a density $f_Y$ such that $f_Y(0) = 0$.
	That is, $Y = \phi(G)$. Then $\phi$ has a vertical tangent at $0$.
\end{lemma}

\begin{proof}
	We have:
	\begin{align*}
	\phi(x) = F_Y^{-1}(F_G(x)) ,
	\end{align*}
	where $F_G$ and $F_Y$ are the respective CDFs of $G$ and $Y$.
	
	Thus:
	\begin{align*}
	\phi'(x) = F_G'(x) \frac{1}{F_Y'(F_Y^{-1}(F_G(x)))}
	\end{align*}
	
	Therefore:
	\begin{align*}
	\phi'(0) &= F_G'(0) \frac{1}{F_Y'(F_Y^{-1}(F_G(0)))}
	= \frac{1}{\sqrt{2 \pi}} \frac{1}{F_Y'(F_Y^{-1}(1/2))}
	= \frac{1}{\sqrt{2 \pi}} \frac{1}{F_Y'(0)}
	= \infty .
	\end{align*}
\end{proof}

\section{The Mellin transform} \label{app:mellin}

\subsection{Generalities}

We assume that $G = W Y \sim \mathcal{N}(0, 1)$. Let us consider the random variables $|W|$, $|Y|$ and $|G| = |W| \cdot 
|Y|$.
Let $f_{|W|}$, $f_{|Y|}$ and $f_{|G|}$ be their densities. Under integrability conditions, 
we can express the density $f_{|G|}$ of the product $|G| = |W| |Y|$ with the product-convolution operator $\dot{*}$:
\begin{align*}
f_{|G|}(z) = (f_{|W|} \dot{*} f_{|Y|})(z), \quad \text{where} \quad 
(f_{|W|} \dot{*} f_{|Y|})(z) = \int_{0}^{\infty} f_{|W|}\left(\frac{z}{t}\right) f_{|Y|}(t) \frac{1}{t} \, \mathrm{d} 
t .
\end{align*}
We can also express the CDF of $|G|$ this way:
\begin{align}
F_{|G|}(z) = \int_{0}^{\infty} F_{|W|}\left(\frac{z}{t}\right) f_{|Y|}(t)  \, \mathrm{d} t . \label{eqn:mellin:cdf}
\end{align}
Then, we can use the following property of the Mellin transform $\mathcal{M}$:
\begin{align*}
\mathcal{M}f_{|G|} = (\mathcal{M} f_{|W|}) \cdot (\mathcal{M} f_{|Y|}) , \quad 
\text{where} \quad (\mathcal{M}f)(t) = \int_{0}^{\infty} x^{t - 1} f(x) \, \mathrm{d} x .
\end{align*}
In short, $\mathcal{M}$ transforms a product-convolution into a product in the same manner as the Fourier transform 
$\mathcal{F}$ transforms a convolution into a product. We have then:
\begin{align*}
f_{|Y|}(y) &:= \mathcal{M}^{-1}\left[\frac{\mathcal{M} f_{|G|}}{\mathcal{M}f_{|W|}}\right](y) .
\end{align*}
Then, by symmetry, we can obtain $f_{Y}$ from $f_{|Y|}$.
However, while $\mathcal{M} f_{|G|}$ and $\mathcal{M}f_{|W|}$ are easy to compute, the inverse Mellin transform 
$\mathcal{M}^{-1}$ seems to be analytically untractable in this case:
\begin{align*}
(\mathcal{M} f_{|G|})(s) = \frac{2^{\frac{s}{2} - \frac{1}{2}} \Gamma( \frac{s}{2} )}{\sqrt{\pi}},& \qquad
(\mathcal{M} f_{|W|})(s) = \Gamma\left( \frac{s - 1}{\theta} + 1 \right), \\
\text{so: } \quad \frac{(\mathcal{M} f_{|G|})(s)}{(\mathcal{M}f_{|W|})(s)} &= \frac{1}{\sqrt{\pi}} \, 
\frac{2^{\frac{s}{2} - \frac{1}{2}} \Gamma( 
	\frac{s}{2} )}{\Gamma\left( \frac{s - 1}{\theta} + 1 \right)} .
\end{align*}

\subsection{Numerical inversion of the Mellin transform} \label{app:mellin:numerical}

\paragraph{Computation of $f_{|Y|}$ by numerical inverse Mellin transform.}
The Mellin transform of a function can be inverted by using Laguerre polynomials. Specifically, we use the method 
proposed by \cite{theocaris1977numerical} and slightly accelerated by the numerical procedure of 
\cite{gabutti1991numerical}:
\begin{align}
(\mathcal{M}^{-1}f)(z) = e^{-\frac{z}{2}} \sum_{k = 0}^{\infty} c_{k + 1} L_k\left(\frac{z}{2}\right) , 
\label{eqn:laguerre} \quad \text{with } c_k := \sum_{n = 1}^{k} {k - 1 \choose n - 1} (-1)^{n - 1} \frac{f(n)}{2^n 
\Gamma(n)} , 
\end{align}
where the $(L_k)_k$ are the Laguerre polynomials \citep[see Section 7.41,][]{gradshteyn2014table}.

\paragraph{Experiments.}

A common way of computing the inverse Mellin transform consists of using Equation~\eqref{eqn:laguerre}. Specifically, 
the sequence $(c_k)_k$ must be computed. 

In order to compute the density $f_{Y}$ of $\mathrm{Q}_{\theta}$ with $\theta = 2.05$, we have computed numerically 
$(c_k)_k$ for $k \in [1, 500]$. 
The results are plotted in Figure~\ref{fig:mellin_ck}. We have tested three methods to compute the $c_k$: 
\begin{itemize}
	\item floating-point operations using 64 bits floats;
	\item floating-point operations using 128 bits floats;
	\item SymPy: make the whole computation using SymPy, a Python library of symbolic computation (very slow).
\end{itemize}

In all three cases, instabilities appear before the sequence $(c_k)_k$ has fully converged to $0$. 
Moreover, the oscillations of $(c_k)_k$ around $0$ have an increasing wavelength, which indicates that we may have to 
go far beyond $k = 500$ to get enough coefficients $(c_k)_k$ to reconstruct the wanted inverse Mellin transform.

In Figure~\ref{fig:mellin_ck_inverse}, we have plotted two estimations of $f_{|Y|}$: the density obtained directly by 
using $(c_k)_{k \in [1, 300]}$, and the density obtained by using a sequence $(c_k)_{k \in [1, 20000]}$ where the 
values $(c_k)_{k \in [301, 20000]}$ have been extrapolated from $(c_k)_{k \in [1, 300]}$.%
\footnote{The extrapolation has been performed by modeling the graph of $(c_k)_k$ as the product of a decreasing 
	function and a cosine with decreasing frequency.}

The resulting estimations of the density $f_{|Y|}$ take negative values and seem to be noised. So, more work is needed 
to obtain smooth and proper densities, especially if we want them to meet Constraints~\ref{constr_2b} 
and~\ref{constr_3} (density at $0$ and decay rate at $\infty$).

\begin{figure}[h!]
	\includegraphics[width=1.\linewidth]{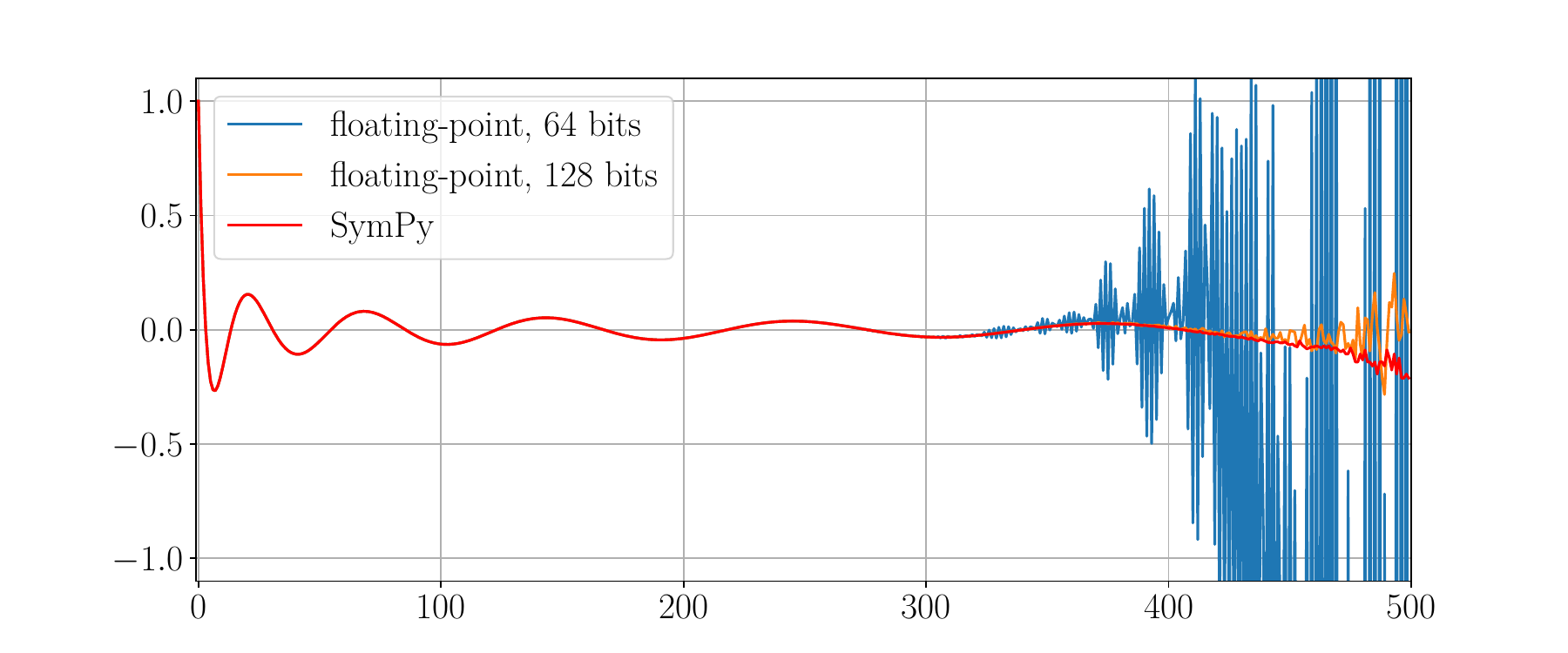}
	\caption{Evolution of various numerical computations of $c_k$ as $k$ grows.} \label{fig:mellin_ck}
\end{figure}

\begin{figure}[h!]
	\includegraphics[width=1.\linewidth]{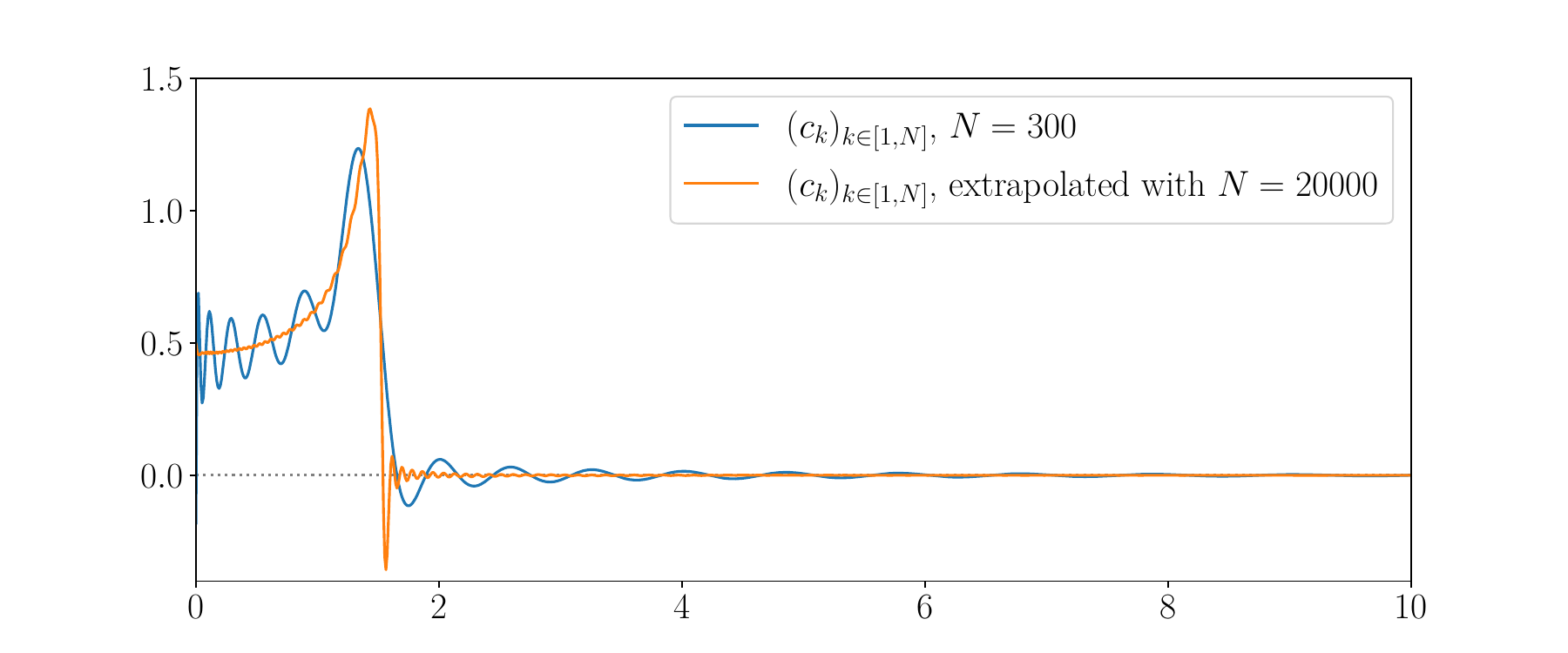}
	\caption{Numerical inverse Mellin transform with two different computations of $(c_k)_k$: direct computation of 
		$(c_k)_{k \in [1, N]}$ with $N = 300$; extrapolation of $(c_k)_{k \in [301, 20000]}$ from $(c_k)_{k \in [1, 
			300]}$.} \label{fig:mellin_ck_inverse}
\end{figure}

\paragraph{Conclusion.}
We observe that this computation of the inverse Mellin transform has several 
intrinsic problems:
\begin{itemize}
	\item the computation of the $c_k$ coefficients involves a sum of terms with alternating signs, which become larger 
	(in absolute value) as $k$ grows 
	and which are supposed to compensate such that $c_k \rightarrow 0$ as $k \rightarrow \infty$. Such a numerical 
	computation, involving both large and small terms, makes the resulting $c_k$ 
	very unstable as $k$ grows;
	\item when $\theta \approx 2$, the sequence $(c_k)_k$ tends extremely slowly to $0$;
	\item if we approximate $\mathcal{M}^{-1}f$ with the finite sum of the first $K$ terms of the series in 
	Equation~\eqref{eqn:laguerre}, we cannot guarantee the non-negativeness of the resulting function, 
	which is meant to be a density.
\end{itemize}
So, this method is unpractical to compute the density of a distribution in our case.

\section{Obtaining $f_{|Y|}$: experimental details} \label{app:numerical_inv}

We optimize the vector of parameters $\Lambda$ with respect to the following loss:
\begin{align*}
\ell(\Lambda) &:= \| \hat{F}_{\Lambda} - F_{|G|} \|_{\infty} \\
\hat{F}_{\Lambda}(z) &:= \int_{0}^{\infty} F_{|W|}\left(\frac{z}{t}\right) g_{\Lambda}(t)  \, \mathrm{d} t .
\end{align*}

\paragraph{Dataset.}
We build the dataset $\mathcal{Z}$ of size $d$:
\begin{align*}
	\mathcal{Z} = \left\{ 0, z_{\mathrm{max}} \frac{1}{d - 1}, z_{\mathrm{max}} \frac{2}{d - 1}, \cdots , 
	z_{\mathrm{max}} \right\} .
\end{align*}
In our setup, $d = 200$ and $z_{\mathrm{max}} = 5$.

\paragraph{Computing the loss.}
For each $z$ in $\mathcal{Z}$, we compute numerically $\hat{F}_{\Lambda}(z)$. Then, we are able to compute 
$\ell(\Lambda)$. We keep track of the computational graph with PyTorch, in order to backpropagate the gradient and
train the parameters $\Lambda$ by gradient descent.

\paragraph{Initialization of the parameters.}
We initialize $\Lambda = (\alpha, \gamma, \lambda_1, \lambda_2)$ in the following way:
$\alpha = 3; \gamma = 1; \lambda_1 = 1; \lambda_2 = 1$.

\paragraph{Optimizer.}
We use the Adam optimizer \citep{kingma2015adam} with the parameters:
learning rate $= 0.001$;
$\beta_1 = 0.9$;
$\beta_2 = 0.999$;
weight decay $= 0$.
We train $\Lambda$ for $100$ epochs.

\paragraph{Learning rate scheduler.}
We use a learning rate scheduler based on the reduction of the training loss.
If the training loss does not decrease at least by a factor $0.01$ for $20$ epochs, 
then the learning rate is multiplied by a factor $1/\sqrt[3]{10}$.
After a modification of the learning rate, we wait at least $20$ epochs before any modification.

\paragraph{Scheduler for $\theta'$.}
We recall that the definition of $g_{\Lambda}$ involves $\theta'$, defined by: $\frac{1}{\theta} + \frac{1}{\theta'} = 
\frac{1}{2}$. It is not a parameter to train. Empirically, we found that the following schedule improves the 
optimization process:
\begin{itemize}
	\item from epoch $0$ to epoch $49$, $\theta'$ increases linearly from $2$ to its theoretical value $(\frac{1}{2} - 
	\frac{1}{\theta})^{-1}$;
	\item at the beginning of epoch $50$, we reinitialize the optimizer and the learning rate scheduler;
	\item we finish the training normally, with $\theta' = (\frac{1}{2} - \frac{1}{\theta})^{-1}$.
\end{itemize}

\section{The Kolmogorov--Smirnov test} \label{app:ks_test}

\paragraph{Description.}
We describe here the Kolmogorov--Smirnov (KS) test \citep{kolmogoroff1941confidence,smirnov1948table}. Given a 
sequence $(Z_1', \cdots , Z_s')$
of $s$ i.i.d.\ random variables sampled from $\mathrm{P}'$:
\begin{enumerate}
	\item we build the empirical CDF $F_s$ of this sample:
	$F_s(z) = \frac{1}{s} \sum_{k = 1}^{s} \mathds{1}_{Z_k' \leq z}$;
	\item we compare $F_s$ to the CDF $F_G$ of $G \sim \mathcal{N}(0, 1)$ by using the $\mathcal{L}^{\infty}$ norm: $D_s = \| F_s - F_G \|_{\infty}$,
	where $D_s$ is the ``KS statistic'';
	\item under the null hypothesis, i.e.\ $\mathrm{P}' = \mathcal{N}(0, 1)$, we have:
	$\sqrt{s} D_s \overset{d}{\rightarrow} K$,
	where $K$ is the Kolmogorov distribution \citep{smirnov1948table}. We denote by $(K_{\alpha})_{\alpha}$ the 
	quantiles of $K$: 
	$\mathbb{P}(K \leq K_{\alpha}) = 1 - \alpha$, for all $\alpha \in [0, 1]$;
	\item finally, we reject the null hypothesis at level $\alpha$ if: 
	$\sqrt{s} D_s \leq K_{\alpha}$.
\end{enumerate}

\paragraph{Limitations.}
In the KS test presented above, the null hypothesis $\mathbb{H}_0$ is $\mathrm{P}' = \mathcal{N}(0, 1)$, which 
is exactly what we intend to demonstrate when using our family  $\{(\mathrm{P}_{\theta}, \phi^{\actop}_{\theta}) : 
\theta \in (2, \infty)\}$, while the alternative hypothesis $\mathbb{H}_1$ is $\mathrm{P}' \neq \mathcal{N}(0, 
1)$.
With this test design, we face a problem: it is impossible to claim that $\mathbb{H}_0$ holds.
More precisely, we have two possible outcomes: either $\mathbb{H}_0$ is rejected, or it is not. Since 
$\mathbb{H}_1$ is the complementary of $\mathbb{H}_0$, a reject of $\mathbb{H}_0$ means an accept of 
$\mathbb{H}_1$. 
But the converse does not hold: if $\mathbb{H}_0$ is not rejected, then it is impossible to conclude that 
$\mathbb{H}_0$ is true: the sample size $s$ may simply be too small, or the KS test may be inadequate for our 
case.

So, to be able to conclude that the pre-activations are Gaussian, we should have built an alternative test, 
where the null hypothesis $\mathbb{H}_0'$ is some \emph{misfit} between $\mathrm{P}'$ and $\mathcal{N}(0, 1)$. 
Or, at least, we should have computed the power of the current KS test.

However, our experimental results are sufficient to \emph{compare} the quality of fit between $\mathrm{P}'$ 
and $\mathcal{N}(0, 1)$ in the tested setups (see Figures~\ref{fig:ks_test} and~\ref{fig:testing_prop}).
It remains clear that, with $\{(\mathrm{P}_{\theta}, \phi_{\theta}) : \theta \in (2, \infty)\}$, the Gaussian 
pre-activations hypothesis is far more likely to hold than in setups involving $\mathrm{tanh}$, 
$\mathrm{ReLU}$ and Gaussian weights.

\section{Experiments} \label{app:experiments}

\subsection{Propagation of the correlations with MNIST} \label{app:expe:extra:MNIST}

Within the setup of Section~\ref{sec:prop:realistic}, we have plotted in Figure~\ref{fig:corr_suppl}
the correlations propagated in a multilayer perceptron with inputs sampled for MNIST ($10$ samples per class, that is, 
$100$ samples in total). 
The layers of the perceptron have $n_l = 10$ neurons each, and the activation function is $\phi = \relu$.

We observe that, in this narrow NN case ($n_l = 10$), some irregularities appear as information propagates into the 
network: some classes seem to (de)correlate in an inconsistent way with the others. 
Specifically, while most of the classes tend to correlate exactly ($C^* \approx 1$), 
the third class (corresponding to the digit ``$2$'') tends to a correlation $C^* \approx 0.7 \neq 1$. 
This result is contradictory to the EOC theory.

\begin{figure}[h!]
	\setlength\tabcolsep{5pt}
	\begin{tabular}{cccccc}
		& $\phi$ & input & $l = 10$ & $l = 30$ & $l = 50$ \\
		& \makecell[b]{$\mathrm{ReLU}$ \\ (MNIST)\\ \\ \\} &
		\includegraphics[page=1,width=.20\linewidth]{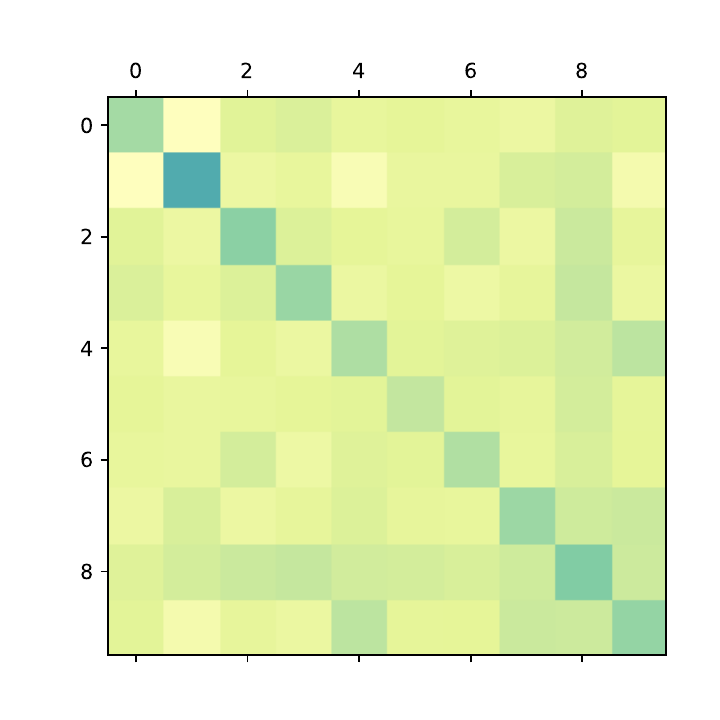}
		 &
		\includegraphics[page=11,width=.20\linewidth]{preactiv-init-mnist-01/StatsEcab_wth-10_act-relu_the-0.00_sampler-normal_.pdf}
		 &
		\includegraphics[page=31,width=.20\linewidth]{preactiv-init-mnist-01/StatsEcab_wth-10_act-relu_the-0.00_sampler-normal_.pdf}
		 &
		\includegraphics[page=51,width=.20\linewidth]{preactiv-init-mnist-01/StatsEcab_wth-10_act-relu_the-0.00_sampler-normal_.pdf}
	\end{tabular}
	\begin{landscape}
		\hspace*{8mm}
		\begin{subfigure}{1.\linewidth}
			\includegraphics[width=.742\linewidth]{ColorBar.pdf}
			\caption*{\hspace*{-55mm}Correlation}
		\end{subfigure}
		\vspace*{-5mm}
	\end{landscape}
	\caption{Propagation of correlations $c^l_{ab}$ in a multilayer perceptron with activation function $\phi = 
		\mathrm{ReLU}$ and inputs sampled from the MNIST dataset. The neural network is initialized at the EOC. 
		Each plot displays a $10 \times 10$ matrix $C_{pq}^l$ whose entries are the average correlation between the 
		pre-activations propagated by samples from classes $p,q\in \{0, \cdots , 9\}$, at the input and right after 
		layers $l \in \{10, 30, 50\}$. See also Figure~\ref{fig:corr_main} in Section~\ref{sec:prop:realistic} for 
		results on CIFAR-10.} \label{fig:corr_suppl}
\end{figure}

\subsection{Propagation of the correlations with $\phi = \phi_{\theta}^{\acto}$} \label{app:expe:extra:phi}

Within the setup of Section~\ref{sec:prop:realistic}, we have plotted in Figure~\ref{fig:corr_suppl_phi}
the correlations propagated in a multilayer perceptron with $\phi = \phi_{\theta}$. The weights have been sampled from 
$\mathcal{W}(\theta, 1)$.

In this setup, the results are consistent, and are consistent with the results for $\phi = \tanh$ (see 
Figure~\ref{fig:corr_main}): the sequence $(C_{pq}^l)_l$ converges to $1$, which was not the case for $\phi = \relu$ 
and $n_l = 10$.

\begin{figure}[p!]
	\setlength\tabcolsep{5pt}
	\begin{tabular}{cccccc}
		& $\phi$ & input & $l = 10$ & $l = 30$ & $l = 50$ \\
		\makecell[b]{\parbox[t]{2mm}{\multirow{2}{*}{\rotatebox[origin=c]{90}{$n_l = 10$ neurons per layer}}}\\ \\ \\ 
		\\ \\} & \makecell[b]{$\phi_{2.05}^{\acto}$ \\ \\ \\ \\} &
		\includegraphics[page=1,width=.20\linewidth]{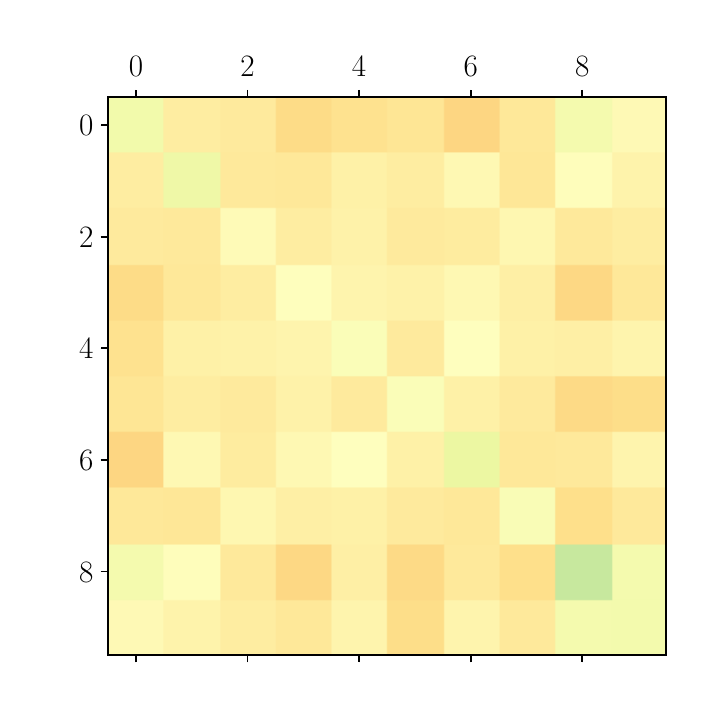}
		&
		\includegraphics[page=11,width=.20\linewidth]{preactiv-init-cifar10-02/StatsEcab_wth-10_act-weibull_the-2.05_sampler-weibull_.pdf}
		&
		\includegraphics[page=31,width=.20\linewidth]{preactiv-init-cifar10-02/StatsEcab_wth-10_act-weibull_the-2.05_sampler-weibull_.pdf}
		&
		\includegraphics[page=51,width=.20\linewidth]{preactiv-init-cifar10-02/StatsEcab_wth-10_act-weibull_the-2.05_sampler-weibull_.pdf}
		\\
		& \makecell[b]{$\phi_{7.00}^{\acto}$ \\ \\ \\ \\} &
		\includegraphics[page=1,width=.20\linewidth]{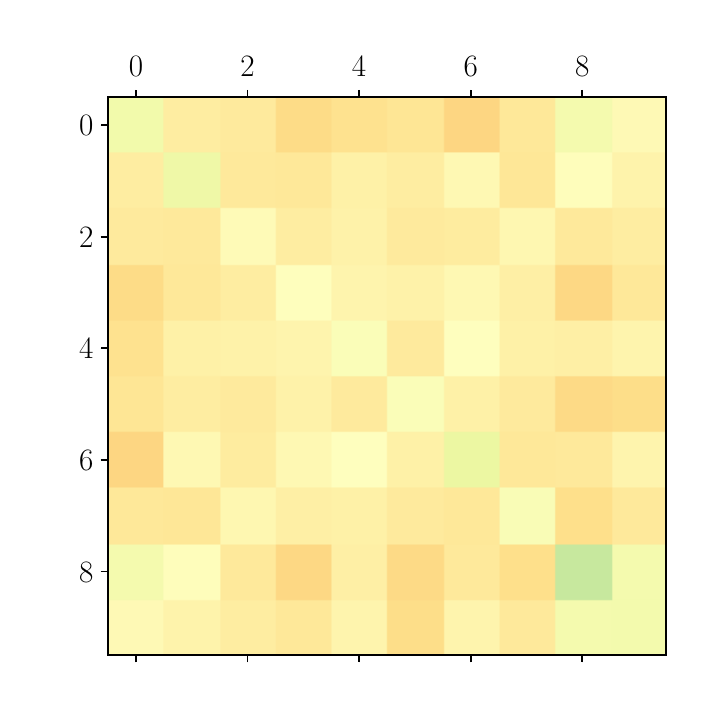}
		& 
		\includegraphics[page=11,width=.20\linewidth]{preactiv-init-cifar10-02/StatsEcab_wth-10_act-weibull_the-7.00_sampler-weibull_.pdf}
		&
		\includegraphics[page=31,width=.20\linewidth]{preactiv-init-cifar10-02/StatsEcab_wth-10_act-weibull_the-7.00_sampler-weibull_.pdf}
		&
		\includegraphics[page=51,width=.20\linewidth]{preactiv-init-cifar10-02/StatsEcab_wth-10_act-weibull_the-7.00_sampler-weibull_.pdf}
		\\
		\cmidrule{3-6}
		\makecell[b]{\parbox[t]{2mm}{\multirow{2}{*}{\rotatebox[origin=c]{90}{$n_l = 100$ neurons per layer}}}\\ \\ \\ 
		\\ \\} & \makecell[b]{$\phi_{2.05}^{\acto}$ \\ \\ \\ \\} &
		\includegraphics[page=1,width=.20\linewidth]{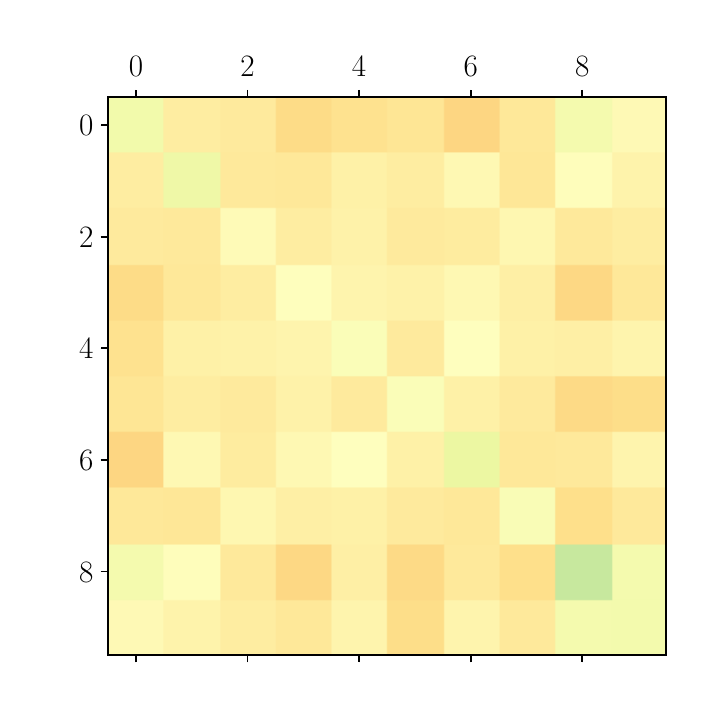}
		&
		\includegraphics[page=11,width=.20\linewidth]{preactiv-init-cifar10-02/StatsEcab_wth-100_act-weibull_the-2.05_sampler-weibull_.pdf}
		&
		\includegraphics[page=31,width=.20\linewidth]{preactiv-init-cifar10-02/StatsEcab_wth-100_act-weibull_the-2.05_sampler-weibull_.pdf}
		&
		\includegraphics[page=51,width=.20\linewidth]{preactiv-init-cifar10-02/StatsEcab_wth-100_act-weibull_the-2.05_sampler-weibull_.pdf}
		\\
		& \makecell[b]{$\phi_{7.00}^{\acto}$ \\ \\ \\ \\} &
		\includegraphics[page=1,width=.20\linewidth]{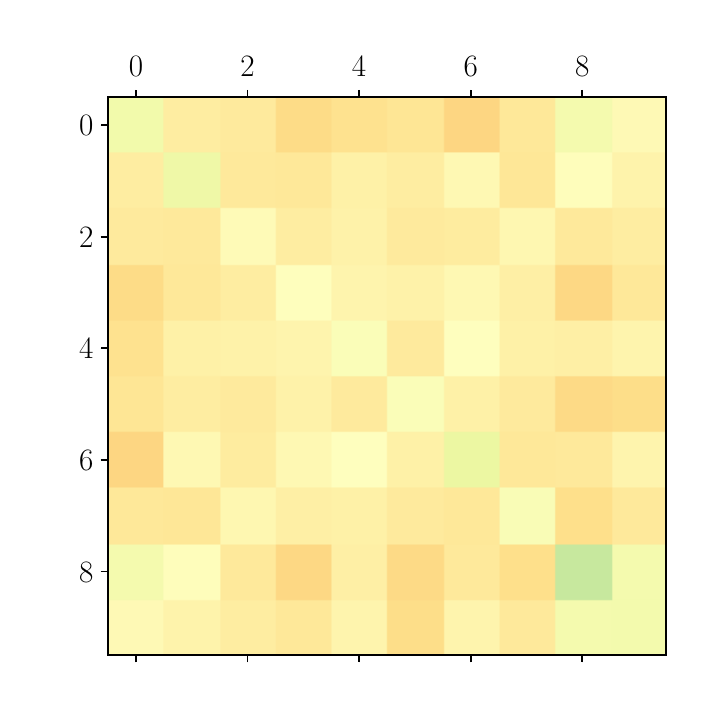}
		& 
		\includegraphics[page=11,width=.20\linewidth]{preactiv-init-cifar10-02/StatsEcab_wth-100_act-weibull_the-7.00_sampler-weibull_.pdf}
		&
		\includegraphics[page=31,width=.20\linewidth]{preactiv-init-cifar10-02/StatsEcab_wth-100_act-weibull_the-7.00_sampler-weibull_.pdf}
		&
		\includegraphics[page=51,width=.20\linewidth]{preactiv-init-cifar10-02/StatsEcab_wth-100_act-weibull_the-7.00_sampler-weibull_.pdf}
		
	\end{tabular}
	\begin{landscape}
		\hspace*{4mm}
		\begin{subfigure}{1.\linewidth}
			\includegraphics[width=.735\linewidth]{ColorBar.pdf}
			\caption*{\hspace*{-55mm}Correlation}
		\end{subfigure}
		\vspace*{-5mm}
	\end{landscape}
	\caption{Propagation of correlations $c^l_{ab}$ in a multilayer perceptron with activation function $\phi_{\theta}^{\acto}$
		with $\theta \in \{2.05, 7.00\}$ and inputs sampled from the CIFAR-10 dataset. 
		The weights are sampled from $\mathcal{W}(\theta, 1)$ and the biases are zero. 
		Each plot displays a $10 \times 10$ matrix $C_{pq}^l$ whose entries are the average correlation between the 
		pre-activations propagated by samples from classes $p,q\in \{0, \cdots , 9\}$, at the input and right after 
		layers 
		$l \in \{10, 30, 50\}$.} \label{fig:corr_suppl_phi}
\end{figure}

\subsection{Variance of the pre-activation when using $\phi = \phi_{\theta}^{\acto}$} \label{app:expe:ks}

As observed in Figure~\ref{fig:ks_test:no_std}, the product $W \phi_{\theta}(X)$ is above the KS threshold,
corresponding to $s = 10^7$ samples and a $p$-value of $0.05$. 
This is not necessarily the case in Figure~\ref{fig:ks_test:with_std}, where the samples are standardized. 
So, we suspect that the variance of $W \phi_{\theta}^{\acto}(X)$ is not exactly $1$.

Therefore, we have reported in 
Table~\ref{tab:ks_stddev} the empirical standard deviation of the product $W \phi(X)$ (i.e., $Z'$ with $n = 1$), 
computed with $s = 10^7$ samples, 
where $X \sim \mathcal{N}(0, 1)$, $W \sim 
\mathcal{W}(\theta, 1)$ if $\phi = \phi_{\theta}$ and $W \sim \mathcal{N}(0, 1)$ if $\phi = \tanh$ or $\relu$.
We observe that the standard deviation of $Z'$, which is expected to be $1$ with $\phi = \phi_{\theta}^{\acto}$, 
is actually a bit different. This would largely explain the differences between Figure~\ref{fig:ks_test:no_std}
and Figure~\ref{fig:ks_test:with_std}.

\begin{table}[h!]
	\caption{Empirical standard deviation of $Z'$ with $n = 1$.}
	\begin{center}
		\begin{tabular}{cccccccccc}
			\toprule
			&\multicolumn{7}{c}{$\phi^{\acto}_{\theta}$} & $\tanh$ & $\relu$\\
			
			$\theta$ & $2.05$ & $2.5$ & $3$ & $4$ & $5$ & $7$ & $10$ &  &  \\
			\midrule 
			$\bar{\sigma}$ &1.003 & 0.994 & 0.989 & 0.984 & 0.981 & 0.979 & 0.978 & 0.628 & 0.707 \\
			\bottomrule
		\end{tabular}
	\end{center}
	\label{tab:ks_stddev}
\end{table}

\subsection{Propagation of the pre-activations} \label{app:expe:prop_comp}

In this subsection, we show how the choice of the input data points affect the propagation of the pre-activations. 

Let $\mathcal{D}^l$ be the distribution of the pre-activation $Z^l_1$ after layer $l$. 
In Figure~\ref{fig:init:comparison_data_pts}, we have plotted the $\mathcal{L}^{\infty}$ distance between the CDFs of 
$\mathcal{D}^l$ and $\mathcal{N}(0, 1)$ for various input points, sampled from different 
classes to improve diversity between them. We have chosen the following setup: CIFAR-10 inputs, multilayer perceptron 
with $100$ layers 
and $100$ neurons per layer, Gaussian initialization at the EOC when using the activation function $\phi = \tanh$ or 
$\mathrm{ReLU}$, and symmetric Weibull initialization $\mathcal{W}(\theta, 1)$ 
when using $\phi = \phi^{\acto}_{\theta}$ or $\phi = \phi^{\actp}_{\theta}$. 

This setup has been selected to illustrate clearly the variability of $\mathcal{D}^l$ when using various data points. 
This variability can be observed with $10$ or $1000$ neurons per layer, though it is less striking.

\paragraph{Input normalization over the whole dataset.}
We perform the usual normalization over the whole dataset: 
\begin{align*}
	\hat{x}_{a;ij} &:= \frac{x_{a;ij} - \mu_i}{\sigma_i} , \\
	\mu_i &:= \frac{1}{N p_i} \sum_{\mathbf{x} \in \mathbb{D}} \sum_{j = 1}^{p_i} x_{ij} \\
	\sigma_i^2 &:= \frac{1}{N p_i - 1} \sum_{\mathbf{x} \in \mathbb{D}} \sum_{j = 1}^{p_i} (x_{ij} - \mu_i)^2 ,
\end{align*}
where $\hat{\mathbf{x}}$ is the normalized data point,  
$x_{a;ij}$ is the $j$-th component of the $i$-th channel of the input image $\mathbf{x}_a$, $p_i$ is the size 
of the $i$-th channel, and $N$ is the size of the dataset $\mathbb{D}$.

\paragraph{Results.}
In Figure \ref{fig:init:comparison_data_pts}, we observe that, 
when using the activation function $\phi = \tanh$, the shape of the curves is the same for all data 
points. This is not the case for $\mathrm{ReLU}$: for the input point ``bird'', the sequence 
$(\mathcal{D}^l)_l$ drifts away from 
$\mathcal{N}(0, 1)$ since the beginning, while for the input ``car'', $(\mathcal{D}^l)_l$ first becomes closer to 
$\mathcal{N}(0, 1)$, then drifts away. It is also the case for $\phi_{\theta}$ with $\theta = 10$: for the input 
``deer'', the distance remains high, while it starts low and increases very slowly (input ``dog''), or even decreases 
in the first place (input ``truck'').
But, in general, when the curves have converged after $100$ layers, it seems that the limit is the same whatever the 
starting data point, and depends only on the choice of the activation function.

The most striking result lies in the difference between Figures \ref{fig:init:comparison_data_pts}
and \ref{fig:init:comparison_data_pts_pos}. With $\phi = \phi^{\actp}_{\theta}$, 
the pre-activations tend \emph{always to be Gaussian}, whatever $\theta$ or the input data point.
This is not the case with $\relu$, $\tanh$ and some $\phi^{\actp}_{\theta}$ with large $\theta$. However, 
$\phi^{\actp}_{\theta}$ remain ``more Gaussian'' than $\relu$ or $\tanh$.

\begin{figure}[h!]
	\includegraphics[width=1.\linewidth]{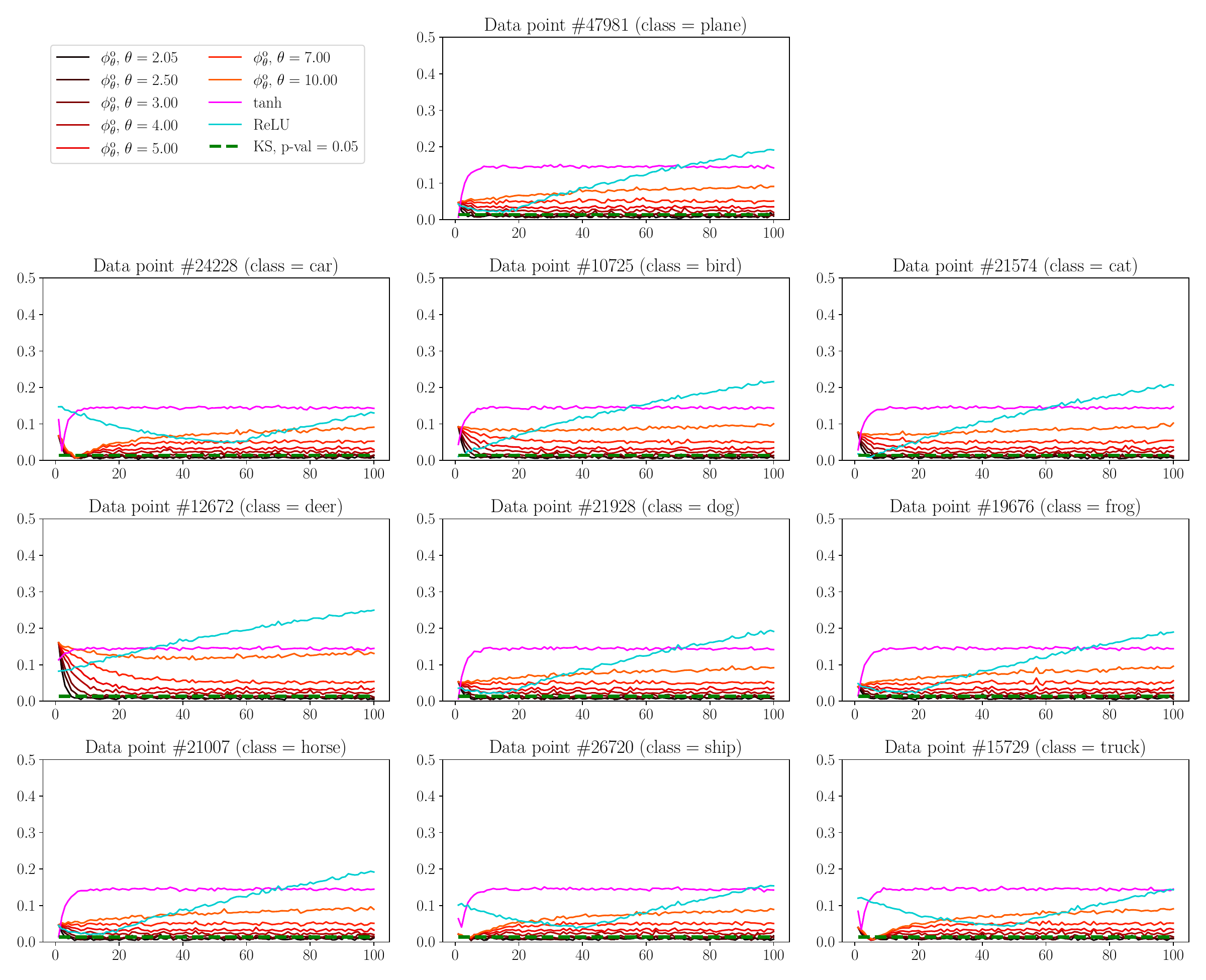}
	\caption{Activations functions $\tanh$, $\relu$ and $\phi^{\acto}_{\theta}$. \\
		Propagation of the distribution $\mathcal{D}^l$ of the pre-activations across the layers $l \in [1, 
		100]$ when inputting various data points of CIFAR-10. 
		Each curve represents the $\mathcal{L}^{\infty}$ distance between the CDF of $\mathcal{D}^l$ and the CDF of the 
		Gaussian $\mathcal{N}(0, 1)$.
		The dashed green line is the threshold of rejection of the Kolmogorov--Smirnov test with $p$-value $0.05$: 
		if a distribution $\mathcal{D}^l$ is represented by a point above this threshold, then the hypothesis 
		``$\mathcal{D}^l = \mathcal{N}(0, 1)$'' is rejected with $p$-value $0.05$.} \label{fig:init:comparison_data_pts}
\end{figure}

\begin{figure}[h!]
	\includegraphics[width=1.\linewidth]{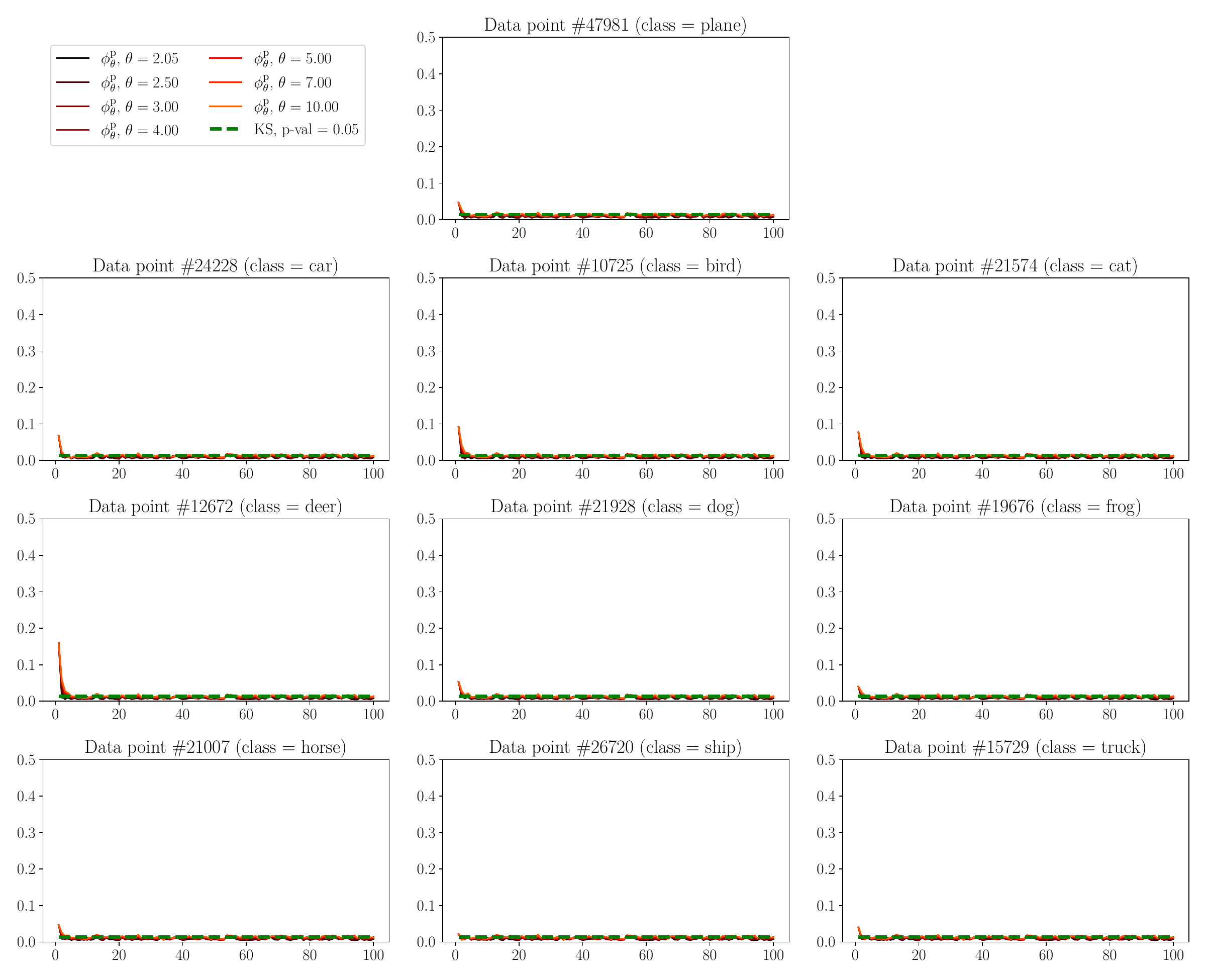}
	\caption{Activation functions $\phi^{\actp}_{\theta}$, same setup as in Figure~\ref{fig:init:comparison_data_pts}.} 
	\label{fig:init:comparison_data_pts_pos}
\end{figure}

\FloatBarrier

\subsection{Experimental details of the training procedure} \label{app:expe:training}

\paragraph{Training, validation, and test sets.}
For MNIST and CIFAR-10, we split randomly the initial training set into two sets: 
the training set, which will be actually used to train the neural network, and the validation set, 
which will be used to stop training when the network begins to overfit.

The sizes of the different sets are as follows:
\begin{itemize}
	\item MNIST: 50000 training samples; 10000 validation samples; 10000 test samples;
	\item CIFAR-10: 42000 training samples, 8000 validation samples; 10000 test samples.
\end{itemize}
The training sets are split into mini-batches with 200 samples each. 
No data augmentation is performed.

\paragraph{Loss.}
Given a classification task with $P$ classes, let $\mathbf{z}^L \in \mathbb{R}^P$ be the pre-activation outputted by 
the last layer of the neural network. First, we perform a $\mathrm{softmax}$ operation:
\begin{align*}
	y_p := \mathrm{softmax}(z^L_p) = \frac{\exp(z^L_p)}{\sum_{p' = 1}^{P} \exp(z^L_{p'})} .
\end{align*}
where the $(y_p)_p$ are the components of $\mathbf{y}\in \mathbb{R}^P$ and 
$(z^L_p)_p$ are the components of $\mathbf{z}^L$.
Then, we compute the negative log-likelihood loss.
For a target class $p \in \{1, \cdots, P\}$, we pose:
\begin{align*}
	\ell(\mathbf{y}, p) := - \log(y_p) .
\end{align*}

\paragraph{Optimizer.}
We use the Adam optimizer \citep{kingma2015adam} with the parameters:
learning rate $= 0.001$;
$\beta_1 = 0.9$;
$\beta_2 = 0.999$;
weight decay $= 0$.

\paragraph{Learning rate scheduler.}
We use a learning rate scheduler based on the reduction of the training loss.
If the training loss does not decrease at least by a factor $0.01$ for $10$ epochs, 
then the learning rate is multiplied by a factor $1/\sqrt{10}$.

\paragraph{Early stopping.}
We add an early stopping rule based on the reduction of the validation loss. 
If the validation loss does not decrease at least by a factor of $0.001$ for $30$ epochs, 
then we stop training.

\subsection{Additional training results: $\phi = \phi_{\theta}^{\actp}$} \label{app:expe:additional}

We provide here additional training experiments, with the same setup as in
Section \ref{sec:expe:training} but with our family of positive
activation functions $(\phi_{\theta}^{\actp})_{\theta}$.

\paragraph{Results with CIFAR-10 + LeNet.}
We have plotted in Figure \ref{fig:training:lenet_pos} the evolution of the training
loss and the test accuracy during training.
The test loss is slightly better than with the family of odd activation 
functions $(\phi_{\theta}^{\acto})_{\theta}$ (see Fig.~\ref{fig:training:lenet:ts_acc})
and the training loss is slightly worse (see Fig.~\ref{fig:training:lenet:tr_nll}).
Overall, the plots are very similar with both types of activations functions.

\begin{figure}[h!]
	\begin{subfigure}{1.\linewidth}
		\includegraphics[width=1.\linewidth]{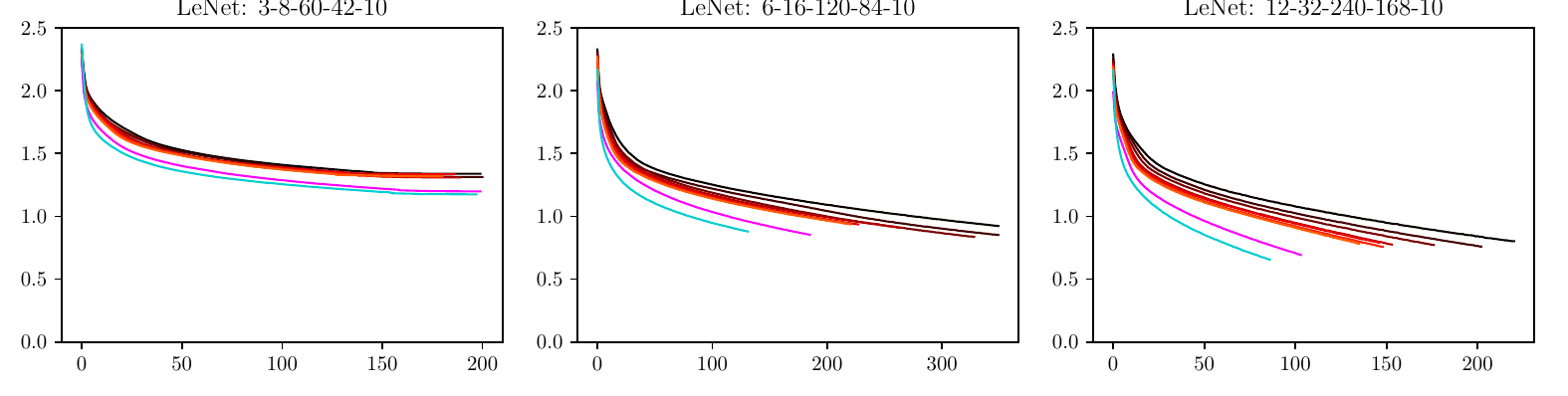}
		\subcaption{Training loss (negative log-likelihood).} 
		\label{fig:training:lenet:tr_nll_pos}
	\end{subfigure}
	
	~~\\
	
	\begin{subfigure}{1.\linewidth}
		\includegraphics[width=1.\linewidth]{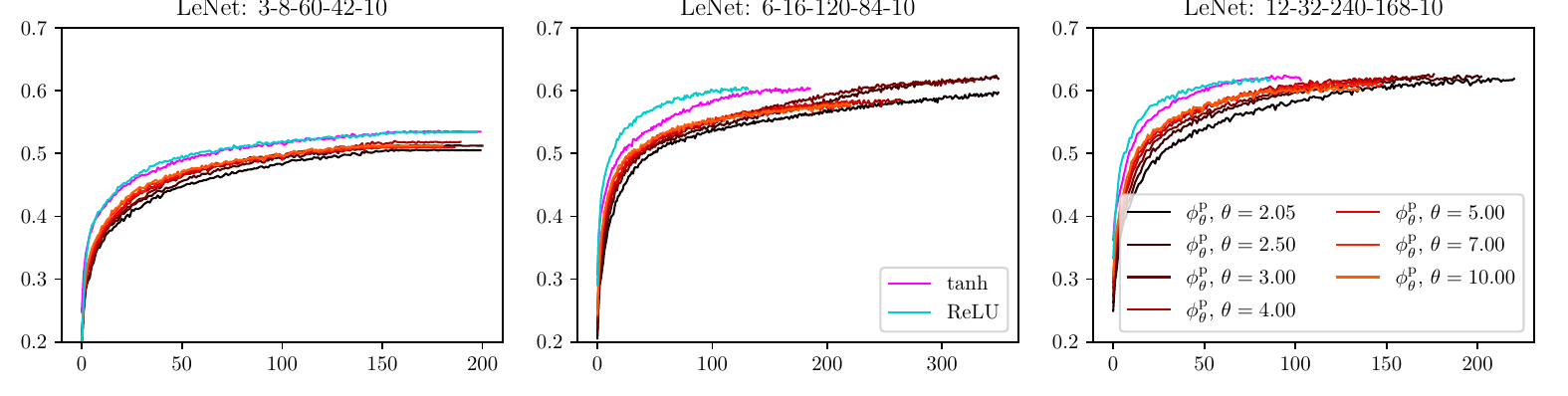}
		\subcaption{Test accuracy.} 
		\label{fig:training:lenet:ts_acc_pos}
	\end{subfigure}
	
	\caption{Training curves for CIFAR-10 + LeNet with 3 different numbers of neurons per layer
		and $\phi = \phi_{\theta}^{\actp}$.} 
	\label{fig:training:lenet_pos}
\end{figure}
\FloatBarrier

\paragraph{Results with MNIST + Multilayer perceptron.}
We have plotted in Figure \ref{fig:training:perc_pos} the results of the training
of a narrow multilayer perceptron, as in Figure \ref{fig:training:perc}.
In this extreme setup, the positive activation functions $\phi_{\theta}^{\actp}$
perform worse than the odd ones $\phi_{\theta}^{\acto}$. 
In particular, many trainings failed with a narrow and deep network ($L = 30$).

\begin{figure}[h!]
	\includegraphics[width=1.\linewidth]{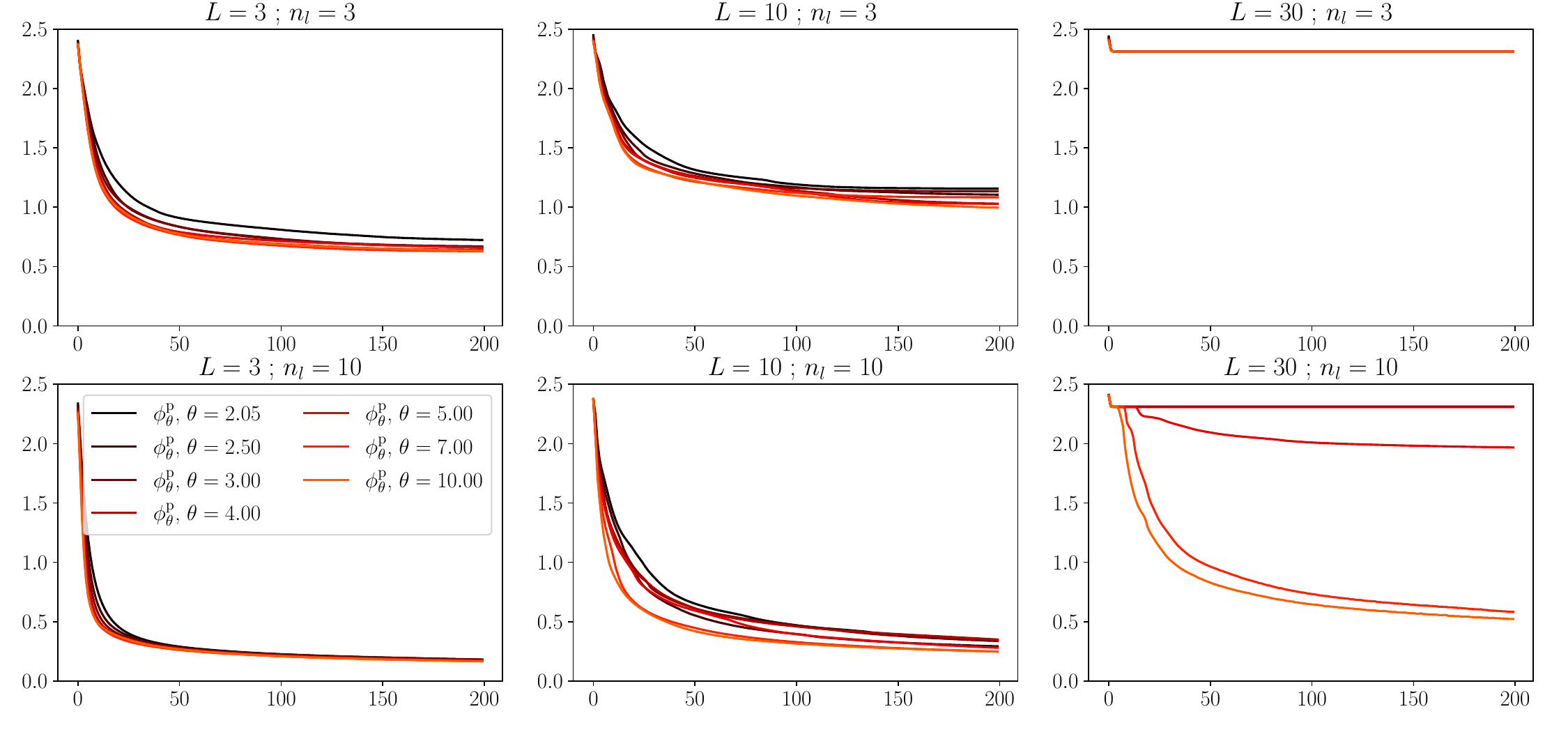}
	\caption{Training loss for a multilayer perceptron, narrow ($n_l \in \{3, 10\}$) and of various depths ($L \in \{3, 10, 30\}$).
		Training curves averaged over 5 experiments.} 
	\label{fig:training:perc_pos}
\end{figure}

\end{document}